\documentclass{article}

\PassOptionsToPackage{numbers, compress}{natbib}

\usepackage[preprint]{neurips_2026}

\usepackage[utf8]{inputenc} % allow utf-8 input
\usepackage[T1]{fontenc}    % use 8-bit T1 fonts
\usepackage{hyperref}       % hyperlinks
\usepackage{url}            % simple URL typesetting
\usepackage{booktabs}       % professional-quality tables
\usepackage{amsfonts}       % blackboard math symbols
\usepackage{nicefrac}       % compact symbols for 1/2, etc.
\usepackage{microtype}      % microtypography
\usepackage{xcolor}         % colors
\usepackage{amsmath}
\usepackage{algorithm}
\usepackage{algorithmic}
\usepackage{multirow}
\usepackage{graphicx}
\usepackage{subcaption}
\usepackage{caption}
\usepackage[font=small]{caption}
\usepackage[export]{adjustbox}
\usepackage{wrapfig}

\usepackage{amsmath}
\usepackage{amssymb}
\usepackage{mathtools}
\usepackage{amsthm}

\theoremstyle{plain}
\newtheorem{theorem}{Theorem}[section]
\newtheorem{proposition}[theorem]{Proposition}

\theoremstyle{definition}

\theoremstyle{remark}

\title{Parallel Tempering Initial Sampling in \\Inference-Time Reward Alignment}

\author{
  Myeongjun Oh$^{1}$ \quad
  Gwangho Kim$^{2}$ \quad
  Sungyoon Lee$^{1,2}$ \\
  \\
  $^{1}$Department of Artificial Intelligence \\
  $^{2}$Department of Computer Science \\
  Hanyang University, Seoul, Korea \\
  \texttt{\{dhaudwns0526, ggh1999, sungyoonlee\}@hanyang.ac.kr}
}

\begin{document}

\maketitle

\begin{abstract}
Inference-time reward alignment steers pretrained diffusion and flow-based generative models to satisfy user-specified rewards without retraining. Recently, Sequential Monte Carlo (SMC) has emerged as a powerful framework for this task by iteratively filtering and propagating multiple particles. However, we show that standard SMC-based methods often suffer from poor performance because they initialize particles from a standard prior, whereas high-reward regions in complex reward landscapes are extremely rare. Further, we show that even recent reward-aware initial sampling approaches remain vulnerable to getting trapped in local modes, as complex reward landscapes are often multi-modal. To overcome these limitations, we propose PATHS (PArallel Tempering for High-complexity reward Sampling), a novel initialization method that couples multiple sampling chains through parallel tempering. PATHS maintains a ladder of reward-tempered chains and periodically performs Metropolis swaps, enabling efficient exploration across flattened reward landscapes, thereby mitigating the mode-trapping issues. Our analysis reveals that this mechanism substantially enhances the finite-budget exploration of rare, high-reward regions that are typically challenging to sample. Experiments on layout-to-image and quantity-aware generation show that PATHS achieves consistent gains in alignment quality, particularly on complex prompts.
\end{abstract}

\section{Introduction}
\label{introduction}

The frontier of generative modeling is shifting beyond scaling pre-training alone toward improving post-training and inference-time computation. In continuous-domain generative modeling, diffusion and flow models~\citep{sohl2015deep, ho2020denoising, song2021score, lipman2023flow, liu2022flow} are increasingly expected to satisfy user-specified rewards, including fine-grained text alignment, layout constraints, and structural plausibility~\citep{chung2023dps, yu2023freedom, feng2024layoutgpt}. \textit{Inference-time reward alignment} addresses this objective without modifying pretrained models~\citep{uehara2025inference}.

Recent approaches formulate inference-time reward alignment as Sequential Monte Carlo (SMC), where multiple particles are propagated and reweighted throughout the denoising trajectory~\citep{smc_wu2023practical, smc_kim2025inference, smc_kim2025testtime, smc_singhal2025general, smc_skreta2025feynman}. However, \citet{yoon2025psi} show that standard SMC methods often suffer from poor performance due to inadequate initialization from reward-agnostic priors. To address this, $\Psi$-Sampler proposes sampling from a reward-informed initial distribution using the preconditioned Crank--Nicolson Langevin (pCNL) algorithm~\citep{cotter2013mcmc}, substantially improving initialization quality. This focus on initialization is especially timely given the rise of recent few-step generative models~\citep{flux2024}, where early latent selection becomes increasingly critical.

Despite these advances, we show that both SMC-based methods and $\Psi$-Sampler can still fail under complex reward landscapes, particularly when (1) high-reward states are extremely \textbf{rare} within the prior distribution, or (2) the reward landscape is highly \textbf{multi-modal}. First, while prior work attributes SMC failure partly to vanishing guidance strength across denoising steps, we show that SMC can fail more severely in rare, high-complexity reward settings, where particles initialized from the prior often struggle to sufficiently cover reward-relevant regions (Section~\ref{sec:smc_fail}). Second, although $\Psi$-Sampler mitigates this initialization mismatch, we show that independent pCNL chains remain vulnerable to \textit{mode trapping} in multi-modal reward landscapes, often collapsing onto limited subsets of valid modes and failing to adequately explore high-reward regions (Section~\ref{sec:psi_fail}).

To address these limitations, we propose \textbf{PATHS} (\textbf{PA}rallel \textbf{T}empering for \textbf{H}igh-complexity reward \textbf{S}ampling), an inference-time initialization framework that leverages Parallel Tempering (Replica Exchange)~\citep{swendsen1986replica, geyer1991markov, earl2005parallel}. PATHS runs multiple pCNL chains across a temperature ladder: higher-temperature chains more freely traverse flattened reward landscapes and discover promising but otherwise difficult-to-reach reward modes, while the lowest-temperature chain targets the original reward-aware posterior. As illustrated in Figure~\ref{fig:swap_trajectory}, periodic Metropolis swap proposals allow high-reward states discovered by exploratory hot chains to be exchanged into colder chains, enabling the cold chain to inherit better reward-aligned configurations without independently overcoming severe local barriers. This cross-temperature exploration substantially improves initialization under rare and multi-modal reward landscapes.

We evaluate PATHS on two reward-alignment tasks characterized by inherently complex reward landscapes: \emph{layout-to-image generation}, where multiple object arrangements can satisfy target bounding-box constraints, and \emph{quantity-aware generation}, where many spatial compositions can realize the same target object count. Across both tasks, PATHS consistently outperforms prior SMC-based methods, including TDS~\citep{smc_wu2023practical} and DAS~\citep{smc_kim2025testtime}, as well as $\Psi$-Sampler~\citep{yoon2025psi}, with particularly strong gains in complex settings. These results highlight the importance of not only reward-aware initialization, but also robust cross-mode exploration for inference-time reward alignment under complex reward landscapes.

\begin{figure}[t]
    \centering
    {\small \textit{Prompt : "A \textcolor{green}{car} and \textcolor{orange}{airplane} above a \textcolor{blue}{horse} and on the right of a \textcolor{red}{dog}..."} \\}
    \includegraphics[width=0.9\linewidth]{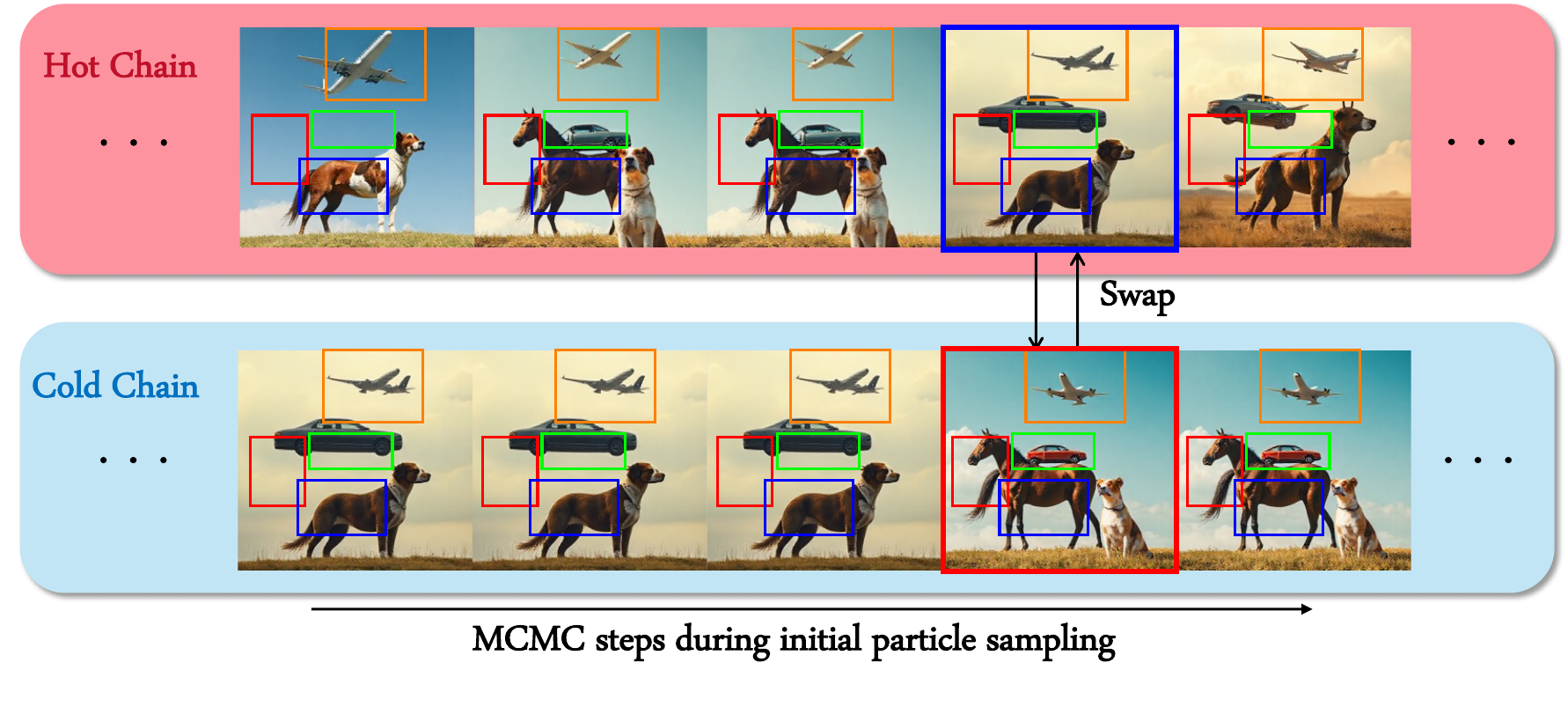}

\caption{\textbf{Replica-exchange swap in PATHS for initial noise sampling on a layout task.} 
We visualize the initial noise using Tweedie's formula across MCMC steps. 
\textbf{Hot chain:} While the hot chain explores diverse particles under a flattened reward landscape, it lacks the stability to settle on high-reward states, causing discovered particles to continuously shift. 
\textbf{Cold chain:} Conversely, the cold chain tends to get trapped in local optima and suffers from frequent rejections, failing to explore new regions. 
\textbf{Metropolis swap:} By periodically swapping states, the cold chain inherits high-reward particles discovered by the hot chain, \textit{effectively transmitting global exploration to a stable exploitation process}.}
    \label{fig:swap_trajectory} 
\end{figure}
\section{Related Work}

\subsection{Inference-time alignment for diffusion models}
A natural approach to aligning diffusion models is to steer pretrained models at sampling time using external reward or guidance signals without additional training. Classifier guidance \citep{dhariwal2021diffusion} requires access to classifiers trained on noisy inputs at intermediate diffusion steps. Several subsequent works \citep{chung2023dps, bansal2023universal, yu2023freedom, he2023manifold} circumvent this requirement by leveraging Tweedie’s formula \citep{efron2011tweedie} to approximate the guidance term. To address the inexactness and bias introduced by these approximations, more recent approaches have incorporated Sequential Monte Carlo (SMC) and inference-time scaling strategies for more accurate alignment \citep{smc_wu2023practical, smc_kim2025testtime, smc_singhal2025general, smc_skreta2025feynman, smc_kim2025inference}.

Another line of work in inference-time scaling focuses on searching for better initial noise, leveraging the observation that some noises consistently lead to higher-quality generations than others \citep{ma2025inference, ahn2024noise, qi2024not, xu2025good}. Crucially, \citet{yoon2025psi} argue that SMC-based methods can be fundamentally limited by poor initial noise sampling, and propose pCNL-based reward-informed initialization to address this issue. This perspective is particularly relevant for recent few-step or distilled diffusion models, where straighter generation trajectories and more reliable early-stage Tweedie estimates make reward-informed initialization increasingly tractable and effective. However, we show that naive pCNL can still fail on complex reward, and use Parallel Tempering to overcome this limitation.

Recently, \citet{he2025crepe} proposed CREPE, a replica-exchange framework for inference-time diffusion control. Unlike CREPE, which manages the full diffusion process, PATHS targets the initialization bottleneck in SMC-based alignment where finite particles struggle to cover sparse, multi-modal reward peaks. PATHS employs replica exchange on the initial noise posterior to generate superior starting particles for subsequent SMC, maintaining a matched reward-evaluation budget.

\subsection{Fine-tuning-based alignment for diffusion models}
Aligning pre-trained diffusion models through fine-tuning has been extensively studied. Early approaches primarily relied on supervised fine-tuning with reward weighting \citep{lee2023aligning, wu2023human}, while subsequent works incorporated reinforcement learning by formulating the denoising process as a Markov Decision Process \citep{black2024training, fan2023dpok}. More recently, diffusion models have adopted preference-based optimization such as DPO \citep{wallace2024diffusion, yang2024using} and GRPO-style methods \citep{liu2025flow, xue2025dancegrpo}. Beyond RL-based frameworks, some studies directly optimize models by backpropagating differentiable rewards \citep{xu2023imagereward, clark2024directly, prabhudesai2023aligning}, or decouple data collection from policy optimization to improve efficiency \citep{zheng2025diffusionnft}. However, these fine-tuning approaches require costly retraining whenever the reward objective changes, whereas inference-time methods offer greater flexibility by adapting without additional training.

% \subsection{Fine-tuning-based alignment for diffusion models}
% Aligning pre-trained diffusion models through fine-tuning has been extensively studied. Early approaches have primarily relied on supervised fine-tuning with reward weighting \citep{lee2023aligning, wu2023human}. Inspired by alignment advances in LLMs \citep{ouyang2022training}, subsequent work has incorporated reinforcement learning by formulating the denoising process as a Markov Decision Process (MDP) \citep{black2024training, fan2023dpok}. Following the success of DPO \citep{rafailov2023direct} and GRPO \citep{shao2024deepseekmath} in LLM alignment, diffusion models have similarly adopted DPO-based approaches \citep{wallace2024diffusion, yang2024using} and GRPO-style methods \citep{liu2025flow, xue2025dancegrpo}. Beyond RL-based frameworks, several studies directly optimize diffusion models by backpropagating differentiable rewards \citep{xu2023imagereward, clark2024directly, prabhudesai2023aligning}. More recently, \citet{zheng2025diffusionnft} propose leveraging the forward process to decouple data collection from policy optimization, enabling the use of arbitrary black-box solvers while improving training efficiency. However, these approaches still require costly retraining whenever the reward model changes, whereas inference-time alignment methods remain more flexible by adapting to new objectives without additional fine-tuning.
\section{Background}
\label{sec:preliminary}
%In this section, we briefly review the fundamentals of score-based generative models and the application of Sequential Monte Carlo (SMC) methods to reward alignment.

\subsection{Diffusion Models and Flow Models}
\label{sec:pre_score}
Flow-based generative models \citep{lipman2023flow, liu2022flow, albergo2025stochastic} learn a transport from a simple prior distribution $p_1 = \mathcal{N}(0, I)$ to a complex data distribution $p_0$ via a deterministic probability flow. This process is governed by an Ordinary Differential Equation (ODE):
\begin{equation*}
    \mathrm{d}x_t = u(x_t, t) \mathrm{d}t,
\end{equation*}
where the time-dependent vector field $u_t$ is parameterized by a neural network $u_\theta$. 

While these flow models are inherently deterministic, they share the same marginal densities $\{p_t\}_{t \in [0,1]}$ with a corresponding class of reverse-time Stochastic Differential Equations (SDEs), such as diffusion models \citep{sohl2015deep, ho2020denoising, song2021score}:
\begin{align*}
\label{eq:preliminaries:pretrained model SDE}
\mathrm{d}x_t = f(x_t, t) \mathrm{d}t + g(t) \mathrm{d}w, \quad f(x_t, t) = u(x_t, t) - \frac{g(t)^2}{2} \nabla \log p_t(x_t),
\end{align*}
where $x_1 \sim p_1$, $f(x_t, t)$ denotes the drift coefficient, $g(t)$ is the diffusion coefficient, and $w$ represents a standard Brownian motion. This equivalent SDE formulation is essential for SMC-based sampling, as it introduces the stochasticity required for particle-based exploration.
We refer to both diffusion and flow-based models as the score-based generative model.

\subsection{Sequential Monte Carlo and Reward Alignment}
\label{sec:pre_SMC}
%\citep{smc_kim2025inference, smc_kim2025testtime, smc_singhal2025general, smc_skreta2025feynman, smc_wu2023practical}
%\citep{doucet2001sequential}
\paragraph{Reward-tilted target.}
A common objective for aligning score-based generative models is the following KL-regularized reward maximization problem~\citep{uehara2024fine}:
\begin{equation}
    \max_p
    \mathbb{E}_{x \sim p}[r(x)]
    -
    \alpha
    D_{\mathrm{KL}}(p \| p_{\theta}),
\end{equation}
where \(p_{\theta}\) denotes the pretrained distribution, and the KL term is used for preventing reward over-optimization~\citep{gao2023scaling}. The solution is the exponentially tilted distribution
\begin{equation}
    \pi^*_0(x_0)
    =
    \frac{1}{Z}
    p_{\theta}(x_0)
    \exp\left(\frac{r(x_0)}{\alpha}\right),
    \label{eq:pi_0}
\end{equation}
where \(Z\) is the normalizing constant \citep{rafailov2023direct}. SMC-based methods target the corresponding trajectory-level target:
\begin{equation}
    p(x_{0:T})
    =
    \frac{1}{Z}
    p(x_T)
    \left[
        \prod\nolimits_{t=0}^{T-1}
        p_\theta(x_t \mid x_{t+1})
    \right]
    \exp\left(\frac{r(x_0)}{\alpha}\right)
    \label{eq:smc_tar}
\end{equation}
where $x_{0:T} = (x_0, \dots, x_T)$ is the discretization of $[0, 1]$ into $T$ steps.
The marginal at \(t=0\) recovers the reward-tilted target in Eq.~\eqref{eq:pi_0}.

\paragraph{Twisted intermediate targets and proposals \citep{smc_wu2023practical}.}
Naively applying importance sampling to Eq.~\eqref{eq:smc_tar} can be inefficient, since the number of particles required for accurate estimation may grow rapidly with the discrepancy between the proposal and target distributions~\citep{chatterjee2018sample}. Twisted diffusion sampling \citep{smc_wu2023practical} addresses this issue by introducing intermediate targets and guided proposals along the reverse diffusion trajectory.

Following TDS-style twisting constructions~\citep{smc_wu2023practical}, the intermediate target at time $t$ is defined as:
\begin{equation}
    \pi_t(x_t) := \frac{\gamma_t(x_t)}{Z_t}, \qquad \gamma_t(x_t) := p_t(x_t) \exp\left( \frac{r(\hat{x}_0(x_t))}{\alpha} \right),
    \label{eq:intermediate_target}
\end{equation}
where $Z_t$ is the normalizing constant, $p_t(x_t)$ is the marginal distribution induced by the pretrained diffusion model, and $\hat{x}_0(x_t)$ denotes the denoised prediction of $x_0$ obtained via Tweedie's formula~\citep{efron2011tweedie}. Since \(\hat{x}_0(x_0)=x_0\), the final target \(\pi_0\) matches Eq.~\eqref{eq:pi_0}. Notably, by incorporating a tempering schedule into the intermediate targets (i.e., multiplying $r$ by $\lambda_t$), this framework recovers the Diffusion Alignment as Sampling (DAS) \citep{smc_kim2025testtime} approach.

For the pretrained reverse transition $p_\theta(x_{t-1} \mid x_t) = \mathcal{N}(x_{t-1} ; \mu_\theta(x_t,t), \sigma_t^2I)$, the commonly used proposal~\citep{smc_wu2023practical, smc_kim2025testtime} is:
\begin{equation}
    m_{t-1}(x_{t-1} \mid x_t) := \mathcal{N} \left( x_{t-1} ; \mu_\theta(x_t,t) + \frac{\sigma_t^2}{\alpha} \nabla_{x_t} r(\hat{x}_0(x_t)), \sigma_t^2 I \right).
    \label{eq:reward_guided_proposal}
\end{equation}

At each step, particles are assigned incremental importance weights:
\begin{equation}
w_{t-1}(x_{t-1}, x_t) 
%= \frac{\gamma_{t-1}(x_{t-1})}{\gamma_t(x_t)} \frac{L_t(x_t \mid x_{t-1})}{m_{t-1}(x_{t-1} \mid x_t)} 
:= \frac{p_\theta(x_{t-1} \mid x_t) \exp \left( r(\hat{x}_0(x_{t-1})/{\alpha} \right)}{m_{t-1}(x_{t-1} \mid x_t) \exp \left( r(\hat{x}_0(x_t))/{\alpha} \right)}. \label{eq:incremental_weight}
\end{equation}

\paragraph{Sampling procedure.}
The resulting SMC sampler initializes particles from $\mathcal{N}(0,I)$, propagates them using the twisted proposal in Eq.~\eqref{eq:reward_guided_proposal}, assigns importance weights according to Eq.~\eqref{eq:incremental_weight}, and resamples particles according to their normalized weights. Repeating this procedure from \(t=T\) to \(t=0\) yields samples from the reward-tilted target \(\pi_0^*\).

\section{Why current methods fail in complex reward}
\label{section4}
\begin{figure*}[t]
    \centering
    \includegraphics[width=\textwidth]{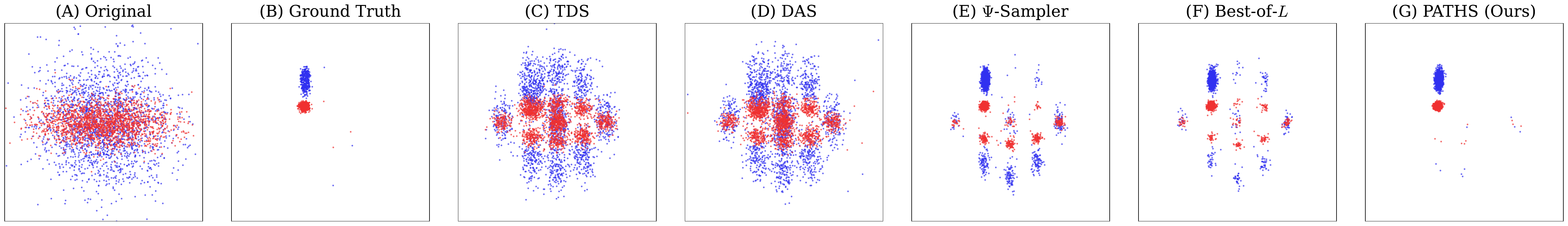}
    \caption{
        \textbf{Toy sampling method comparison.} 
        Each panel visualizes initial latent samples (blue) and their corresponding denoised samples (red). 
        The target density is dominated by a single primary mode with higher weight, while the remaining 8 modes have low weights, making them hardly visible in the ground truth distribution. 
        From left to right: 
        samples from
        (A) the pre-trained flow model; 
        (B) the target distribution defined by Eq.~\eqref{eq:pi_0} (red) and Eq.~\eqref{eq:pi_1} (blue); 
        (C) TDS \citep{smc_wu2023practical}, an SMC-based method; 
        (D) DAS \citep{smc_kim2025testtime}, another SMC-based method; 
        (E) $\psi$-Sampler \citep{yoon2025psi}; 
        (F) Best-of-$L$, which selects the highest-reward chain among $L=3$ independent chains; and 
        (G) results from our proposed \textbf{PATHS}.
    }
    \label{fig:synth}
\end{figure*}

In this section, visually motivated by the toy example in Fig.~\ref{fig:synth}, we examine the challenges of \textit{complex rewards}—landscapes characterized by target regions that are rare, sharp, and highly multi-modal. Within this context, we argue for two key limitations of inference-time reward alignment in diffusion models:
(1) why SMC-based methods (Fig.~\ref{fig:synth}C, D) can fail when initialization is poor, particularly because high-reward regions are too \textit{rare} to be adequately covered by standard priors (Section~\ref{sec:smc_fail}), and
(2) why $\Psi$-sampler and its multi-chain extensions (Fig.~\ref{fig:synth}E, F) can also struggle despite explicitly improving initialization, as local MCMC transitions are highly vulnerable to mode-trapping in \textit{multi-modal} landscapes (Section~\ref{sec:psi_fail}).
Addressing these sequential bottlenecks ultimately motivates our proposed PATHS (Fig.~\ref{fig:synth}G).

\subsection{Why Does Using Only SMC Fail? Why Initial Noise Sampling Matters}
\label{sec:smc_fail}
While \citet{yoon2025psi} emphasize the importance of initial noise sampling in addressing vanishing guidance strength across denoising steps, we argue that its role becomes even more critical in rare, high-reward settings.

SMC-based methods are asymptotically exact \citep{smc_wu2023practical}, but this guarantee applies only in the infinite-particle limit. In practical regimes with limited particles, approximation quality can degrade substantially. In particular, when the proposal and target distributions are highly mismatched, the number of particles required can scale exponentially with their KL divergence \citep{chatterjee2018sample}. Twisting partially alleviates this issue by improving intermediate targets and proposals \citep{smc_wu2023practical}, but we argue that a fundamental bottleneck remains at the very first step.

Specifically, SMC-based methods initialize particles from the prior distribution $p_1 = \mathcal{N}(0, I)$. Thus, performance critically depends on how well this prior aligns with the reward-aware target initialization distribution 
\begin{equation}
\pi_1^*(x_1) \propto p_1(x_1) \exp \left( r(\hat{x}_0(x_1))/ {\alpha} \right).    
\label{eq:pi_1}
\end{equation}
When high-reward regions occupy only a tiny fraction of prior mass—as is often the case under complex rewards that are sharp, multi-modal, or highly compositional (e.g., spatial layout constraints requiring a person to be generated on top of a car, as in Fig.~\ref{fig:qualitative_results})—the initial KL divergence becomes large. In such cases, even perfectly designed later-stage twisting cannot recover from poor initialization without an impractically large particle budget. We formalize this below.

\begin{proposition}[Large Initial KL under complex reward]
\label{prop:init}
Assume there exists a high-reward region $B \subset \mathcal{X}_0$ in the data space. Let $\nu$ be the pushforward measure of the prior $p_1$ under the Tweedie estimator $\hat{x}_0$, such that the probability of initial particles falling into the high-reward region is $\nu(B) = \mathbb{P}_{x_1 \sim p_1}(\hat{x}_0(x_1) \in B) = \varepsilon \ll 1$. Furthermore, assume there is a reward gap $\Delta > 0$ separating $B$ from its complement. Let $\pi_1^*$ denote the reward-aware initial distribution defined in Eq.~\eqref{eq:pi_1}. If $\Delta$ is sufficiently large, satisfying $\Delta/\alpha \ge \log(1/\varepsilon)+C$ for some constant $C>0$, then:
\begin{equation}
D_{\mathrm{KL}}(\pi_1^* \| p_1) = \Omega\left(\log \frac{1}{\varepsilon}\right).
\end{equation}
\end{proposition}See Appendix~\ref{appendix:proof} for a proof. Since particle complexity scales exponentially with KL divergence \citep{chatterjee2018sample}, even asymptotically exact SMC becomes practically ineffective unless initialization is fundamentally improved. 

Crucially, this bottleneck persists even when advanced tempering schedules are employed. For instance, while a tempering approach like DAS \citep{smc_kim2025testtime} theoretically improves sample efficiency, its practical effectiveness remains severely constrained. Due to the finite number of sampling steps and the intractability of finding an optimal tempering path, even tempered SMC cannot fully bridge the vast KL divergence inherent to structurally complex rewards. Consequently, as shown in Fig.~\ref{fig:synth}D, these methods continue to waste particles in suboptimal regions and struggle to consistently discover the global optimum within a limited budget. This explains why improving the initialization itself is far more critical than simply smoothing the target distribution over time.

\subsection{Why Does $\Psi$-sampler Fail? Mode Trapping Problem}
\label{sec:psi_fail}

We argue that $\Psi$-sampler \citep{yoon2025psi} can fail even if it aims to sample initial noise from the reward-aware posterior $\pi_1^*$, especially under complex reward landscapes. To approximate $\pi_1^*$, $\Psi$-sampler employs the Preconditioned Crank--Nicolson Langevin (pCNL) algorithm \citep{cotter2013mcmc}. Specifically, for a current state $x$ and step size $h > 0$, the pCNL proposal $x'$ is generated as:
\begin{equation}
    x' = \rho x + \sqrt{1 - \rho^2} \left( \xi + \frac{\sqrt{h}}{2} \frac{\nabla r(\hat{x}_0(x))}{\alpha} \right), \quad \xi \sim \mathcal{N}(0, I),
    \label{eq:pcnl_proposal_main}
\end{equation}
where $\rho = (1 - h/4)/(1 + h/4)$. 
While pCNL is dimension-robust and effective for high-dimensional latent spaces, it remains fundamentally a \textit{local} MCMC method. As visually demonstrated in Fig.~\ref{fig:synth}E, this restricted exploration leads to a critical mode-trapping problem: the Markov chain's trajectory is largely dictated by its initialization, often exhibiting metastability within suboptimal regions. In complex, sharp reward landscapes, escaping a local mode (e.g., correctly matching only two out of four bounding boxes in a layout task) requires large-scale proposals that are inherently incompatible with the moderate step-size $h$ necessary to maintain stable acceptance rates. Consequently, any attempt to traverse high-energy barriers results in vanishingly low acceptance probabilities, as the target density $\pi_1^*$ drops precipitously outside the narrow reward peaks. This effectively confines the chain to the initial basin of attraction, preventing the discovery of the global optimum regardless of the sampling budget.

One might consider the Best-of-$L$ strategy (Fig.~\ref{fig:synth}F)—running $L$ independent $\Psi$-sampler chains and selecting the one with the highest reward—as a potential remedy. However, this fails to fundamentally resolve the exploration bottleneck. Such a brute-force approach relies on the "lucky" chance that at least one initialization lands near the global optimum—a probability that vanishes rapidly as the reward landscape grows complex. Thus, overcoming this bottleneck necessitates a more systematic approach.
\section{PATHS: Parallel Tempering for High-complexity Reward Sampling}
\label{sec:paths}

To address the mode-trapping and exploration bottlenecks discussed in Section~\ref{sec:psi_fail}, we propose \textbf{PATHS} (\textbf{PA}rallel \textbf{T}empering for \textbf{H}igh-complexity reward \textbf{S}ampling). Unlike $\Psi$-sampler, which relies on a single local Markov chain, PATHS employs a systematic global exploration strategy by constructing a sequence of tempered target distributions.

\subsection{PATHS: Parallel-Tempering pCNL}

PATHS leverages the efficiency of pCNL for high-dimensional spaces while overcoming its locality through a parallel tempering (PT) framework \citep{swendsen1986replica}. The core idea is to interpolate between the tractable prior $\pi_1$ (where exploration is easy) and the sharp, complex reward-aware posterior $\pi_1^*$ (where the global optimum resides).

\paragraph{Tempered ladder.} PATHS maintains $L$ chains targeting a sequence of tempered distributions:
\begin{equation}
\pi^*_{1, \ell}(x) \propto p_1(x) \exp \! \left(\frac{r(\hat{x}_0(x_1))}{\alpha \cdot T_\ell}\right), \quad \ell = 1, \dots, L, \label{eq:tempered_posterior}
\end{equation}
with a geometric temperature ladder $T_1 = 1 < T_2 < \cdots < T_L$. This setup assigns distinct roles to each chain (as illustrated in Fig.~\ref{fig:PT_mechanism}):

\begin{itemize}
    \item \textbf{Low-temperature (Cold) Chains} (Small $\ell$): These chains, including the target chain where $T_1=1$, are dedicated to high-fidelity sampling from the reward-aware posterior $\pi^*_1$. While they effectively refine samples within precise modes, they are highly susceptible to \textit{mode-trapping} within the rugged reward landscape, often leading to frequent rejections and an inability to escape local optima.
    
    \item \textbf{High-temperature (Hot) Chains} (Large $\ell$): By employing elevated temperatures, these chains effectively \textit{flatten the reward landscape} to facilitate global exploration. However, due to this reduced sensitivity to the reward signal, they lack the stability to remain on any single high-reward mode; the discovered states continuously shift rather than being refined, necessitating a swap to transfer these findings to a more stable chain.
\end{itemize}

We adopt a geometric ladder $T_\ell = T_{\mathrm{ratio}}^{\ell - 1}$ 
following the standard parallel-tempering practice, which provides 
approximately uniform overlap between adjacent 
chains~\citep{predescu2004incomplete}.

\begin{figure}
  \centering
  % figure
  \begin{minipage}[b]{0.48\textwidth}
    \centering
    \includegraphics[height=5.2cm, width=\linewidth, keepaspectratio]{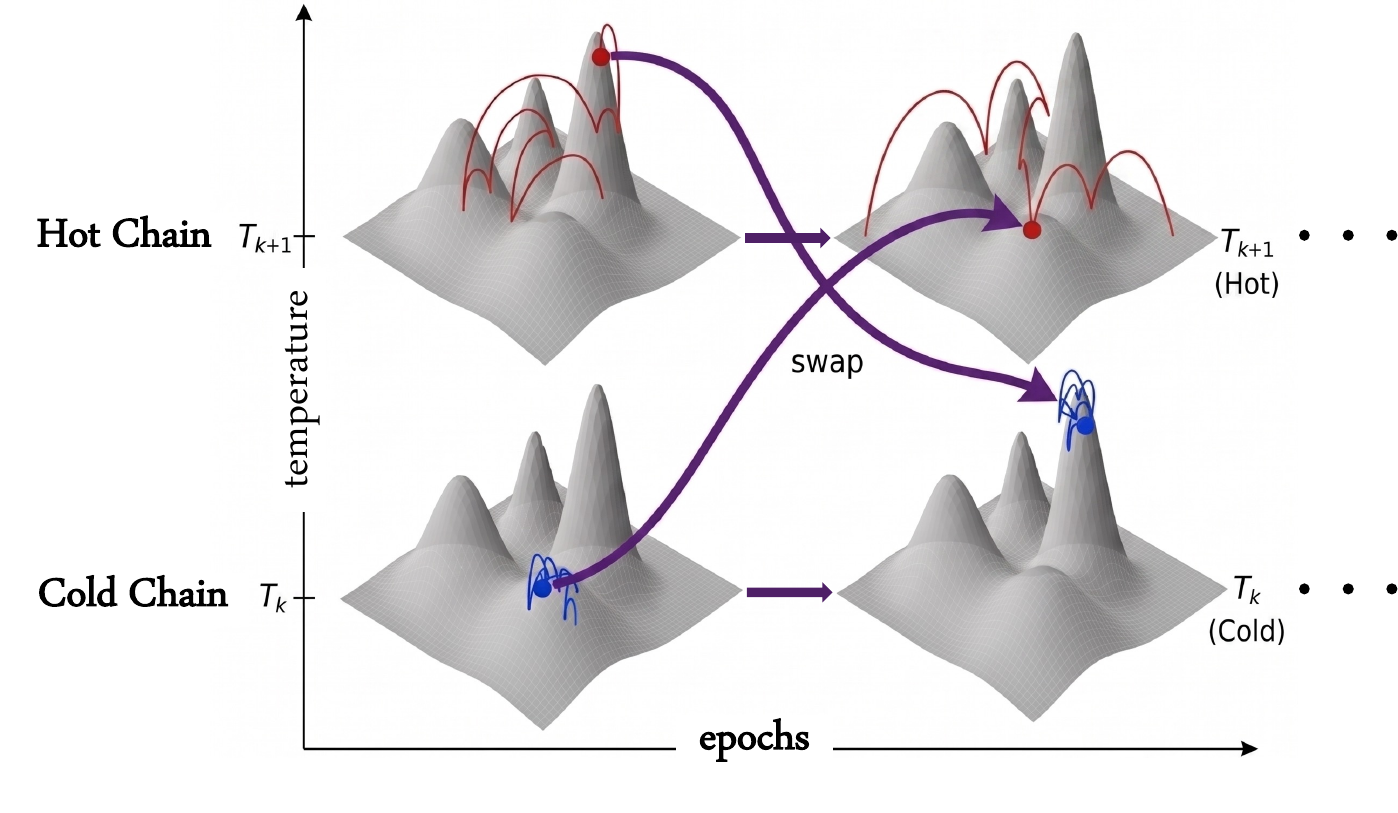}
    \captionsetup{width=0.95\linewidth} 
    \caption{Visualization of the Parallel Tempering mechanism. Hot chains explore flattened reward landscapes, while accepted moves propagate states.}
    \label{fig:PT_mechanism}
  \end{minipage}\hfill
  % figure 
  \begin{minipage}[b]{0.48\textwidth}
    \centering
    \includegraphics[height=3.5cm, width=\linewidth, keepaspectratio]{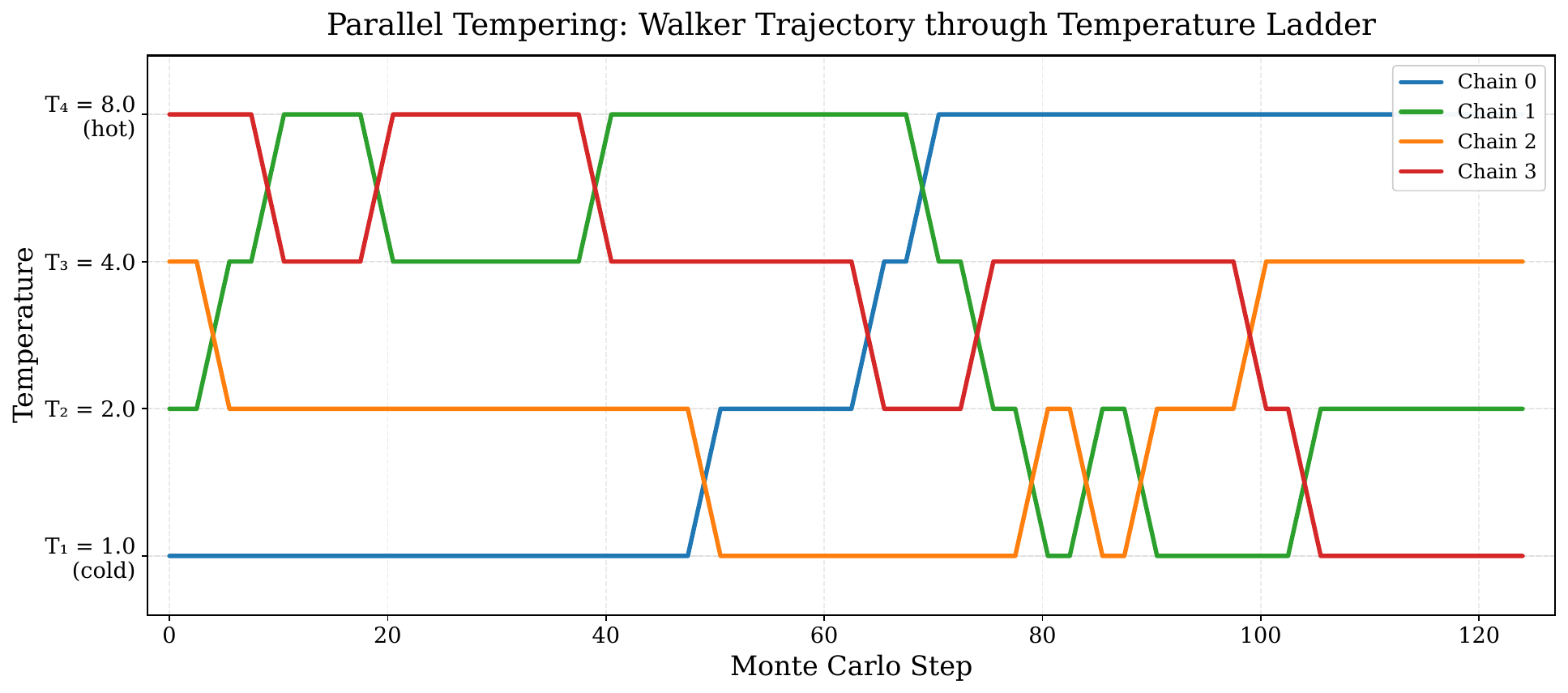}
    \captionsetup{width=0.95\linewidth}
    \caption{Visualization of replica-exchange swap trajectories 
across the temperature ladder $T \in \{1,2,4,8\}$, recorded 
from an actual PATHS run on a single layout-to-image instance. 
Swap proposals are attempted every $m=5$ steps; accepted swaps 
allow chains to move between temperature levels, enabling 
information exchange.}
    \label{fig:PT_trajectory}
  \end{minipage}
\end{figure}

\paragraph{Replica Exchange (Swap).} 
PATHS periodically attempts to swap states between adjacent chains $\ell$ and $\ell + 1$ every $m$ steps (see Fig.~\ref{fig:PT_trajectory} for trajectory visualization). The exchange is governed by the Metropolis-Hastings acceptance ratio. For two states $x_\ell$ and $x_{\ell+1}$ in adjacent chains, the swap probability is defined as:
\begin{equation}
A(x_\ell, x_{\ell+1}) = \min \left( 1, \frac{\pi^*_{\ell}(x_{\ell+1})\pi^*_{\ell+1}(x_{\ell})}{\pi^*_{\ell}(x_{\ell})\pi^*_{\ell+1}(x_{\ell+1})} \right)
\end{equation}
Substituting the tempered distribution from Eq.~\eqref{eq:tempered_posterior}, the prior terms $p_1(x)$ cancel out, and the ratio simplifies to:
\begin{equation}
A(x_\ell, x_{\ell+1}) = \min \left( 1, \exp \left[ \left( \frac{1}{T_\ell} - \frac{1}{T_{\ell+1}} \right) \left( \frac{r(\hat{x}_0(x_{\ell+1}))}{\alpha} - \frac{r(\hat{x}_0(x_\ell))}{\alpha} \right) \right] \right)
\end{equation}

This criterion establishes a \textit{reward-driven sorting} mechanism across the temperature ladder. Samples with higher rewards are preferentially transitioned to colder chains for local refinement near the target modes, while lower-reward samples are moved to warmer chains. In these warmer chains, the flattened energy landscape facilitates more effective mode-hopping and global exploration, ultimately concentrating the particle population in the high-density regions of the target distribution $\pi_1^*$.

After the parallel tempering process concludes, we collect $K$ samples specifically from the coldest chain ($T_1 = 1$), which directly targets the reward-aware posterior $\pi_1^*$ defined in Eq.~\ref{eq:pi_1}. We then proceed with the final sampling process using SMC as described in Section~\ref{sec:pre_SMC}. As visually demonstrated in Fig.~\ref{fig:synth}G, while other methods struggle with mode-trapping, PATHS is the only approach that consistently escapes suboptimal local modes with closely matching the ground truth. The detailed procedure is provided in Algorithm~\ref{alg:paths} in Appendix~\ref{app:paths_algorithm}.

\section{Experiments}
\label{sec:experiments}

We evaluate PATHS on two compositional reward-alignment tasks:
layout-to-image generation and quantity-aware sampling. Both tasks induce
multi-modal reward landscapes: multiple object configurations can satisfy
the same spatial layout, and many spatial arrangements can realize the
same target count. We therefore use these tasks to test whether
parallel-tempering initialization improves finite-budget exploration
beyond independent posterior-initialization baselines.

\subsection{Experimental Setup}
\label{sec:experiment_setup}

\paragraph{Tasks.}
In the layout-to-image task, the goal is to generate images that place
specified objects in target bounding boxes while preserving text-image
alignment and image quality. In the quantity-aware task, the goal is to
generate a specified number of target objects. For both tasks, we report
results on the full benchmark and on difficulty-stratified simple and
complex subsets. The complex subsets contain prompts with more demanding
compositional structure, such as multiple objects, spatial relations, or
larger target counts. Further benchmark and subset details are provided in Appendix~\ref{app:experiment_details}.

\paragraph{Baselines and compute.}
We compare PATHS against prior-initialized reward-guided SMC baselines,
including TDS~\citep{smc_wu2023practical} and DAS~\citep{smc_kim2025testtime},
and posterior-initialization baselines, including
\(\Psi\)-Sampler~\citep{yoon2025psi} and Best-of-4. \(\Psi\)-Sampler
uses independent same-temperature pCNL chains, while Best-of-4 runs four
independent cold pCNL chains and forwards the highest-reward chain to the
SMC stage. We use Best-of-4 because PATHS also uses four parallel chains,
making the comparison compute-matched. All experiments use
FLUX.1-schnell~\citep{flux2024} and a total budget of \(1000\)
reward-model evaluations, split evenly between initial sampling and the
downstream SMC stage. Thus, improvements over \(\Psi\)-Sampler and
Best-of-4 reflect the effect of tempering and online replica exchange
rather than additional reward-model evaluations. Full hyperparameters,
temperature ladders, and SMC details are provided in
Appendix~\ref{app:experiment_details}.

\paragraph{Evaluation metrics.}
For layout-to-image generation, we report the seen GroundingDINO-based
layout reward, held-out mIoU, ImageReward, and VQAScore. For
quantity-aware sampling, we report the seen T2I-Count error together with
ImageReward and VQAScore. Metrics used directly for guidance are marked
with \(\dagger\), while the remaining metrics are used only for
evaluation.

\subsection{Experiment Results}
\label{sec:experiment_results}

\paragraph{Qualitative Results}\label{sec}

Figure~\ref{fig:qualitative_results} provides representative qualitative comparisons for both tasks. PATHS better satisfies complex layout and count constraints while preserving visual fidelity. Additional qualitative examples, including prompt-level comparisons and failure cases, are provided in Appendix~\ref{qualitative_results}.

\begin{figure}[t]
  \centering
  \setlength{\tabcolsep}{1.5pt} 
  
  \begin{tabular}{c ccccc}
    %\toprule
    % (Prior / Posterior)
    & \multicolumn{2}{c}{\small \textbf{Sampling from Prior}} 
    & \multicolumn{3}{c}{\small \textbf{Sampling from Posterior}} \\
    \cmidrule(lr){2-3} \cmidrule(lr){4-6}
    
    & \footnotesize TDS & \footnotesize DAS & \footnotesize $\psi$-Sampler & \footnotesize Best-of-4 & \footnotesize \textbf{PATHS (Ours)} \\
    %\midrule
    
    % ==========================================
    % 1. Layout-to-Image Task
    % ==========================================
    \adjustbox{valign=m}{\rotatebox{90}{\scriptsize \textbf{Layout-to-Image}}} &
    \includegraphics[width=0.19\textwidth, valign=m]{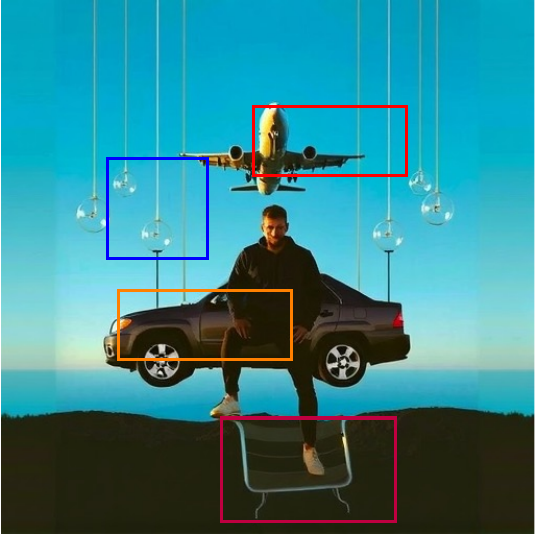} &
    \includegraphics[width=0.19\textwidth, valign=m]{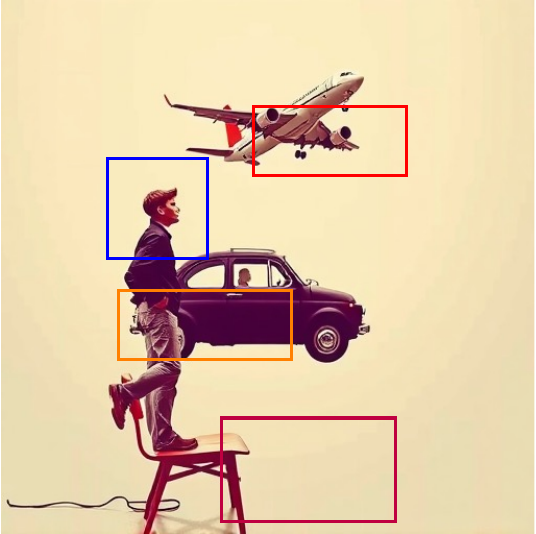} &
    \includegraphics[width=0.19\textwidth, valign=m]{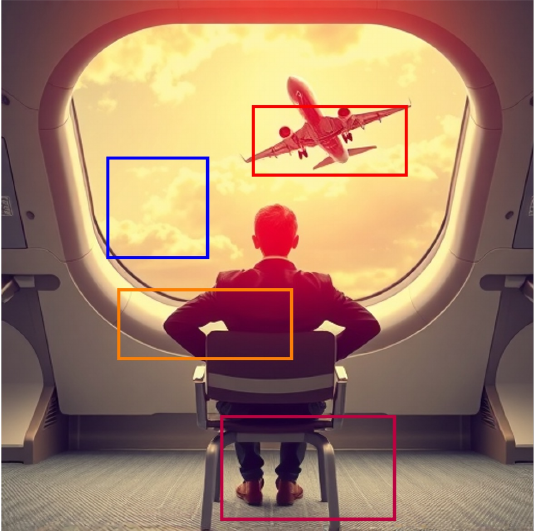} &
    \includegraphics[width=0.19\textwidth, valign=m]{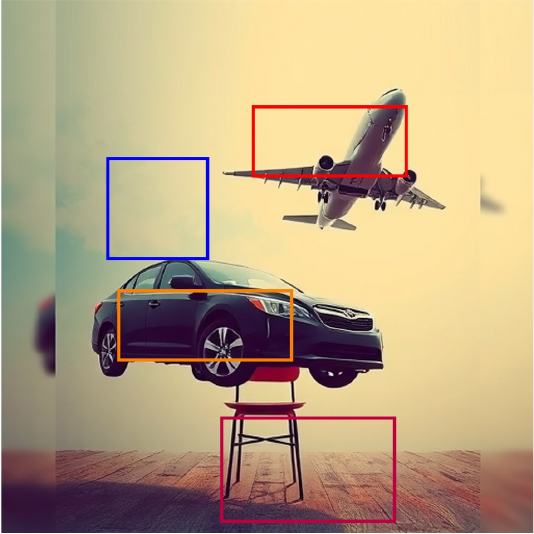} &
    \includegraphics[width=0.19\textwidth, valign=m]{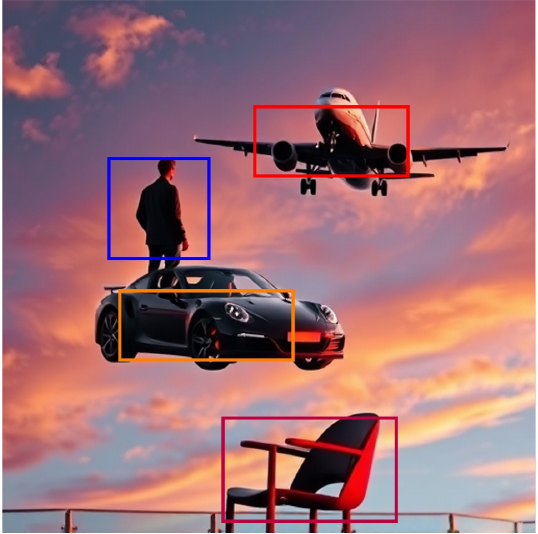} \\
        
    % Layout-to-Image 
      & \multicolumn{5}{c}{\textit{``A \textcolor{blue}{person} and \textcolor{red}{airplane} over a \textcolor{orange}{car} and under the \textcolor{purple}{chair}....''}} \\
      \addlinespace[10pt] 
    
    %\midrule
    
    % ==========================================
    % 2. Quantity-Aware Task (total 37)   
    % ==========================================
    \adjustbox{valign=m}{\rotatebox{90}{\scriptsize \textbf{Quantity-Aware}}} &
    \includegraphics[width=0.19\textwidth, valign=m]{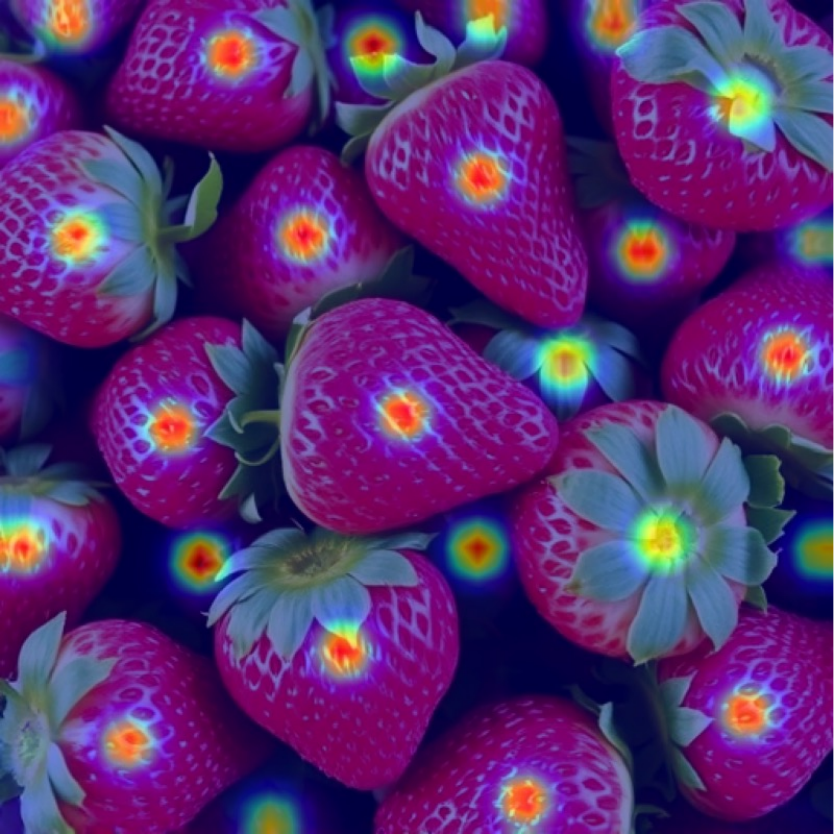} & % 24. delta: 13
    \includegraphics[width=0.19\textwidth, valign=m]{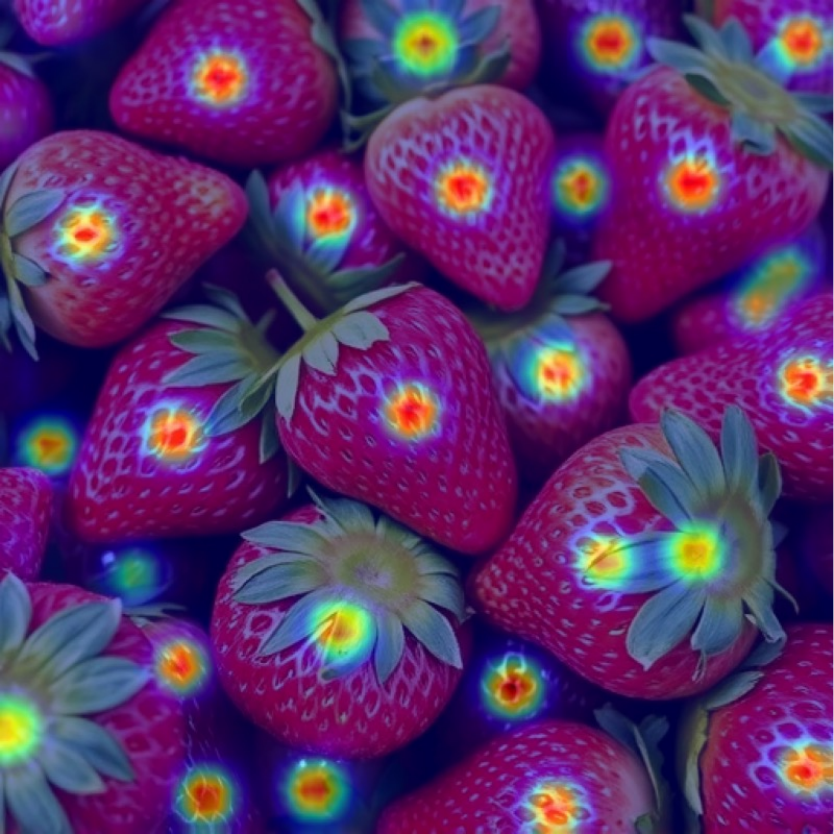} & % 25. delta: 12
    \includegraphics[width=0.19\textwidth, valign=m]{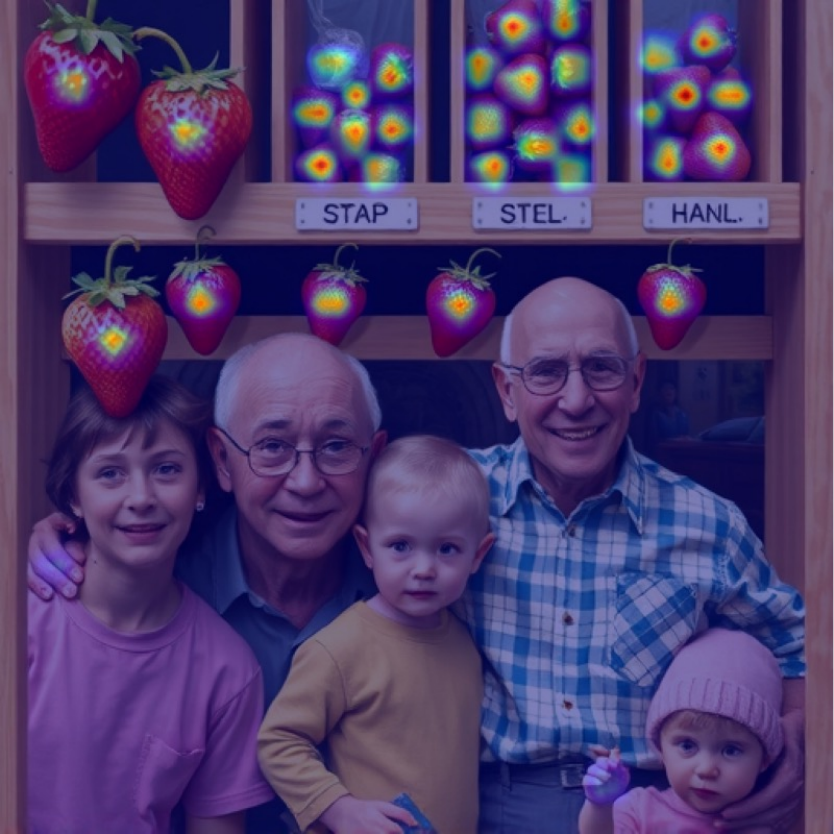} & % 33. delta: 4
    \includegraphics[width=0.19\textwidth, valign=m]{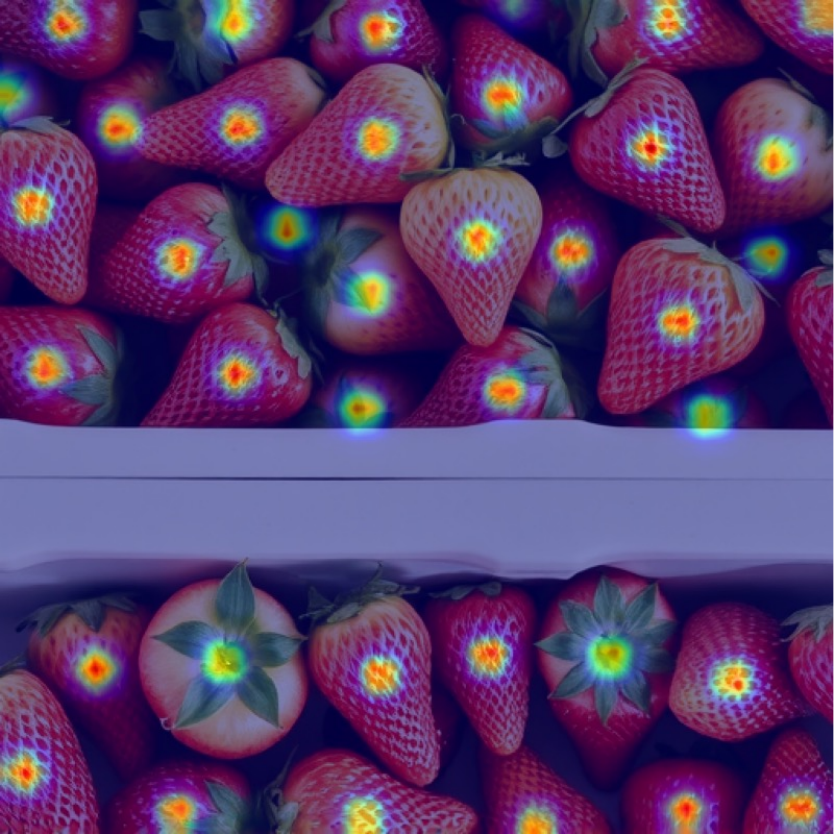} & % 35. delta:2 
    \includegraphics[width=0.19\textwidth, valign=m]{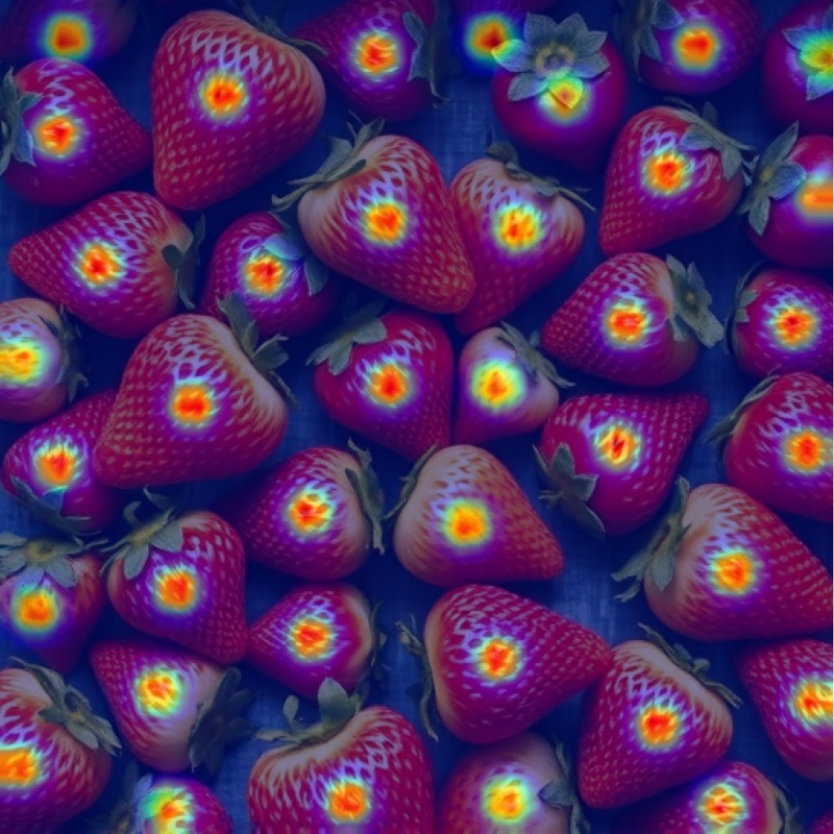} \\ % 37. delta: 0
    & 24 \small($\Delta$13) & 25 \small($\Delta$12) & 33 \small($\Delta$4) & 35 \small($\Delta$2) & 37 \small(\textcolor{blue}{$\Delta$0)} \\
    
    % Quantity-Aware 
    & \multicolumn{5}{c}{ ``\textit{\textcolor{red}{37} strawberries}''} \\
    
    %\bottomrule
  \end{tabular}
  
  %\vspace{1mm} 
  \caption{Qualitative comparison on Layout-to-Image and Quantity-Aware tasks. We compare sampling from prior (TDS, DAS) and sampling from posterior ($\psi$-Sampler, Best-of-4, \textbf{PATHS}). PATHS effectively aligns with complex geometric layouts and exact object counts without degrading image quality.}
  \label{fig:qualitative_results}
\end{figure}
\begin{table}[t]
  \caption{
  Quantitative comparison on \textbf{Layout-to-Image} and
  \textbf{Quantity-Aware Sampling}. We report results on the full benchmark
  and difficulty-stratified \textit{Complex} and \textit{Simple} subsets.
  Metrics marked with $\dagger$ are used as seen rewards during reward-guided
  sampling, while the remaining metrics are held-out evaluations.
  Best results are shown in bold and second-best results are underlined.
  }
  \label{tab:table_total}
  \centering
  \setlength{\tabcolsep}{3pt}
  \footnotesize
  %\scriptsize
  \begin{tabular}{c c c ccccc}
    \toprule
    \multirow{2}{*}{Task} & \multirow{2}{*}{Subset} & \multirow{2}{*}{Metric}
      & \multicolumn{2}{c}{Sampling from Prior}
      & \multicolumn{3}{c}{Sampling from Posterior} \\
    \cmidrule(lr){4-5} \cmidrule(lr){6-8}
    & & &
    TDS~\citep{smc_wu2023practical}
    & DAS~\citep{smc_kim2025testtime}
    & $\psi$-Sampler~\citep{yoon2025psi}
    & Best-of-4
    & \textbf{PATHS} \\
    \midrule

    \multirow{12}{*}{\shortstack[c]{Layout-\\to-Image}}
    & \multirow{4}{*}{\shortstack[c]{Overall\\($N=50$)}}
    & GroundingDINO$^\dagger$~\citep{liu2024grounding} $\uparrow$
        & 0.352 & 0.297 & 0.348 & \underline{0.363} & \textbf{0.376} \\
    \cmidrule(lr){3-8}
    & & mIoU~\citep{hou2024salience} $\uparrow$
        & 0.348 & 0.299 & 0.325 & \underline{0.355} & \textbf{0.366} \\
    & & ImageReward~\citep{xu2023imagereward} $\uparrow$
        & 0.827 & \underline{0.872} & 0.705 & 0.826 & \textbf{0.949} \\
    & & VQA~\citep{lin2024evaluating} $\uparrow$
        & 0.766 & \underline{0.774} & 0.738 & 0.771 & \textbf{0.797} \\
    \cmidrule(lr){2-8}

    & \multirow{4}{*}{\shortstack[c]{Complex\\($N=20$)}}
    & GroundingDINO$^\dagger$~\citep{liu2024grounding} $\uparrow$
        & 0.261 & 0.226 & 0.239 & \underline{0.265} & \textbf{0.266} \\
    \cmidrule(lr){3-8}
    & & mIoU~\citep{hou2024salience} $\uparrow$
        & \textbf{0.287} & 0.265 & 0.266 & 0.273 & \underline{0.274} \\
    & & ImageReward~\citep{xu2023imagereward} $\uparrow$
        & \underline{1.034} & 0.861 & 0.625 & 0.781 & \textbf{1.411} \\
    & & VQA~\citep{lin2024evaluating} $\uparrow$
        & \underline{0.618} & 0.608 & 0.580 & 0.588 & \textbf{0.706} \\
    \cmidrule(lr){2-8}

    & \multirow{4}{*}{\shortstack[c]{Simple\\($N=30$)}}
    & GroundingDINO$^\dagger$~\citep{liu2024grounding} $\uparrow$
        & 0.412 & 0.344 & 0.420 & \underline{0.428} & \textbf{0.450} \\
    \cmidrule(lr){3-8}
    & & mIoU~\citep{hou2024salience} $\uparrow$
        & 0.390 & 0.321 & 0.365 & \underline{0.411} & \textbf{0.428} \\
    & & ImageReward~\citep{xu2023imagereward} $\uparrow$
        & 0.689 & \textbf{0.879} & 0.764 & \underline{0.856} & 0.820 \\
    & & VQA~\citep{lin2024evaluating} $\uparrow$
        & 0.864 & \underline{0.885} & 0.844 & \textbf{0.893} & 0.858 \\
    \midrule

    \multirow{9}{*}{\shortstack[c]{Quantity-\\Aware\\Sampling}}
    & \multirow{3}{*}{\shortstack[c]{Overall\\($N=60$)}}
    & T2I-Count$^\dagger$~\citep{qian2025t2icount} $\downarrow$
        & 2.083 & 1.786 & 2.412 & \underline{1.613} & \textbf{1.344} \\
    \cmidrule(lr){3-8}
    & & ImageReward~\citep{xu2023imagereward} $\uparrow$
        & -0.243 & -0.331 & 0.048 & \underline{0.095} & \textbf{0.244} \\
    & & VQAScore~\citep{lin2024evaluating} $\uparrow$
        & 0.549 & 0.553 & 0.645 & \underline{0.652} & \textbf{0.669} \\
    \cmidrule(lr){2-8}

    & \multirow{3}{*}{\shortstack[c]{Complex\\($N=30$)}}
    & T2I-Count$^\dagger$~\citep{qian2025t2icount} $\downarrow$
        & 3.250 & 3.108 & 4.181 & \underline{2.463} & \textbf{1.725} \\
    \cmidrule(lr){3-8}
    & & ImageReward~\citep{xu2023imagereward} $\uparrow$
        & -0.300 & -0.325 & 0.104 & \underline{0.170} & \textbf{0.202} \\
    & & VQAScore~\citep{lin2024evaluating} $\uparrow$
        & 0.546 & 0.555 & \underline{0.616} & 0.608 & \textbf{0.657} \\
    \cmidrule(lr){2-8}

    & \multirow{3}{*}{\shortstack[c]{Simple\\($N=30$)}}
    & T2I-Count$^\dagger$~\citep{qian2025t2icount} $\downarrow$
        & 1.305 & 0.904 & \textbf{0.644} & \underline{0.763} & 0.963 \\
    \cmidrule(lr){3-8}
    & & ImageReward~\citep{xu2023imagereward} $\uparrow$
        & -0.159 & -0.339 & -0.007 & \underline{0.019} & \textbf{0.286} \\
    & & VQAScore~\citep{lin2024evaluating} $\uparrow$
        & 0.555 & 0.550 & 0.673 & \textbf{0.696} & \underline{0.681} \\
    \bottomrule
  \end{tabular}
\end{table}

\paragraph{Layout-to-Image Generation}
\label{sec:layout_results}

Table~\ref{tab:table_total} reports quantitative results on the
layout-to-image task. PATHS improves performance most clearly on the
complex subset, where prompts contain multiple objects and spatial
relations. Compared with independent posterior-initialization baselines,
PATHS improves held-out image quality and text-image alignment while
remaining competitive on layout metrics. The comparison with Best-of-4 is
particularly important: selecting the best independent chain after
sampling cannot replace the online information exchange induced by
replica-exchange swaps.

\paragraph{Quantity-Aware Sampling}
\label{sec:quantity_results}

Table~\ref{tab:table_total} reports quantitative results on
quantity-aware sampling. The gains are strongest on the complex subset,
where target counts are harder to realize under a fixed inference-time
budget. PATHS substantially reduces count error compared with
\(\Psi\)-Sampler and Best-of-4 while also improving held-out preference
and text-image alignment metrics. This indicates that the benefit is not
simply due to running multiple independent chains, but to tempering and
online inter-chain communication.

\paragraph{Takeaways.}
Across both tasks, the largest gains appear in the complex subsets, where
the reward landscapes contain more separated high-reward modes induced by
multi-object layouts or difficult counting constraints. This supports the
central hypothesis of PATHS: under a matched reward-model evaluation
budget, replacing independent same-temperature chains with a tempered,
swap-coupled ensemble improves finite-budget exploration. The gains on
held-out metrics further suggest that PATHS does not merely optimize the
seen reward, but provides better initial particles for the downstream SMC
stage.
\section{Conclusion}
\label{sec:conclusion}

We introduced PATHS, an inference-time initialization framework that improves reward alignment in diffusion models under rare and multi-modal reward landscapes. By leveraging Parallel Tempering, PATHS enables more effective exploration of complex reward-aware posteriors and produces stronger initial particles for downstream SMC-based alignment. Across layout-to-image and quantity-aware generation tasks, PATHS consistently outperforms prior SMC-based methods and $\Psi$-Sampler, demonstrating that robust initialization and cross-mode exploration are critical for inference-time reward alignment in high-complexity settings.

% \begin{ack}
% Use unnumbered first level headings for the acknowledgments. All acknowledgments

% Do {\bf not} include this section in the anonymized submission, only in the final paper. You can use the \texttt{ack} environment provided in the style file to automatically hide this section in the anonymized submission.
% \end{ack}

%%%%%%%%%%%%%%%%%%% reference %%%%%%%%%%%%%%%%%%%%%%%%%%%%%%%%%%%%%%%%%
%\newpage
\bibliographystyle{plainnat}
\bibliography{reference}

%%%%%%%%%%%%%%%%%%%%%%%%%%%%%%%%%%%%%%%%%%%%%%%%%%%%%%%%%%%%%%%%%%%%%%%%
\newpage

\appendix
\newpage

\section{Temperature-Scaled pCNL Transition}
\label{app:pcnl_temperature_scaled_transition}

We describe the within-temperature pCNL transition used in PATHS and
show how the temperature scaling enters both the proposal drift and the
Metropolis-Hastings correction. For each temperature level \(T_\ell\), the
tempered initial posterior is
\begin{equation}
\pi^*_{1,\ell}(x)
=
\frac{1}{Z_\ell}
p_1(x)
\exp\left(
\frac{r(\hat{x}_0(x_1))}{\alpha T_\ell}
\right),
\label{eq:app_tempered_initial_posterior}
\end{equation}
where \(p_1(x)=\mathcal{N}(\mathbf{0},\mathbf{I})\),
\(\hat{x}_0(x_1)\) is the one-step denoised estimate from
\(x\), \(r(\hat{x}_0(x_1))\) is the reward, \(\alpha\)
controls the reward strength, and \(T_\ell\) is the temperature. The cold
chain has \(T_1=1\) and therefore targets the original reward-aware
posterior, while larger \(T_\ell\) values flatten the reward contribution
and define hotter chains.

Let
\begin{equation}
\Phi_\ell(x)
=
\frac{r(\hat{x})}{\alpha T_\ell}
\label{eq:app_temperature_scaled_reward_potential}
\end{equation}
denote the temperature-scaled reward potential, and let
\begin{equation}
\mathbf{g}_\ell(x)
=
\nabla_{x}\Phi_\ell(x)
=
\frac{1}{\alpha T_\ell}
\nabla_{x} r(\hat{x}_0(x_1))
\label{eq:app_temperature_scaled_reward_gradient}
\end{equation}
be the corresponding temperature-scaled reward gradient. Thus, hotter
chains use weaker reward gradients, which allows them to explore a flatter
reward landscape.

Given the current state \(x_k^{(\ell)}\) of chain \(\ell\) at
MCMC iteration \(k\), the pCNL proposal is
\begin{equation}
x' =
\rho_\ell x_k^{(\ell)}
+
\sqrt{1-\rho_\ell^2}
\left(
\boldsymbol{\xi}
+
\frac{\sqrt{s_\ell}}{2}
\mathbf{g}_\ell(x_k^{(\ell)})
\right),
\qquad
\boldsymbol{\xi}\sim\mathcal{N}(\mathbf{0},\mathbf{I}),
\label{eq:app_pcnl_proposal}
\end{equation}
where \(s_\ell>0\) is the Langevin step size and
\(\rho_\ell\in(0,1)\) controls the pCN persistence. Equivalently, this
proposal has Gaussian transition density
\begin{equation}
q_\ell(x'\mid x)
=
\mathcal{N}
\left(
x';
\,
\rho_\ell x
+
\frac{\sqrt{1-\rho_\ell^2}\sqrt{s_\ell}}{2}
\mathbf{g}_\ell(x),
\,
(1-\rho_\ell^2)\mathbf{I}
\right).
\label{eq:app_pcnl_proposal_density}
\end{equation}

The proposed state is accepted with the Metropolis-Hastings probability
\begin{equation}
a_\ell(x,x')
=
\min\left\{
1,\,
\frac{
\pi^*_{1,\ell}(x')
q_\ell(x\mid x')
}{
\pi^*_{1,\ell}(x)
q_\ell(x'\mid x)
}
\right\}.
\label{eq:app_pcnl_mh_probability}
\end{equation}
Substituting Eq.~\eqref{eq:app_tempered_initial_posterior}, the log
acceptance ratio becomes
\begin{align}
\log A_\ell(x,x')
&=
\Phi_\ell(x')
-
\Phi_\ell(x)
+
\log p_1(x')
-
\log p_1(x)
+
\log q_\ell(x\mid x')
-
\log q_\ell(x'\mid x)
\nonumber \\
&=
\frac{
r(\hat{x}'_0(x_1))
-
r(\hat{x}_0(x_1))
}{\alpha T_\ell}
+
\log p_1(x')
-
\log p_1(x)
+
\log q_\ell(x\mid x')
-
\log q_\ell(x'\mid x).
\label{eq:app_pcnl_log_acceptance_general}
\end{align}
Since \(p_1=\mathcal{N}(\mathbf{0},\mathbf{I})\), the prior contribution is
\begin{equation}
\log p_1(x')
-
\log p_1(x)
=
-\frac{1}{2}\|x'\|_2^2
+
\frac{1}{2}\|x\|_2^2.
\label{eq:app_pcnl_prior_ratio}
\end{equation}
The proposal-density correction is
\begin{align}
\log q_\ell(x\mid x')
-
\log q_\ell(x'\mid x)
&=
-\frac{1}{2(1-\rho_\ell^2)}
\left\|
x
-
\rho_\ell x'
-
\frac{\sqrt{1-\rho_\ell^2}\sqrt{s_\ell}}{2}
\mathbf{g}_\ell(x')
\right\|_2^2
\nonumber \\
&\quad
+
\frac{1}{2(1-\rho_\ell^2)}
\left\|
x'
-
\rho_\ell x
-
\frac{\sqrt{1-\rho_\ell^2}\sqrt{s_\ell}}{2}
\mathbf{g}_\ell(x)
\right\|_2^2 .
\label{eq:app_pcnl_proposal_ratio}
\end{align}
Combining Eqs.~\eqref{eq:app_pcnl_log_acceptance_general},
\eqref{eq:app_pcnl_prior_ratio}, and
\eqref{eq:app_pcnl_proposal_ratio} yields the temperature-scaled
Metropolis-Hastings correction used by each PATHS replica.

This correction is important for two reasons. First, the reward difference
is divided by \(\alpha T_\ell\), so hotter chains are less sharply biased
toward local reward peaks. Second, the proposal drift
\(\mathbf{g}_\ell(x)\) is also scaled by \(1/(\alpha T_\ell)\),
so hotter chains take weaker reward-gradient-guided moves and can explore
the prior geometry more broadly. The Metropolis-Hastings correction then
ensures that, despite this biased proposal, the Markov kernel at
temperature \(T_\ell\) leaves the intended tempered target
\(\pi^*_{1,\ell}\) invariant.

As a useful special case, if the Langevin drift is removed by setting
\(\mathbf{g}_\ell(x)=\mathbf{0}\), the proposal reduces to the
standard pCN proposal
\begin{equation}
x'
=
\rho_\ell x
+
\sqrt{1-\rho_\ell^2}\boldsymbol{\xi},
\qquad
\boldsymbol{\xi}\sim\mathcal{N}(\mathbf{0},\mathbf{I}).
\label{eq:app_pcn_proposal}
\end{equation}
This proposal is reversible with respect to the Gaussian prior
\(p_1\), so the prior and proposal-density terms cancel in the
Metropolis-Hastings ratio. The acceptance probability then simplifies to
\begin{equation}
a_\ell^{\mathrm{pCN}}(x,x')
=
\min\left\{
1,\,
\exp\left(
\frac{
r(\hat{x}'_0(x_1))
-
r(\hat{x}_0(x_1))
}{\alpha T_\ell}
\right)
\right\}.
\label{eq:app_pcn_acceptance}
\end{equation}
The pCNL transition used in PATHS can therefore be viewed as a
reward-gradient-informed extension of this prior-preserving pCN update,
with an additional Metropolis-Hastings correction that preserves the
temperature-specific target in Eq.~\eqref{eq:app_tempered_initial_posterior}.
\section{Derivation of the Replica-Exchange Swap Acceptance Rule}
\label{app:swap_acceptance_derivation}

We derive the Metropolis acceptance probability used for replica-exchange
moves in PATHS. PATHS uses an increasing temperature ladder
\[
1=T_1<T_2<\cdots<T_L,
\]
where \(T_1=1\) denotes the cold chain and larger \(T_\ell\) values
denote hotter chains. For each temperature level \(T_\ell\), the tempered
initial posterior is defined as
\begin{equation}
\pi^*_{1,\ell}(x)
=
\frac{1}{Z_\ell}
p_1(x)
\exp\left(
\frac{r(\hat{x}_0(x_1))}{\alpha T_\ell}
\right),
\label{eq:appendix_tempered_target}
\end{equation}
where \(p_1(x)=\mathcal{N}(\mathbf{0},\mathbf{I})\), \(Z_\ell\)
is the normalizing constant, \(r(\hat{x}_0(x_1))\) is the reward
computed on the denoised estimate, and \(\alpha\) controls the strength
of reward tempering. Larger \(T_\ell\) values flatten the reward
contribution and therefore define hotter, more exploratory chains.

The joint target distribution of the parallel tempering system is
\begin{equation}
\pi_{\mathrm{pt}}(x^{(1)},\ldots,x^{(L)})
=
\prod_{\ell=1}^{L}
\pi^*_{1,\ell}(x^{(\ell)}).
\label{eq:appendix_pt_joint_target}
\end{equation}

Consider swapping the states of two neighboring temperature levels
\(i\) and \(j=i+1\), where \(T_i<T_j\). Let
\(x^{(i)}\) and \(x^{(j)}\) denote the current states at
temperatures \(T_i\) and \(T_j\), respectively. The proposed swap is
\[
(x^{(i)},x^{(j)})
\mapsto
(x^{(j)},x^{(i)}).
\]
Since the swap proposal is symmetric, the Metropolis-Hastings acceptance
ratio is given by the ratio of the joint target density after and before
the swap:
\begin{equation}
A_{\mathrm{swap}}
=
\frac{
\pi^*_{1,i}(x^{(j)})
\pi^*_{1,j}(x^{(i)})
}{
\pi^*_{1,i}(x^{(i)})
\pi^*_{1,j}(x^{(j)})
}.
\label{eq:appendix_swap_ratio}
\end{equation}

Substituting Eq.~\eqref{eq:appendix_tempered_target} into
Eq.~\eqref{eq:appendix_swap_ratio}, we obtain
\begin{align}
\log A_{\mathrm{swap}}
&=
\log \pi^*_{1,i}(x^{(j)})
+
\log \pi^*_{1,j}(x^{(i)})
-
\log \pi^*_{1,i}(x^{(i)})
-
\log \pi^*_{1,j}(x^{(j)}) \nonumber \\
&=
\left[
\log p_1(x^{(j)})
+
\frac{r(\hat{x}^{(j)}_0(x_1))}{\alpha T_i}
-
\log Z_i
\right]
+
\left[
\log p_1(x^{(i)})
+
\frac{r(\hat{x}^{(i)}_0(x_1))}{\alpha T_j}
-
\log Z_j
\right] \nonumber \\
&\quad -
\left[
\log p_1(x^{(i)})
+
\frac{r(\hat{x}^{(i)}_0(x_1))}{\alpha T_i}
-
\log Z_i
\right]
-
\left[
\log p_1(x^{(j)})
+
\frac{r(\hat{x}^{(j)}_0(x_1))}{\alpha T_j}
-
\log Z_j
\right].
\label{eq:appendix_swap_expansion}
\end{align}
The Gaussian prior terms \(\log p_1(\cdot)\) cancel because all
temperature levels share the same prior reference measure. The
normalizing constants \(\log Z_i\) and \(\log Z_j\) also cancel because
the swap only exchanges states between two fixed temperature levels.
Therefore,
\begin{equation}
\log A_{\mathrm{swap}}
=
\left(
\frac{1}{T_i}
-
\frac{1}{T_j}
\right)
\frac{
r(\hat{x}^{(j)}_0(x_1))
-
r(\hat{x}^{(i)}_0(x_1))
}{\alpha}.
\label{eq:appendix_swap_log_acceptance}
\end{equation}
The swap is accepted with probability
\begin{equation}
a_{\mathrm{swap}}
=
\min\left\{
1,\,
\exp\left(\log A_{\mathrm{swap}}\right)
\right\}.
\label{eq:appendix_swap_probability}
\end{equation}

This expression has an intuitive interpretation. Since \(T_i<T_j\), the
coefficient
\[
\frac{1}{T_i}-\frac{1}{T_j}
\]
is positive. Therefore, if the hotter chain \(j\) has discovered a
higher-reward state than the colder chain \(i\), namely if
\[
r(\hat{x}^{(j)}_0(x_1))
>
r(\hat{x}^{(i)}_0(x_1)),
\]
then \(\log A_{\mathrm{swap}}>0\), and the swap is accepted with
probability one. This allows high-reward states discovered by hotter
chains to move down toward colder temperatures. Conversely, swaps that
move lower-reward states into colder chains may still be accepted with
probability \(\exp(\log A_{\mathrm{swap}})\), which is necessary for
detailed balance and prevents the algorithm from degenerating into greedy
reward maximization.

In our implementation, swaps are applied particle-wise. If
\(x^{(\ell)}_c\) denotes the \(c\)-th particle at temperature
level \(\ell\), then the particle-wise log acceptance ratio is
\begin{equation}
\log A^{(c)}_{\mathrm{swap}}
=
\left(
\frac{1}{T_i}
-
\frac{1}{T_j}
\right)
\frac{
r(\hat{x}^{(j)}_{0,c}(x_1))
-
r(\hat{x}^{(i)}_{0,c}(x_1))
}{\alpha}.
\label{eq:appendix_particlewise_swap_log_acceptance}
\end{equation}
This is the expression used in Eq.~\eqref{eq:appendix_swap_log_acceptance}. If an
entire block of \(C\) particles were swapped jointly instead, the log
acceptance ratio would be the sum of the particle-wise log ratios:
\begin{equation}
\log A^{\mathrm{block}}_{\mathrm{swap}}
=
\sum_{c=1}^{C}
\log A^{(c)}_{\mathrm{swap}}.
\label{eq:appendix_block_swap_log_acceptance}
\end{equation}

Finally, the swap transition preserves the joint tempered distribution
\(\pi_{\mathrm{pt}}\). Because the adjacent-replica swap proposal is
symmetric, the Metropolis rule satisfies detailed balance:
\begin{equation}
\pi_{\mathrm{pt}}(x)q(x,x')
a(x,x')
=
\pi_{\mathrm{pt}}(x')q(x',x)
a(x',x),
\label{eq:appendix_swap_detailed_balance}
\end{equation}
where \(x'\) denotes the swapped state. Hence, the
replica-exchange kernel leaves \(\pi_{\mathrm{pt}}\) invariant. Since
the within-temperature pCNL kernels also preserve their respective
tempered targets, the composition of pCNL updates and replica-exchange
swaps preserves the full joint target distribution in
Eq.~\eqref{eq:appendix_pt_joint_target}.

\section{Algorithm Details for PATHS}
\label{app:paths_algorithm}

Here we provide the detailed pseudocode for PATHS. The procedure is
summarized in Algorithm~\ref{alg:paths}. The implementation uses the
compact per-particle pCNL log-density contribution \(\Lambda_{\ell,c}\)
defined below, which is equivalent to the temperature-scaled
Metropolis--Hastings correction in
Appendix~\ref{app:pcnl_temperature_scaled_transition}. Replica-exchange
swaps are implemented using an even--odd schedule: at successive swap
intervals, PATHS alternates between the disjoint adjacent pairs
\(\{(1,2),(3,4),\ldots\}\) and
\(\{(2,3),(4,5),\ldots\}\). This allows multiple non-overlapping swaps to
be proposed in parallel while ensuring that every adjacent temperature
pair is attempted regularly.

\begin{algorithm}[H]
%\scriptsize
\caption{PATHS: Parallel Tempering pCNL for Initial Particle Sampling}
\label{alg:paths}
\begin{algorithmic}[1]
\REQUIRE Initial cold-chain particles
\(\mathbf{X}^{(1)}_0\in\mathbb{R}^{C\times D}\);
reward \(r\); reward scale \(\alpha\); pCNL step sizes
\(\{h_\ell\}_{\ell=1}^L\); temperatures
\(1=T_1<T_2<\cdots<T_L\); swap interval
\(\Delta_{\mathrm{swap}}\); burn-in \(B\); post-burn-in iterations \(M\);
number of SMC initial particles \(N_{\mathrm{init}}\)
\ENSURE Initial particle set \(\mathcal{S}_{\mathrm{init}}\)

\STATE \textbf{// Initialization}
\STATE \(\mathbf{X}^{(1)} \leftarrow \mathbf{X}^{(1)}_0\)
\FOR{\(\ell=2,\ldots,L\)}
    \STATE Sample \(\mathbf{X}^{(\ell)}\sim p_1^{\otimes C}\)
\ENDFOR
\STATE Evaluate rewards \(\mathbf{r}^{(\ell)}\) and gradients
\(\mathbf{G}^{(\ell)}\) for all \(\ell\)
\STATE \(\mathcal{S}\leftarrow\emptyset\)

\FOR{\(k=1,\ldots,B+M\)}
    \STATE \textbf{// Within-temperature pCNL updates}
    \FOR{\(\ell=1,\ldots,L\)}
        \STATE Compute temperature-scaled rewards and gradients:
        \(\tilde{\mathbf{r}}^{(\ell)}=\mathbf{r}^{(\ell)}/(\alpha T_\ell)\),
        \(\tilde{\mathbf{G}}^{(\ell)}=\mathbf{G}^{(\ell)}/(\alpha T_\ell)\)
        \STATE Propose \(\mathbf{X}'^{(\ell)}\) using the
        temperature-scaled pCNL proposal in Eq.~\eqref{eq:paths_pcnl_proposal}
        \STATE Evaluate proposed rewards \(\mathbf{r}'^{(\ell)}\) and
        gradients \(\mathbf{G}'^{(\ell)}\)
        \STATE Compute
        \(\log\mathbf{A}^{(\ell)}_{\mathrm{pCNL}}\)
        using Eq.~\eqref{eq:pcnl_acceptance_particlewise}
        \STATE Accept or reject each particle independently and update
        \((\mathbf{X}^{(\ell)},\mathbf{r}^{(\ell)},\mathbf{G}^{(\ell)})\)
    \ENDFOR

    \STATE \textbf{// Replica-exchange swaps with even--odd schedule}
    \IF{\(k \bmod \Delta_{\mathrm{swap}}=0\) \AND \(L>1\)}
        \STATE \(q\leftarrow k/\Delta_{\mathrm{swap}}\)
        \STATE \(\mathcal{P}_k \leftarrow
        \{(2m-1,2m)\}_{m=1}^{\lfloor L/2\rfloor}\)
        if \(q\) is odd, otherwise
        \(\mathcal{P}_k \leftarrow
        \{(2m,2m+1)\}_{m=1}^{\lfloor (L-1)/2\rfloor}\)
        \FOR{\((i,j)\in\mathcal{P}_k\)}
            \STATE Compute particle-wise swap log-ratio
            \[
            \log \mathbf{A}^{(i,j)}_{\mathrm{swap}}
            =
            \left(
            \frac{1}{T_i}
            -
            \frac{1}{T_j}
            \right)
            \frac{\mathbf{r}^{(j)}-\mathbf{r}^{(i)}}{\alpha}
            \]
            \STATE Accept or reject each particle-wise swap independently
        \ENDFOR
    \ENDIF

    \STATE \textbf{// Cold-chain sample collection}
    \IF{\(k>B\)}
        \STATE Append \(\mathbf{X}^{(1)}\) to \(\mathcal{S}\)
    \ENDIF
\ENDFOR

\STATE Apply chain-stratified thinning to \(\mathcal{S}\) and select
\(N_{\mathrm{init}}\) cold-chain particles
\STATE \textbf{return} \(\mathcal{S}_{\mathrm{init}}\)
\end{algorithmic}
\end{algorithm}

\begin{equation}
\mathbf{X}'^{(\ell)}
=
\rho_\ell \mathbf{X}^{(\ell)}
+
\sqrt{1-\rho_\ell^2}
\left(
\boldsymbol{\Xi}^{(\ell)}
+
\frac{\sqrt{h_\ell}}{2}
\tilde{\mathbf{G}}^{(\ell)}
\right),
\qquad
\boldsymbol{\Xi}^{(\ell)}
\sim
\mathcal{N}(\mathbf{0},\mathbf{I}_D)^{\otimes C}.
\label{eq:paths_pcnl_proposal}
\end{equation}

For each particle \(c\), the pCNL Metropolis--Hastings correction uses
the following log-density contribution. For a proposal from
\(\mathbf{x}_c\) to \(\mathbf{y}_c\) at temperature level \(\ell\), define
\begin{equation}
\Lambda_{\ell,c}
\left(
\mathbf{y}_c\mid\mathbf{x}_c;
\tilde{r}(\mathbf{y}_c),
\tilde{\mathbf{g}}(\mathbf{y}_c)
\right)
=
\tilde{r}(\mathbf{y}_c)
-
\frac{h_\ell}{8}
\left\|
\tilde{\mathbf{g}}(\mathbf{y}_c)
\right\|^2
+
\frac{\sqrt{h_\ell}}{2}
\left\langle
\frac{\mathbf{x}_c-\rho_\ell\mathbf{y}_c}
{\sqrt{1-\rho_\ell^2}},
\tilde{\mathbf{g}}(\mathbf{y}_c)
\right\rangle,
\label{eq:log_q_particlewise}
\end{equation}
where
\[
\tilde{r}(\mathbf{y}_c)
=
\frac{
r(\hat{\mathbf{y}}_{0|1,c})
}{\alpha T_\ell},
\qquad
\tilde{\mathbf{g}}(\mathbf{y}_c)
=
\frac{
\nabla_{\mathbf{x}_1} r(\hat{\mathbf{y}}_{0|1,c})
}{\alpha T_\ell}.
\]
The particle-wise pCNL acceptance probability is
\begin{equation}
A_{\mathrm{pCNL},c}^{(\ell)}
=
\min\left\{
1,
\exp\left[
\Lambda_{\ell,c}
\left(
\mathbf{x}'^{(\ell)}_c
\mid
\mathbf{x}^{(\ell)}_c;
\tilde{r}'^{(\ell)}_c,
\tilde{\mathbf{g}}'^{(\ell)}_c
\right)
-
\Lambda_{\ell,c}
\left(
\mathbf{x}^{(\ell)}_c
\mid
\mathbf{x}'^{(\ell)}_c;
\tilde{r}^{(\ell)}_c,
\tilde{\mathbf{g}}^{(\ell)}_c
\right)
\right]
\right\}.
\label{eq:pcnl_acceptance_particlewise}
\end{equation}
Equivalently, Eq.~\eqref{eq:pcnl_acceptance_particlewise} is applied
independently across the \(C\) particles in the vectorized implementation.
\section{Invariant Distribution of PATHS}
\label{app:path_invariance}

We show that the PATHS transition preserves the joint tempered target
distribution. For one particle per temperature level, define
\begin{equation}
\Pi(x^{(1)},\ldots,x^{(L)})
=
\prod_{\ell=1}^{L}
\pi^*_{1,\ell}(x^{(\ell)}).
\label{eq:app_paths_joint_target_short}
\end{equation}

Let \(\mathcal{K}_\ell\) denote the Metropolis-corrected pCNL kernel at
temperature \(T_\ell\). As described in
Appendix~\ref{app:pcnl_temperature_scaled_transition}, this kernel
targets \(\pi^*_{1,\ell}\) and satisfies detailed balance:
\begin{equation}
\pi^*_{1,\ell}(dx)\,
\mathcal{K}_\ell(x,dx')
=
\pi^*_{1,\ell}(dx')\,
\mathcal{K}_\ell(x',dx).
\label{eq:app_pcnl_db_short}
\end{equation}
Therefore, the product within-temperature kernel
\begin{equation}
\mathcal{K}_{\mathrm{within}}
=
\prod_{\ell=1}^{L}
\mathcal{K}_\ell
\label{eq:app_within_kernel_short}
\end{equation}
preserves the product target \(\Pi\).

Next, consider the replica-exchange swap kernel
\(\mathcal{K}_{\mathrm{swap}}\). As derived in
Appendix~\ref{app:swap_acceptance_derivation}, the adjacent-temperature
swap proposal is accepted with the Metropolis probability
\begin{equation}
a_{\mathrm{swap}}
=
\min\left\{
1,\,
\frac{
\Pi(S_{ij}x)
}{
\Pi(x)
}
\right\},
\label{eq:app_swap_acceptance_short}
\end{equation}
where \(S_{ij}x\) denotes the state obtained by swapping
replicas \(i\) and \(j\). Since the swap proposal is symmetric, this
Metropolis rule satisfies detailed balance with respect to \(\Pi\).
Hence, \(\mathcal{K}_{\mathrm{swap}}\) also preserves \(\Pi\).

A full PATHS update consists of within-temperature pCNL transitions,
optionally followed by a replica-exchange swap:
\begin{equation}
\mathcal{K}_{\mathrm{PATHS}}
=
\mathcal{K}_{\mathrm{swap}}
\circ
\mathcal{K}_{\mathrm{within}}.
\label{eq:app_paths_kernel_short}
\end{equation}
Since both \(\mathcal{K}_{\mathrm{within}}\) and
\(\mathcal{K}_{\mathrm{swap}}\) preserve \(\Pi\), their composition also
preserves \(\Pi\). Therefore,
\begin{equation}
\int
\Pi(dx)\,
\mathcal{K}_{\mathrm{PATHS}}(x,dx')
=
\Pi(dx').
\label{eq:app_paths_invariance_short}
\end{equation}

If swaps are performed every \(\Delta_{\mathrm{swap}}\) iterations, the
kernel becomes
\begin{equation}
\mathcal{K}_{\mathrm{PATHS},\Delta}
=
\mathcal{K}_{\mathrm{swap}}
\circ
\mathcal{K}_{\mathrm{within}}^{\Delta_{\mathrm{swap}}},
\label{eq:app_paths_kernel_delta_short}
\end{equation}
which also preserves \(\Pi\), since powers of an invariant kernel remain
invariant.

In the particle-wise implementation with \(C\) particles per temperature,
the joint target becomes
\begin{equation}
\Pi_C
=
\prod_{c=1}^{C}
\prod_{\ell=1}^{L}
\pi^*_{1,\ell}
\left(
x^{(\ell,c)}
\right).
\label{eq:app_particlewise_target_short}
\end{equation}
Because both pCNL updates and swap moves are applied particle-wise using
the same Metropolis corrections, the same argument shows that
\(\Pi_C\) is invariant under the full PATHS transition.
\section{Proof of Proposition~\ref{prop:init}}
\label{appendix:proof}

We restate the proposition and provide a proof.

\paragraph{Statement}
Assume that there exists a measurable high-reward region 
$B \subset \mathcal{X}_0$ in the data space. Let 
$T(x_1) = \hat{x}_0(x_1)$ denote the Tweedie estimator, and let 
$\nu = T_{\#}p_1$ be the pushforward measure of the prior $p_1$ under $T$.
Suppose that
\[
    \nu(B)
    =
    \mathbb{P}_{x_1 \sim p_1}\bigl(T(x_1) \in B\bigr)
    =
    \varepsilon
    \ll 1.
\]
Furthermore, assume that there is a reward gap $\Delta > 0$ separating $B$
from its complement, namely
\[
    \inf_{x_0 \in B} r(x_0)
    -
    \sup_{x_0 \in B^c} r(x_0)
    \ge \Delta.
\]
Let $\pi_1^*$ denote the reward-aware initial distribution defined in
Eq.~\eqref{eq:pi_1}, i.e.,
\[
    \pi_1^*(x_1)
    =
    \frac{
        p_1(x_1)\exp\left(r(T(x_1))/\alpha\right)
    }{
        \int_{\mathcal{X}_1}
        p_1(z)\exp\left(r(T(z))/\alpha\right)\,dz
    }.
\]
If, for some fixed $\eta \in (0,1)$,
\[
    \frac{\Delta}{\alpha}
    \ge
    \log\frac{1}{\varepsilon}
    +
    \log\frac{1-\eta}{\eta},
\]
then
\[
    D_{\mathrm{KL}}(\pi_1^* \| p_1)
    \ge
    (1-\eta)\log\frac{1}{\varepsilon} - h(\eta),
\]
where
\[
    h(\eta)
    =
    -\eta\log\eta - (1-\eta)\log(1-\eta)
\]
is the binary entropy function. In particular, for any fixed
$\eta \in (0,1)$,
\[
    D_{\mathrm{KL}}(\pi_1^* \| p_1)
    =
    \Omega\left(\log\frac{1}{\varepsilon}\right).
\]

\begin{proof}
Let
\[
    T(x_1) = \hat{x}_0(x_1)
\]
be the Tweedie estimator, and define the pre-image of the high-reward region
$B$ by
\[
    A = T^{-1}(B)
    =
    \{x_1 \in \mathcal{X}_1 : T(x_1) \in B\}.
\]
By the definition of the pushforward measure $\nu = T_{\#}p_1$, we have
\[
    p_1(A)
    =
    \nu(B)
    =
    \varepsilon.
\]

Let $\beta = 1/\alpha$. The probability of $A$ under the reward-aware
initial distribution $\pi_1^*$ is
\begin{align*}
    \pi_1^*(A)
    &=
    \frac{
        \int_A p_1(x_1)\exp\left(\beta r(T(x_1))\right)\,dx_1
    }{
        \int_{\mathcal{X}_1}
        p_1(x_1)\exp\left(\beta r(T(x_1))\right)\,dx_1
    } \\
    &=
    \frac{
        \int_B \exp\left(\beta r(x_0)\right)\,d\nu(x_0)
    }{
        \int_B \exp\left(\beta r(x_0)\right)\,d\nu(x_0)
        +
        \int_{B^c} \exp\left(\beta r(x_0)\right)\,d\nu(x_0)
    }.
\end{align*}

Define
\[
    R_+ = \inf_{x_0 \in B} r(x_0),
    \qquad
    R_- = \sup_{x_0 \in B^c} r(x_0).
\]
By assumption,
\[
    R_+ - R_- \ge \Delta.
\]
Therefore,
\begin{align*}
    \int_B \exp\left(\beta r(x_0)\right)\,d\nu(x_0)
    &\ge
    \nu(B)e^{\beta R_+}
    =
    \varepsilon e^{\beta R_+}, \\
    \int_{B^c} \exp\left(\beta r(x_0)\right)\,d\nu(x_0)
    &\le
    \nu(B^c)e^{\beta R_-}
    =
    (1-\varepsilon)e^{\beta R_-}.
\end{align*}
Substituting these bounds gives
\begin{align*}
    \pi_1^*(A)
    &\ge
    \frac{
        \varepsilon e^{\beta R_+}
    }{
        \varepsilon e^{\beta R_+}
        +
        (1-\varepsilon)e^{\beta R_-}
    } \\
    &=
    \frac{
        \varepsilon
    }{
        \varepsilon
        +
        (1-\varepsilon)e^{-\beta(R_+-R_-)}
    } \\
    &\ge
    \frac{
        \varepsilon
    }{
        \varepsilon
        +
        (1-\varepsilon)e^{-\beta\Delta}
    }
    =
    q_\Delta.
\end{align*}

We now show that $q_\Delta$ is close to one under the assumed scaling of
$\beta\Delta$. Since
\[
    \beta\Delta
    =
    \frac{\Delta}{\alpha}
    \ge
    \log\frac{1}{\varepsilon}
    +
    \log\frac{1-\eta}{\eta},
\]
we have
\[
    e^{-\beta\Delta}
    \le
    \varepsilon \frac{\eta}{1-\eta}.
\]
Hence
\begin{align*}
    q_\Delta
    &=
    \frac{
        \varepsilon
    }{
        \varepsilon
        +
        (1-\varepsilon)e^{-\beta\Delta}
    } \\
    &\ge
    \frac{
        \varepsilon
    }{
        \varepsilon
        +
        (1-\varepsilon)\varepsilon\frac{\eta}{1-\eta}
    } \\
    &=
    \frac{
        1
    }{
        1
        +
        (1-\varepsilon)\frac{\eta}{1-\eta}
    } \\
    &\ge
    1-\eta.
\end{align*}
Therefore,
\[
    \pi_1^*(A) \ge q_\Delta \ge 1-\eta.
\]

Next, apply the data processing inequality for KL divergence under the
binary partition $\{A,A^c\}$. This yields
\[
    D_{\mathrm{KL}}(\pi_1^* \| p_1)
    \ge
    d_{\mathrm{Bern}}\left(\pi_1^*(A) \| p_1(A)\right),
\]
where
\[
    d_{\mathrm{Bern}}(q\|\varepsilon)
    =
    q\log\frac{q}{\varepsilon}
    +
    (1-q)\log\frac{1-q}{1-\varepsilon}.
\]
Since $p_1(A)=\varepsilon$ and $\pi_1^*(A)\ge 1-\eta$, we obtain
\[
    D_{\mathrm{KL}}(\pi_1^* \| p_1)
    \ge
    d_{\mathrm{Bern}}(1-\eta\|\varepsilon),
\]
using the fact that $d_{\mathrm{Bern}}(q\|\varepsilon)$ is increasing in
$q$ for $q>\varepsilon$.

It remains to lower bound this Bernoulli KL divergence. We have
\begin{align*}
    d_{\mathrm{Bern}}(1-\eta\|\varepsilon)
    &=
    (1-\eta)\log\frac{1-\eta}{\varepsilon}
    +
    \eta\log\frac{\eta}{1-\varepsilon} \\
    &=
    (1-\eta)\log\frac{1}{\varepsilon}
    +
    (1-\eta)\log(1-\eta)
    +
    \eta\log\eta
    -
    \eta\log(1-\varepsilon) \\
    &\ge
    (1-\eta)\log\frac{1}{\varepsilon}
    -
    h(\eta),
\end{align*}
where
\[
    h(\eta)
    =
    -\eta\log\eta - (1-\eta)\log(1-\eta),
\]
and we used $-\eta\log(1-\varepsilon)\ge 0$.

Thus,
\[
    D_{\mathrm{KL}}(\pi_1^* \| p_1)
    \ge
    (1-\eta)\log\frac{1}{\varepsilon}
    -
    h(\eta).
\]
For any fixed $\eta \in (0,1)$, the dominant term as
$\varepsilon \to 0$ is $(1-\eta)\log(1/\varepsilon)$. Therefore,
\[
    D_{\mathrm{KL}}(\pi_1^* \| p_1)
    =
    \Omega\left(\log\frac{1}{\varepsilon}\right).
\]
\end{proof}
\section{Experimental Details}
\label{app:experiment_details}

\paragraph{Base model and stochastic sampling.}
All experiments use FLUX.1-schnell~\citep{flux2024} as the pretrained
score-based generative model. Although FLUX.1-schnell is implemented as a
flow-based model, we use its stochastic sampling formulation for
SMC-based inference-time reward alignment. We keep the downstream
reward-guided SMC denoising pipeline fixed across methods and modify only
the initialization stage.

\paragraph{Inference-time budget.}
We use a total inference-time budget of \(1000\) reward-model function
evaluations for all methods and all tasks. For posterior-initialization
methods, \(500\) evaluations are allocated to initial particle sampling
and \(500\) evaluations are allocated to the downstream SMC stage. The
SMC stage is initialized with \(20\) particles produced by the
initialization method and is run for \(25\) denoising steps. Since reward
evaluation is performed once per particle per denoising step, this
corresponds to \(20\times25=500\) reward-model evaluations in the SMC
stage. Prior-initialized SMC baselines allocate the full budget to the
SMC stage, corresponding to \(40\) particles over \(25\) denoising steps.
Thus, all methods are compared under the same total reward-model
evaluation budget.

\paragraph{Benchmarks.}
We evaluate PATHS on two compositional reward-alignment tasks:
layout-to-image generation and quantity-aware sampling. Both tasks induce
multi-modal reward landscapes: multiple object configurations can satisfy
the same spatial layout, and many spatial arrangements can realize the
same target object count. We therefore report results not only on the
full benchmark, but also on difficulty-stratified simple and complex
subsets.

\paragraph{Layout-to-image benchmark.}
We evaluate layout-to-image generation on the prompt set in
\texttt{selected\_prompts\_data.json}. Each prompt specifies object
phrases, target bounding boxes, and spatial relations. We split the
benchmark by the number of objects: prompts with four objects are assigned
to the complex subset, while prompts with three or fewer objects are
assigned to the simple subset. The complex subset contains \(10\) prompts
and the simple subset contains \(15\) prompts. For each prompt, we
generate two images per method, yielding \(25\times2=50\) images in total
for evaluation, with \(20\) complex images and \(30\) simple images.

\paragraph{Quantity-aware benchmark.}
We evaluate quantity-aware sampling on the prompt set in
\texttt{quantity\_aware\_selected\_20.json}. The benchmark contains
\(20\) prompts, each specifying a target object category and count. We
split the prompts by target count: prompts with \(n_{\mathrm{gt}}\ge25\)
are assigned to the complex subset, while prompts with
\(n_{\mathrm{gt}}<25\) are assigned to the simple subset. For each
prompt, we generate three images per method, yielding
\(20\times3=60\) images in total for evaluation. This gives \(30\)
images in the complex subset and \(30\) images in the simple subset.

\paragraph{Reward and evaluation models.}
For layout-to-image generation, the seen reward is computed using
GroundingDINO~\citep{liu2024grounding}. Given a generated image, target
object phrases, and target bounding boxes, GroundingDINO predicts object
boxes for each phrase. We compute the layout reward by measuring the
agreement between predicted boxes and target boxes. We additionally
report held-out mIoU using a different detector, Salience
DETR~\citep{hou2024salience}, which is never used during sampling.

For quantity-aware sampling, the seen reward is computed using
T2I-Count~\citep{qian2025t2icount}. Given a generated image and target
object category, T2I-Count outputs a density map, and the predicted count
is obtained by summing the density map. The reward used during sampling
is the negative smooth L1 loss in Eq.~\eqref{eq:quantity_reward}. In the
quantitative table, we report the corresponding positive T2I-Count error,
where lower values indicate better count alignment.

For both tasks, we further evaluate generated images using
ImageReward~\citep{xu2023imagereward} and
VQAScore~\citep{lin2024evaluating}. ImageReward measures
preference-based image quality, while VQAScore evaluates text-image
alignment using a visual question answering model. These metrics are not
used for guidance and are reported only as held-out evaluations.

\paragraph{Quantity-aware reward.}
For quantity-aware sampling, let \(n_{\mathrm{pred}}\) denote the
predicted count obtained by summing the T2I-Count density map, and let
\(n_{\mathrm{gt}}\) denote the target object count. During sampling, we
define the reward as the negative smooth L1 loss:
\begin{equation}
r_{\mathrm{count}}
=
-\mathrm{SmoothL1}(n_{\mathrm{pred}}, n_{\mathrm{gt}})
=
\begin{cases}
-\frac{1}{2}(n_{\mathrm{pred}}-n_{\mathrm{gt}})^2,
& |n_{\mathrm{pred}}-n_{\mathrm{gt}}| < 1,\\[3pt]
-|n_{\mathrm{pred}}-n_{\mathrm{gt}}| + \frac{1}{2},
& |n_{\mathrm{pred}}-n_{\mathrm{gt}}| \ge 1.
\end{cases}
\label{eq:quantity_reward}
\end{equation}
Thus, maximizing the reward is equivalent to minimizing the smooth
counting discrepancy.

\paragraph{PATHS hyperparameters.}
PATHS uses \(L=4\) temperature levels and \(C=1\) particle per
temperature level. Each temperature chain is run with burn-in \(B=65\)
and post-burn-in MCMC iterations \(M=60\). Replica-exchange swaps are
proposed every \(\Delta_{\mathrm{swap}}=5\) MCMC iterations. Unless
otherwise specified, hot replicas are initialized independently from the
Gaussian reference distribution
\(p_1=\mathcal{N}(\mathbf{0},\mathbf{I})\), as described in
Algorithm~\ref{alg:paths}. The layout-to-image task uses the temperature
ladder
\[
T\in\{1,2,4,8\},
\]
while the quantity-aware task uses the wider ladder
\[
T\in\{1,4,16,64\}.
\]

\begin{table}[t]
\centering
\caption{
Main PATHS hyperparameters used in the experiments.
}
\label{tab:paths_hyperparams}
\begin{tabular}{lcc}
\toprule
Hyperparameter & Layout-to-Image & Quantity-Aware \\
\midrule
Number of temperatures \(L\) & \(4\) & \(4\) \\
Particles per temperature \(C\) & \(1\) & \(1\) \\
Burn-in \(B\) & \(65\) & \(65\) \\
Post-burn-in iterations \(M\) & \(60\) & \(60\) \\
Swap interval \(\Delta_{\mathrm{swap}}\) & \(5\) & \(5\) \\
Temperature ladder & \(\{1,2,4,8\}\) & \(\{1,4,16,64\}\) \\
pCNL step size \(h_\ell\) & \(0.5\) for all \(\ell\) & \(0.5\) for all \(\ell\) \\
pCN persistence \(\rho_\ell\) & \(0.7778\) for all \(\ell\) & \(0.7778\) for all \(\ell\) \\
Reward scale \(\alpha_{\mathrm{MCMC}}\) & \(0.02\) & \(0.03\) \\
SMC reward scale \(\alpha\) & \(0.02\) & \(0.03\) \\
MCMC gradient clipping & None & \(10^{-4}\) \\
SMC gradient clipping & None & \(2.5\times10^{-4}\) \\
\bottomrule
\end{tabular}
\end{table}

We distinguish between the reward scale used during initial MCMC sampling,
\(\alpha_{\mathrm{MCMC}}\), and the reward scale used during downstream
SMC denoising, \(\alpha_{\mathrm{SMC}}\). In our experiments, these two
values are set equal within each task: \(0.02\) for layout-to-image
generation and \(0.03\) for quantity-aware sampling. The pCNL step size is
\(h_\ell=0.5\) for all temperatures, corresponding to
\(\rho_\ell=(1-h_\ell/4)/(1+h_\ell/4)=0.7778\). For quantity-aware
sampling, we apply gradient clipping with thresholds \(10^{-4}\) during
MCMC initialization and \(2.5\times10^{-4}\) during SMC denoising; no
gradient clipping is used for layout-to-image generation.

\paragraph{Even--odd replica-exchange schedule.}
Rather than sampling a single adjacent pair uniformly at random, we use
an even--odd swap schedule. At every swap interval
\(\Delta_{\mathrm{swap}}\), we alternate between two sets of disjoint
adjacent temperature pairs:
\[
\mathcal{P}_{\mathrm{odd}}=\{(1,2),(3,4),\ldots\},
\qquad
\mathcal{P}_{\mathrm{even}}=\{(2,3),(4,5),\ldots\}.
\]
For \(L=4\), this corresponds to alternating between swaps on
\((T_1,T_2)\) and \((T_3,T_4)\), followed by swaps on \((T_2,T_3)\).
Because pairs within each set are disjoint, the swap proposals can be
applied in parallel without any chain participating in two swaps at the
same step. Each proposed adjacent swap is accepted independently using
the Metropolis replica-exchange rule in Eq.~\eqref{eq:appendix_swap_probability}.
This schedule ensures that all neighboring temperature pairs are
regularly attempted while preserving the standard detailed-balance
correction for each swap move.

\begin{figure}[t]
    \centering
    \includegraphics[width=0.92\linewidth]{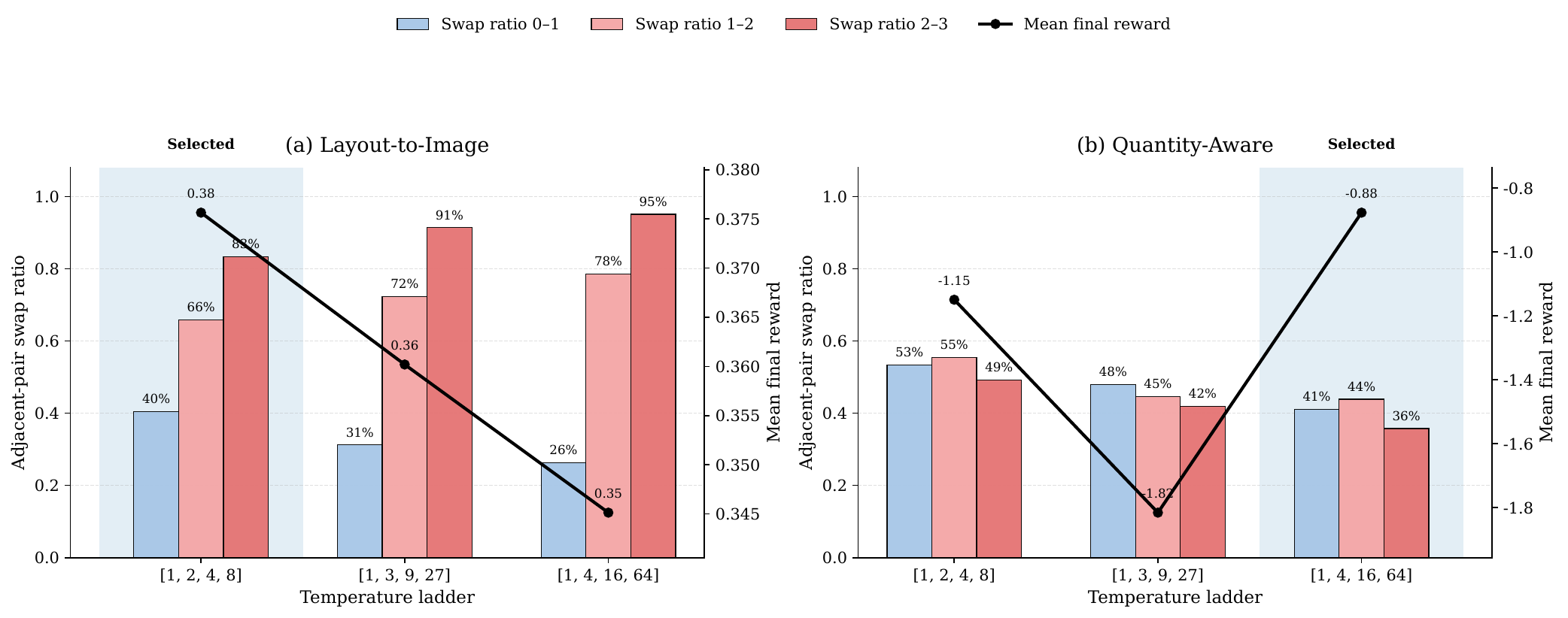}
    \caption{
    Temperature ladder sensitivity of PATHS. Bars show adjacent-pair
    replica-exchange swap ratios between neighboring temperatures
    \((T_0,T_1)\), \((T_1,T_2)\), and \((T_2,T_3)\), while the black line
    shows the mean final reward. The shaded ladders are selected for the
    main experiments. For layout-to-image generation, the milder ladder
    \([1,2,4,8]\) achieves the highest mean final reward. For
    quantity-aware sampling, the wider ladder \([1,4,16,64]\) achieves
    the highest mean final reward and maintains moderate swap ratios
    across adjacent temperature pairs. These results illustrate the
    exploration--communication trade-off in choosing the temperature
    ladder.
    }
    \label{fig:ladder_sensitivity}
\end{figure}

\paragraph{Temperature ladder sensitivity.}
We analyze the effect of the temperature ladder using the sensitivity
study shown in Figure~\ref{fig:ladder_sensitivity}. Parallel tempering
requires sufficient overlap between adjacent tempered distributions so
that replica-exchange moves can transmit information across the ladder,
and geometric temperature ladders are a standard practical choice for
encouraging such overlap~\citep{earl2005parallel,predescu2004incomplete}.
However, the appropriate ladder is task-dependent: overly narrow ladders
may not sufficiently flatten the reward landscape for effective
exploration, while overly wide ladders may weaken the coupling between
hot-chain exploration and cold-chain reward quality.

Figure~\ref{fig:ladder_sensitivity} illustrates this
exploration--communication trade-off by reporting both mean final reward
and adjacent-pair swap ratios for several geometric ladders. For
layout-to-image generation, the milder ladder \([1,2,4,8]\) provides the
best empirical trade-off among the tested settings. For quantity-aware
sampling, the wider ladder \([1,4,16,64]\) performs better, suggesting
that difficult counting constraints benefit from stronger reward
flattening while still maintaining non-negligible swap rates across
adjacent temperature pairs. We therefore use these task-level calibrated
ladders for the main quantitative and qualitative experiments.

We emphasize that this analysis is intended as a task-level sensitivity
study rather than a fully independent validation-based hyperparameter
selection protocol. A more rigorous evaluation would select temperature
ladders on a held-out validation prompt set or adapt them per prompt; we
leave such validation-based or adaptive ladder selection to future work.

\paragraph{\(\Psi\)-Sampler and Best-of-4.}
For \(\Psi\)-Sampler~\citep{yoon2025psi}, we use independent
same-temperature pCNL chains targeting the cold posterior
\(\pi_1^*\). The chains use the same total initialization budget as
PATHS, and collected samples are pooled and thinned before being passed
to the SMC stage. Best-of-4 uses four independent pCNL chains with the
same per-chain budget as PATHS, but all chains use the same temperature
\(T=1\). After burn-in and thinning, each chain is scored by the average
reward of its collected particles, and only the highest-scoring chain is
forwarded to the SMC stage. We use \(K=4\) for all tasks to match the
number of temperature levels in PATHS.

\paragraph{Compute matching.}
PATHS and Best-of-4 both use four parallel MCMC chains with \(C=1\),
\(B=65\), and \(M=60\). The key distinction is that Best-of-4 runs four
independent cold chains, whereas PATHS runs four tempered chains coupled
by replica-exchange swaps. Therefore, the comparison evaluates the effect
of temperature-based exploration and online inter-chain communication
under a matched initialization budget. Replica-exchange swaps do not
require additional reward-model evaluations, since the swap acceptance
ratio reuses the rewards already computed during the pCNL updates.

\paragraph{Qualitative examples.}
The main paper includes representative qualitative comparisons in
Figure~\ref{fig:qualitative_results}. Additional qualitative examples for
layout-to-image and quantity-aware sampling are provided in
Appendix~\ref{qualitative_results}.

\section{Synthetic Experiment Details}
\label{app:synthetic}
All methods are evaluated with the same total NFE budget of $3{,}000$.
Each method is independently repeated $B = 2{,}000$ times,
and the sample with the highest reward is selected from each run.
For PATHS, we use three parallel chains at temperatures
$T \in \{1.0,\,2.0,\,4.0\}$ with a temperature ratio of $2.0$.
% ---------------------------------------------------------
\subsection*{Base Diffusion Model}

We use a pre-trained Rectified Flow model~\citep{liu2022flow}
whose \emph{source distribution} $\pi_0$ is an isotropic
two-dimensional Gaussian,
\begin{equation}
  \pi_0 = \mathcal{N}\!\left(\mathbf{0},\,\mathbf{I}_2\right),
\end{equation}
and whose \emph{training target distribution} $\pi_1$ is an
anisotropic Gaussian,
\begin{equation}
  \pi_1 = \mathcal{N}\!\!\left(
    \begin{pmatrix}0\\0\end{pmatrix},\,
    \begin{pmatrix}1 & 0\\0 & 0.05\end{pmatrix}
  \right).
\end{equation}
At inference, denoising trajectories are computed with the
Euler sampler.

% ---------------------------------------------------------
\subsection*{Reward Function}

We define a multi-modal reward function
$r:\mathbb{R}^2\!\to\!\mathbb{R}$ in the \emph{output} space
($\mathbf{x}_1$) of the flow.
The reward is defined as a weighted sum of $K=9$ Gaussian bumps,
\begin{equation}
  r(x_1)
  \;=\;
  \frac{1}{\alpha}
  \sum_{k=1}^{K}
  w_k \exp\!\left(-10\,\|x_1 - {\mu}_k\|^2\right),
  \quad \alpha = 0.05,
\end{equation}
where $\boldsymbol{\mu}_k$ are fixed mode centres spread across
$[-2, 2]^2$ and the weights $w_k \in {0.4, 0.5, 1.0}$ are chosen to produce modes of varying heights, with one dominant mode at weight $1.0$, making the landscape intentionally multi-modal.

\section{Limitations}
\label{sec:limitations}

\paragraph{Reward complexity dependence.}
The advantage of PATHS arises specifically from the multi-modal structure of the reward landscape. PATHS' replica-exchange mechanism is designed to bridge isolated high-reward modes by exchanging information between hot and cold chains, and consequently delivers the largest gains when the reward landscape contains many disconnected high-reward basins separated by low-reward barriers. The layout-to-image and quantity-aware tasks studied in Section~\ref{sec:experiments} exhibit precisely this structure. In the layout task, a prompt such as \textit{``A person and an airplane over a car and under the chair...''} admits many distinct geometric configurations that satisfy the bounding-box constraints — different placements of each of the four objects, different camera viewpoints, and different scene compositions — each forming a separate mode in the reward-aware posterior. Similarly, in the quantity-aware task, generating \textit{``82 blueberries''} is compatible with a wide range of spatial arrangements, occlusion patterns, and viewpoints, and the reward drops sharply once the count deviates even slightly from the target (e.g., generating 80 instead of 82 blueberries already incurs a substantial penalty under the smooth-$L_1$ count reward in Eq.~\eqref{eq:quantity_reward}), creating sharply peaked, well-separated modes. In both cases, independent chain pCNL methods such as $\Psi$-Sampler tend to commit to a single basin and fail to discover competing high-reward configurations, whereas PATHS' tempered ladder enables systematic exploration across these modes.

\begin{table}[htbp]
  \caption{Quantitative comparison on the \textbf{Aesthetic Score Optimization} task. We evaluate the performance across the \textit{Overall} dataset ($N=60$), as well as two difficulty-based subsets: \textit{Complex} ($N=30$) and \textit{Simple} ($N=30$). Metrics marked with $\dagger$ are used as seen reward during reward-guided sampling, where others are held-out reward.}
  \label{tab:aesthetic_score_results_1}
  \centering
  % 열(Column) 사이의 기본 간격을 줄입니다
  \setlength{\tabcolsep}{4pt}
  % 글자 크기를 한 단계 줄입니다 (8pt)
  \footnotesize 
  \begin{tabular}{lc ccc}
    \toprule
    \multirow{2}{*}{Subset} & \multirow{2}{*}{Metrics} & \multicolumn{3}{c}{Methods} \\
    \cmidrule(lr){3-5}
    & & $\psi$-Sampler \citep{yoon2025psi} & Best-of-4 & \textbf{PATHS (Ours)} \\
    \midrule
    
    % Overall (N=60) 
    \multirow{3}{*}{\shortstack[c]{Overall\\($N=60$)}} 
    & Aesthetic$^\dagger$ \citep{schuhmann2022laionaesthetics} $\uparrow$   & \underline{6.9684} & 6.9487 & \textbf{6.9785} \\
    \cmidrule(lr){2-5}
    & ImageReward \citep{xu2023imagereward} $\uparrow$ & 1.0710 & \underline{1.0980} & \textbf{1.1036} \\
    & VQAScore \citep{lin2024evaluating} $\uparrow$    & \underline{0.8571} & \textbf{0.8579} & 0.8469 \\
    \midrule
    
    % Complex (N=30) 
    \multirow{3}{*}{\shortstack[c]{Complex\\($N=30$)}} 
    & Aesthetic$^\dagger$ \citep{schuhmann2022laionaesthetics} $\uparrow$   & 6.7708 & \underline{6.7746} & \textbf{6.7977} \\
    \cmidrule(lr){2-5}
    & ImageReward \citep{xu2023imagereward} $\uparrow$ & \underline{1.2245} & \textbf{1.2319} & 1.2174 \\
    & VQAScore \citep{lin2024evaluating} $\uparrow$    & \textbf{0.7537} & \underline{0.7536} & 0.7375 \\
    \midrule
    
    % Simple (N=30) 
    \multirow{3}{*}{\shortstack[c]{Simple\\($N=30$)}} 
    & Aesthetic$^\dagger$ \citep{schuhmann2022laionaesthetics} $\uparrow$   & \textbf{7.1659} & 7.1228 & \underline{7.1593} \\
    \cmidrule(lr){2-5}
    & ImageReward \citep{xu2023imagereward} $\uparrow$ & 0.9175 & \underline{0.9642} & \textbf{0.9898} \\
    & VQAScore \citep{lin2024evaluating} $\uparrow$    & \underline{0.9604} & \textbf{0.9622} & 0.9563 \\
    \bottomrule
  \end{tabular}
\end{table}

\paragraph{Reduced gains under low-complexity rewards.}
To illustrate the boundary of PATHS' advantage, we evaluate it on the aesthetic-preference generation task, where the reward is defined by the LAION Aesthetic Predictor~\cite{schuhmann2022laionaesthetics}. Quantitative results are reported in Table~\ref{tab:aesthetic_score_results_1} and qualitative comparisons in Figure~\ref{fig:aesthetic}. In contrast to layout/quantity tasks, PATHS' improvement over $\Psi$-Sampler and Best-of-4 is marginal — for instance, on the overall split, PATHS achieves an aesthetic score of 6.9785 versus 6.9684 for $\Psi$-Sampler, and the differences in held-out metrics are within noise. We attribute this attenuation to two factors that together make the aesthetic-preference reward landscape significantly less complex than the compositional rewards considered in our main experiments.

First, the aesthetic reward is defined purely on visual quality: the LAION Aesthetic Predictor \citep{schuhmann2022laionaesthetics} scores an image based on lighting, composition, color harmony, and similar low-level visual factors, without taking the text prompt into account. As a result, a wide variety of unrelated images can achieve comparably high scores, and the high-reward region in the latent space is broad and well-connected rather than partitioned into narrow, prompt-specific modes. This is in sharp contrast to compositional rewards such as GroundingDINO \citep{liu2024grounding}-based layout matching or T2ICount \citep{qian2025t2icount}-based counting, which couple the prompt to the image and produce sharp, prompt-dependent reward peaks. Second, the absence of a hard structural constraint (such as ``four objects in specified bounding boxes'' or ``exactly 82 instances'') removes the all-or-nothing failure modes that make compositional rewards multi-modal in the first place. When the reward landscape is smooth and largely uni-modal, independent-chain methods such as $\Psi$-Sampler can already locate high-reward regions effectively, and the mode-bridging benefit of replica exchange becomes negligible.

\paragraph{Implications and outlook.}
This boundary is not a defect of PATHS but a direct consequence of its design rationale: PATHS provides a systematic mechanism for escaping local modes and is therefore most useful precisely when such modes exist. Practitioners can use this principle as a guide — PATHS is the appropriate choice when the reward couples the prompt to fine-grained structural properties (object placement, count, relations, or compositional constraints), whereas simpler independent-chain posterior initialization suffices for smooth, structurally-unconstrained rewards. We further note that the trend in inference-time alignment is moving toward increasingly complex and specific guidance signals: layout-conditioned generation, controllable scene composition, instruction-following with multi-object spatial relations, and fine-grained counting are all active research directions that demand exactly the kind of multi-modal, compositional rewards on which PATHS excels. As generative models are deployed in settings that require precise, structured user control rather than generic visual quality, the regime in which PATHS provides the largest benefit is likely to become the dominant rather than the exceptional case. 

At the same time, stronger controllability in text-to-image generation may amplify existing risks of misleading or harmful synthetic content, especially when precise spatial or compositional constraints are used for deceptive purposes. To mitigate release-related risks, we do not introduce or release a new generative model checkpoint or sensitive dataset; PATHS is an inference-time sampling procedure evaluated on synthetic prompts and existing models, and any released code or prompts should be used consistently with the terms and safety policies of the underlying generative model.

\begin{figure}
  \begin{center}
    % 페이지 하단 잘림을 방지하기 위해 간격을 타이트하게 유지합니다.
    \setlength{\tabcolsep}{1.5pt} 
    \renewcommand{\arraystretch}{0.4} 
    
    \begin{tabular}{c c ccccc}

      % 모델별 헤더
      & & \footnotesize TDS & \footnotesize DAS & \footnotesize $\psi$-Sampler & \footnotesize Best-of-4 & \footnotesize \textbf{PATHS (Ours)} \\
      \midrule
      
      % --- Complex Group (2 Rows) ---
      % 2개의 이미지 행 + 2개의 캡션 행 = 총 4개 행을 병합
      \multirow{4}{*}{\raisebox{-0.5cm}[0pt][0pt]{\rotatebox{90}{\textbf{Aesthetic}}}} & 
      \multirow{4}{*}{\raisebox{-0.8cm}[0pt][0pt]{\rotatebox{90}{\small \textbf{Simple prompt}}}} &
      \includegraphics[width=0.17\textwidth]{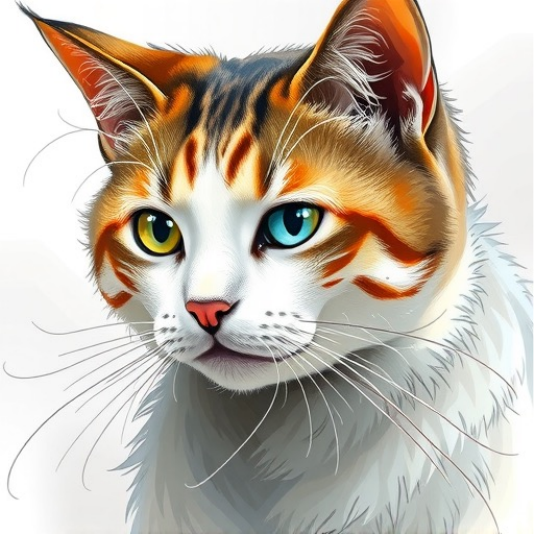} &
      \includegraphics[width=0.17\textwidth]{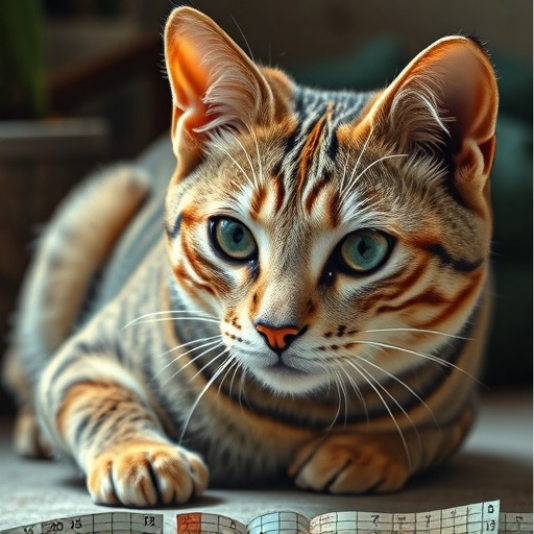} &
      \includegraphics[width=0.17\textwidth]{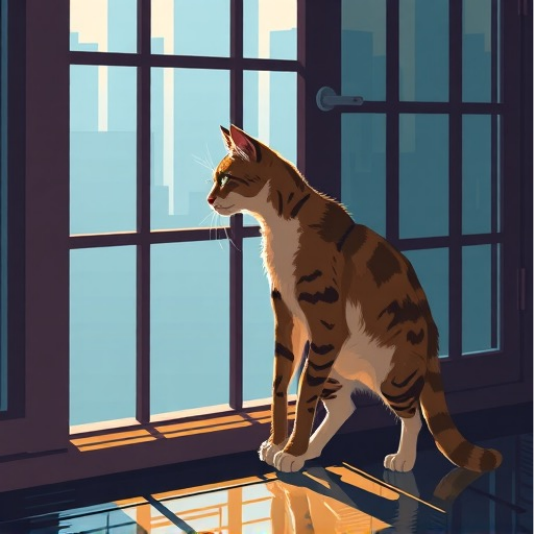} &
      \includegraphics[width=0.17\textwidth]{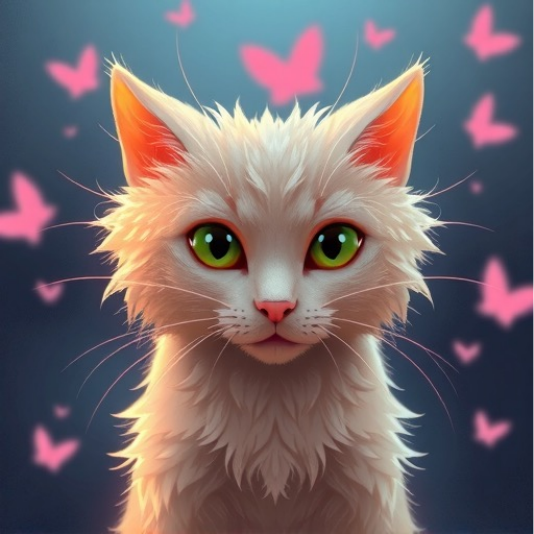} &
      \includegraphics[width=0.17\textwidth]{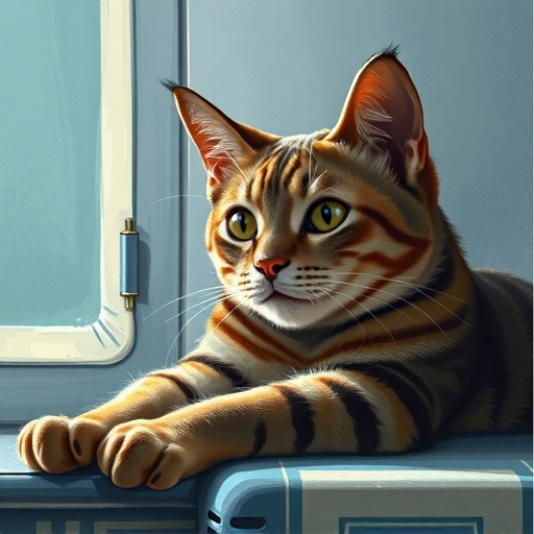} \\

        % 2. ▼ 각 이미지 밑에 들어갈 개별 글자 줄 (새로 추가) ▼
      & & 
      \scriptsize 7.05 & 
      \scriptsize 7.01 & 
      \scriptsize 7.22 & 
      \scriptsize 7.06 & 
      \scriptsize 7.44 \\
      & & \multicolumn{5}{c}{\scriptsize \textit{``Cat''}} \\
      \addlinespace[3pt]
      
      & & \includegraphics[width=0.17\textwidth]{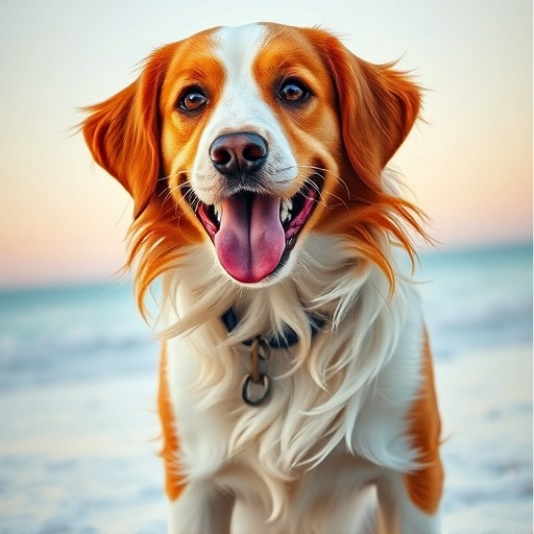} &
      \includegraphics[width=0.17\textwidth]{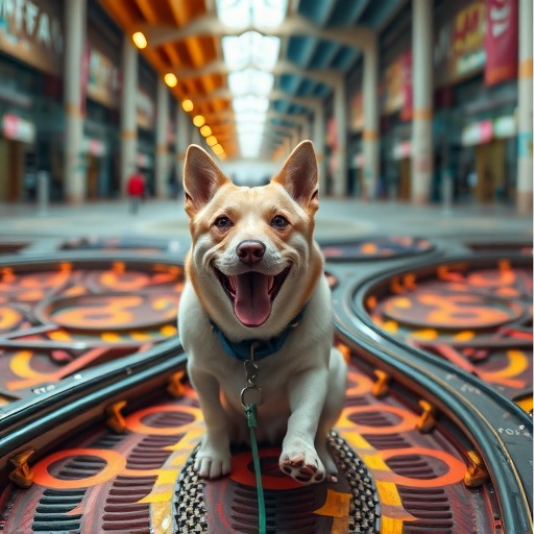} &
      \includegraphics[width=0.17\textwidth]{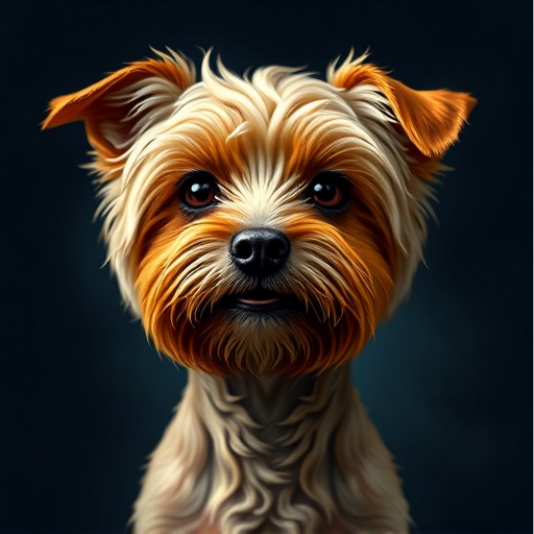} &
      \includegraphics[width=0.17\textwidth]{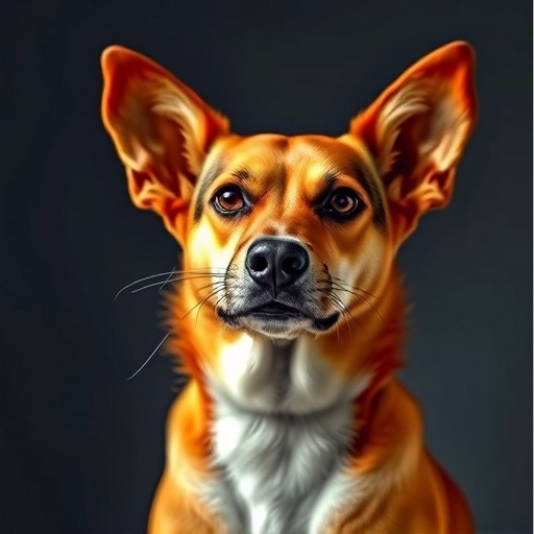} &
      \includegraphics[width=0.17\textwidth]{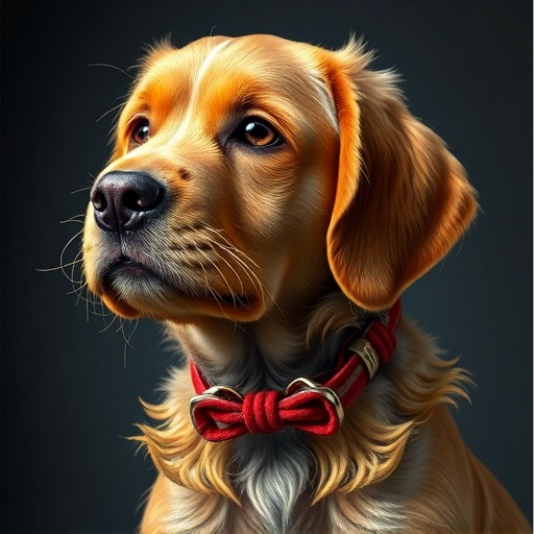} \\

      & & 
      \scriptsize 6.86 & 
      \scriptsize 6.91 & 
      \scriptsize 7.43 & 
      \scriptsize 7.02 & 
      \scriptsize 7.80 \\
      & & \multicolumn{5}{c}{\scriptsize \textit{``Dog''}} \\
      
      \midrule % 구분선
      
      % --- Simple Group (2 Rows) ---
      \multirow{4}{*}{\raisebox{-0.5cm}[0pt][0pt]{\rotatebox{90}{\textbf{Aesthetic}}}} & 
      \multirow{4}{*}{\raisebox{-1.0cm}[0pt][0pt]{\rotatebox{90}{\small \textbf{Complex prompt}}}} &
      \includegraphics[width=0.17\textwidth]{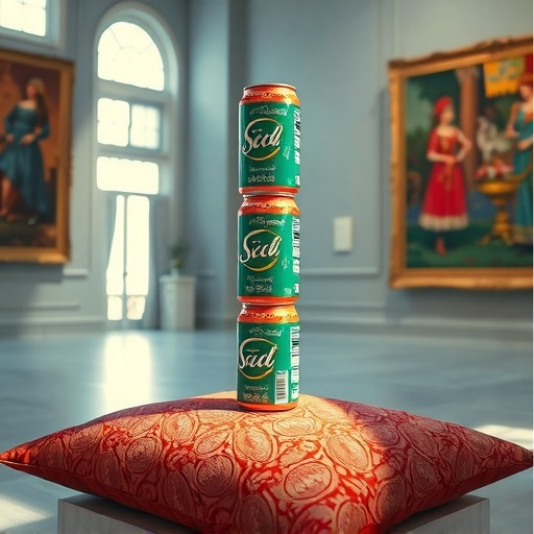} &
      \includegraphics[width=0.17\textwidth]{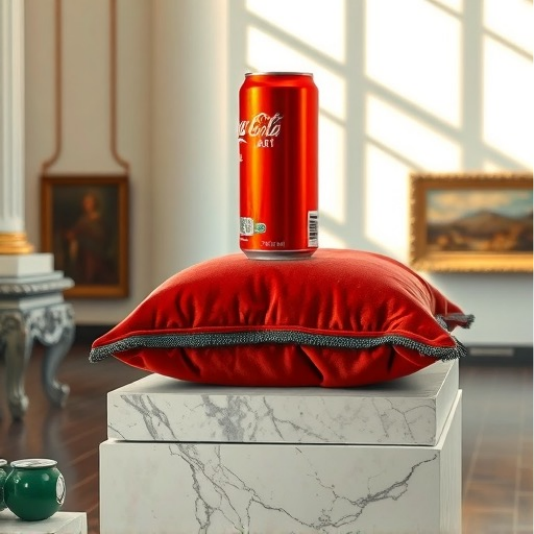} &
      \includegraphics[width=0.17\textwidth]{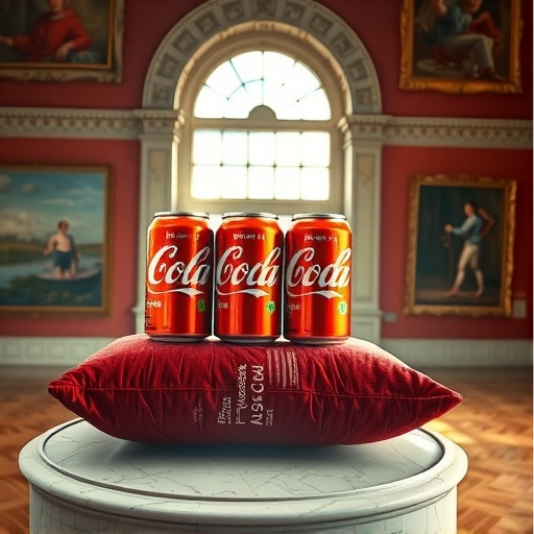} &
      \includegraphics[width=0.17\textwidth]{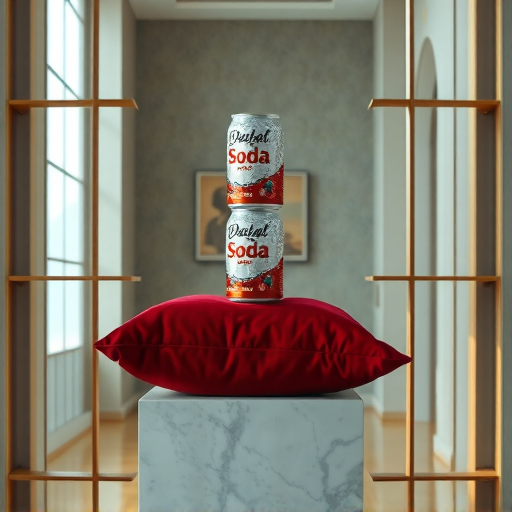} &
      \includegraphics[width=0.17\textwidth]{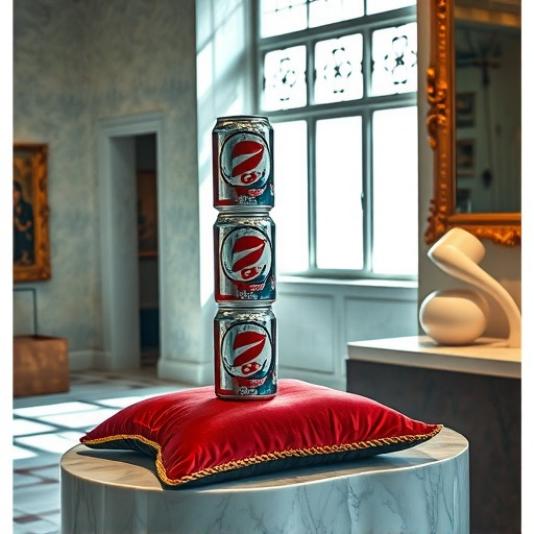} \\

      & & 
      \scriptsize 6.65 & 
      \scriptsize 6.25 & 
      \scriptsize 6.60 & 
      \scriptsize 6.08 & 
      \scriptsize 6.34 \\
      & & \multicolumn{5}{p{0.85\textwidth}}{\centering \scriptsize \textit{``Exactly three crushed, rusty soda cans stacked vertically on top of a luxurious red velvet pillow. The pillow rests on a marble pedestal in a pristine, sunlit art gallery.''}} \\
      \addlinespace[3pt]

      & & \includegraphics[width=0.17\textwidth]{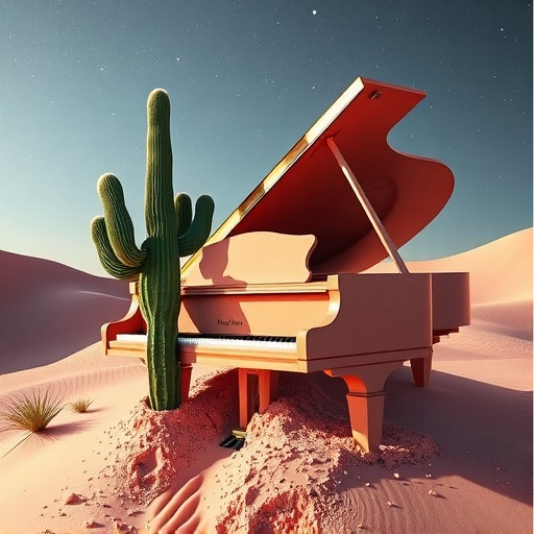} &
      \includegraphics[width=0.17\textwidth]{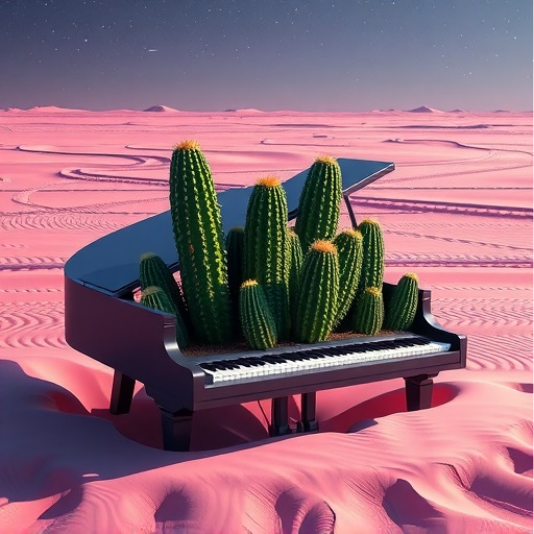} &
      \includegraphics[width=0.17\textwidth]{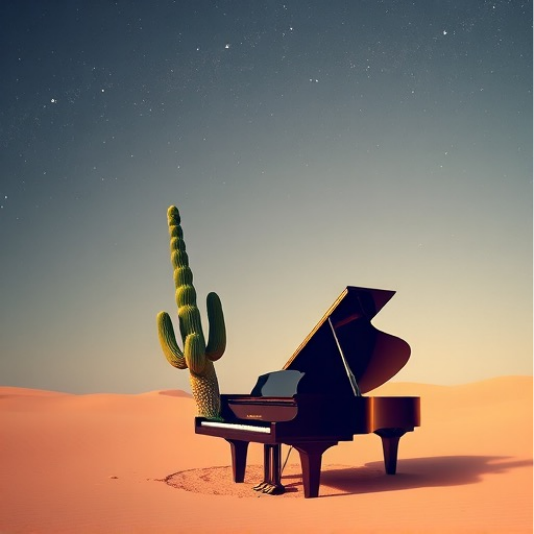} &
      \includegraphics[width=0.17\textwidth]{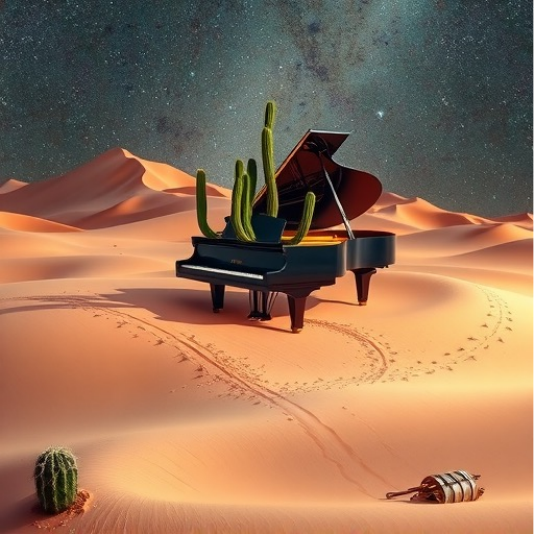} &
      \includegraphics[width=0.17\textwidth]{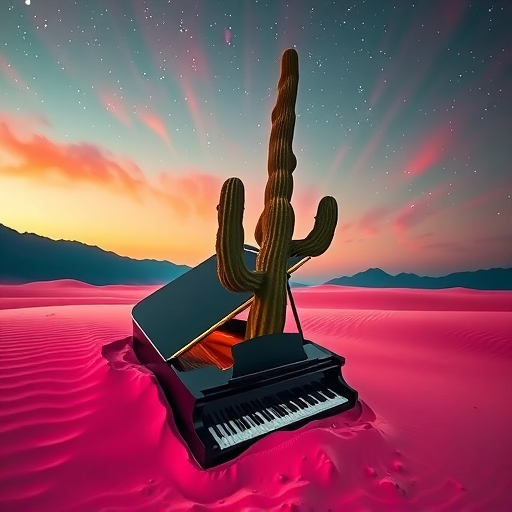} \\

      & & 
      \scriptsize 7.16 & 
      \scriptsize 7.01 & 
      \scriptsize 7.21 & 
      \scriptsize 7.15 & 
      \scriptsize 7.05 \\
      & & \multicolumn{5}{p{0.85\textwidth}}{\centering \scriptsize \textit{``A grand piano sinking into a desert of pink sand. A giant green cactus is growing directly out of the piano's keyboard, under a starry night sky.''}} \\
      \bottomrule

    \end{tabular}
    
    \caption{\textbf{Qualitative comparison on the aesthetic-preference generation task.}
We compare TDS \citep{smc_wu2023practical}, DAS \citep{smc_kim2025testtime}, $\Psi$-Sampler \citep{yoon2025psi}, Best-of-4, and PATHS on prompts of varying descriptive complexity. \textbf{Top two rows (Simple prompts: ``Cat'', ``Dog'')}: short prompts impose minimal structural constraints, leaving the reward landscape dominated by smooth visual-quality preferences. The qualitative differences across methods are subtle, and all methods produce visually plausible outputs. \textbf{Bottom two rows (Complex prompts)}: although these prompts are linguistically detailed (e.g., ``Exactly three crushed, rusty soda cans stacked vertically on top of a luxurious red velvet pillow...''), the LAION Aesthetic Predictor \citep{schuhmann2022laionaesthetics} used as the reward does not evaluate prompt alignment and only scores generic visual quality. Consequently, the reward landscape remains largely uni-modal even for these prompts. This figure illustrates the regime in which PATHS' replica-exchange mechanism offers limited benefit: when the reward landscape is smooth and structurally unconstrained.}
    \label{fig:aesthetic}
  \end{center}
\end{figure}

\newpage
\section{Additional Qualitative Results}
\label{qualitative_results}

\subsection{Replica-exchange swap trajectories}
\label{app:swap_trajectories}

Beyond the qualitative comparisons of final outputs, we visualize 
the inner mechanism of PATHS during initial particle sampling. 
Figures~\ref{fig:swap_trajectories_1}--\ref{fig:swap_trajectories_2} show 
\emph{actual} replica-exchange swap events recorded from PATHS 
runs on individual layout-to-image instances, with each row 
displaying the one-step Tweedie estimates $\hat{x}_{0|1}$ of the 
hot chain (top, $T_L>1$) and the cold chain (bottom, $T_1=1$) 
across consecutive MCMC iterations. The thick-bordered images 
mark the moment of an accepted Metropolis swap.

A consistent pattern emerges across all four examples:

\begin{itemize}
\item \textbf{Cold chain alone gets trapped.} Before the swap, the 
cold chain refines a layout that satisfies only a subset of the 
target boxes (e.g., missing the \emph{car} in 
Figure~\ref{fig:swap_trajectories_2}(Top), or the \emph{horse} in 
Figure~\ref{fig:swap_trajectories_2}(Bottom)). Local pCNL updates struggle to 
recover the missing object because adding it requires a large 
configurational change that the sharp posterior $\pi^*_1$ heavily 
penalizes during proposal.

\item \textbf{Hot chain explores broadly.} Operating under a 
flattened reward landscape, the hot chain visits diverse object 
configurations and eventually produces a state in which all target 
boxes are correctly populated.

\item \textbf{Swap transfers the discovery.} When the hot chain's 
state has higher reward than the cold chain's current state, the 
Metropolis swap is accepted and the well-composed latent is moved 
directly into the cold chain. Because both chains' rewards have 
already been computed for their own pCNL updates, the swap 
\emph{reuses cached reward evaluations} and incurs no additional 
reward-model overhead.

\item \textbf{Cold chain refines from the new starting point.} 
After the swap, the cold chain continues local pCNL refinement 
from the inherited state and produces small, targeted corrections 
(e.g., sharper object boundaries, slightly improved spatial 
alignment) rather than re-exploring the layout from scratch.
\end{itemize}

These trajectories make explicit the mechanism that drives the 
quantitative gains reported in the Table~\ref{tab:table_total}: PATHS does not 
merely run more chains; it actively transfers exploration done at 
high temperature into refinement at low temperature.

\subsection{Prompt-level Qualitative Comparisons}
\label{app:prompt_comparisons}

We provide additional per-prompt qualitative comparisons against 
all baselines considered in the main paper, covering both the 
layout-to-image generation and quantity-aware sampling tasks. 
Figure~\ref{fig:qualitative_appendix} presents layout-to-image 
examples, where target bounding boxes are overlaid on each generated 
image so that successful placement of the requested objects can be 
verified visually. Figure~\ref{fig:qualitative_appendix_quantity} presents 
quantity-aware examples, where the predicted object count and its 
absolute deviation from the target count are reported below each image.

These per-prompt examples illustrate the same pattern as the 
aggregate quantitative results: PATHS more reliably satisfies 
the full set of spatial and counting constraints, whereas 
prior-initialized SMC baselines (TDS, DAS) tend to ignore parts 
of the prompt entirely, and independent posterior-initialization 
baselines ($\Psi$-Sampler, Best-of-4) often partially recover 
 the required structure but still miss specific objects or counts. 
Together, Figures~\ref{fig:qualitative_appendix} 
and~\ref{fig:qualitative_appendix_quantity} show that PATHS' gains are 
most visible on complex prompts requiring coordinated object placement 
or accurate high-count generation. Appendix~\ref{app:swap_trajectories} 
additionally visualizes the internal swap dynamics that produce these 
final outputs.

% ============================================================
% Figure A.1 — First two trajectories
% ============================================================
\begin{figure}[b]
    \centering
    
    % --- Trajectory 1 ---
    \footnotesize \textit{Prompt: "A \textcolor{green}{car} and \textcolor{blue}{cat} beneath an \textcolor{orange}{airplane} and below a \textcolor{red}{banana}..."} \\
    \includegraphics[width=0.95\linewidth]{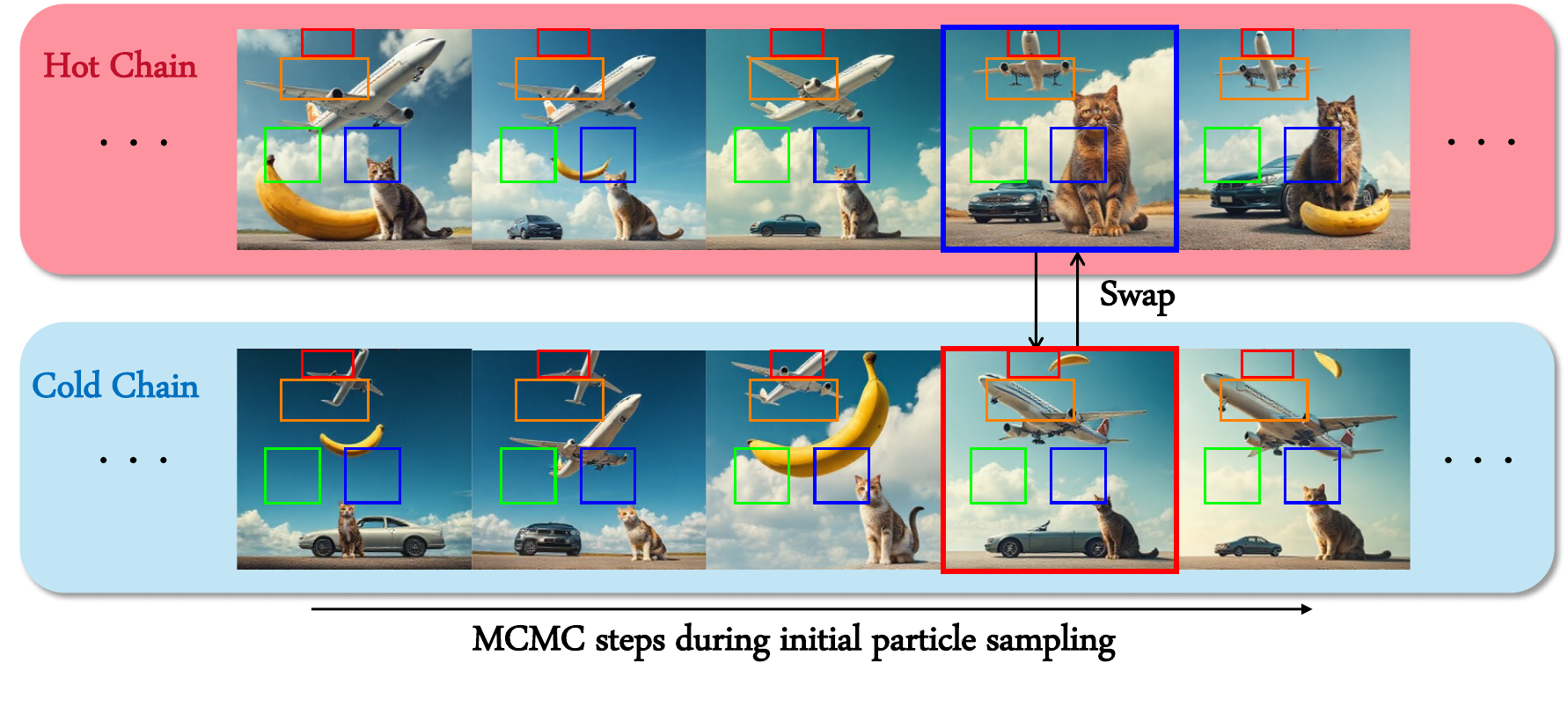}
    \vspace{0.5em}
    
    % --- Trajectory 2 ---
    \footnotesize \textit{Prompt: "A \textcolor{orange}{person} and a \textcolor{green}{dog} in the middle of \textcolor{red}{banana} and \textcolor{blue}{cat}..."} \\
    \includegraphics[width=0.95\linewidth]{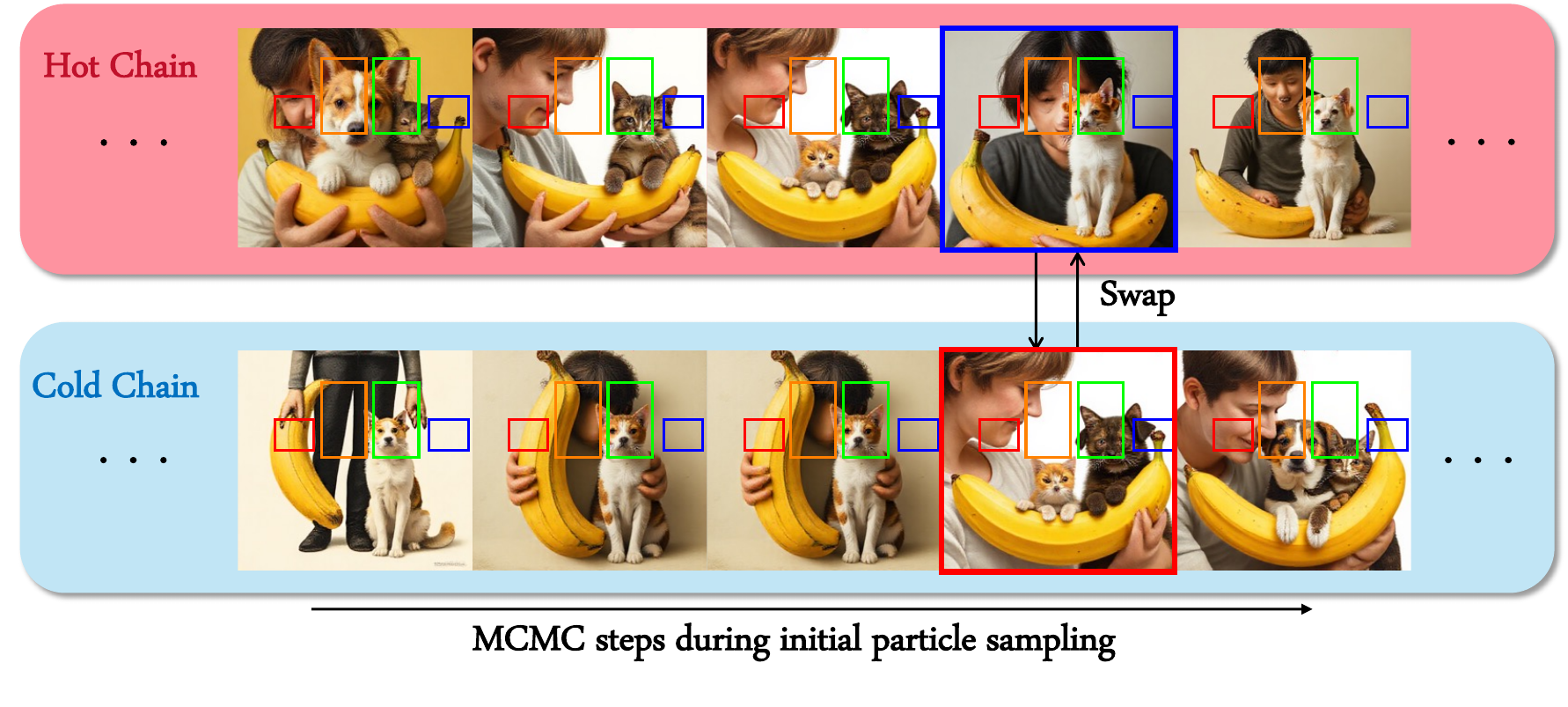}
    
    \caption{\textbf{Replica-exchange swap trajectories from PATHS 
    runs on layout-to-image instances (1/2).} Each row visualizes 
    the one-step Tweedie estimates $\hat{x}_0(x_1)$ of the hot 
    chain (top, $T_L{>}1$) and the cold chain (bottom, $T_1{=}1$) 
    across consecutive MCMC steps during initial particle 
    sampling; thick-bordered images mark the moment of an accepted 
    Metropolis swap. In each example the cold chain alone 
    converges to a layout that satisfies only a subset of the 
    target boxes, while the hot chain explores under a flattened 
    reward and discovers a state matching additional constraints. 
    The Metropolis swap then transfers this state to the cold 
    chain, which inherits it as a new starting point and continues 
    local refinement---all without additional reward-model 
    evaluations, since the swap reuses rewards already computed 
    for the within-chain pCNL updates.}
    \label{fig:swap_trajectories_1}
\end{figure}

% ============================================================
% Figure A.2 — Last two trajectories
% ============================================================
\begin{figure}[b]
    \centering
    
    % --- Trajectory 3 ---
    \footnotesize \textit{Prompt: "A \textcolor{red}{person} and \textcolor{orange}{airplane} over a \textcolor{green}{car} and under the \textcolor{blue}{chair}..."} \\
    \includegraphics[width=0.95\linewidth]{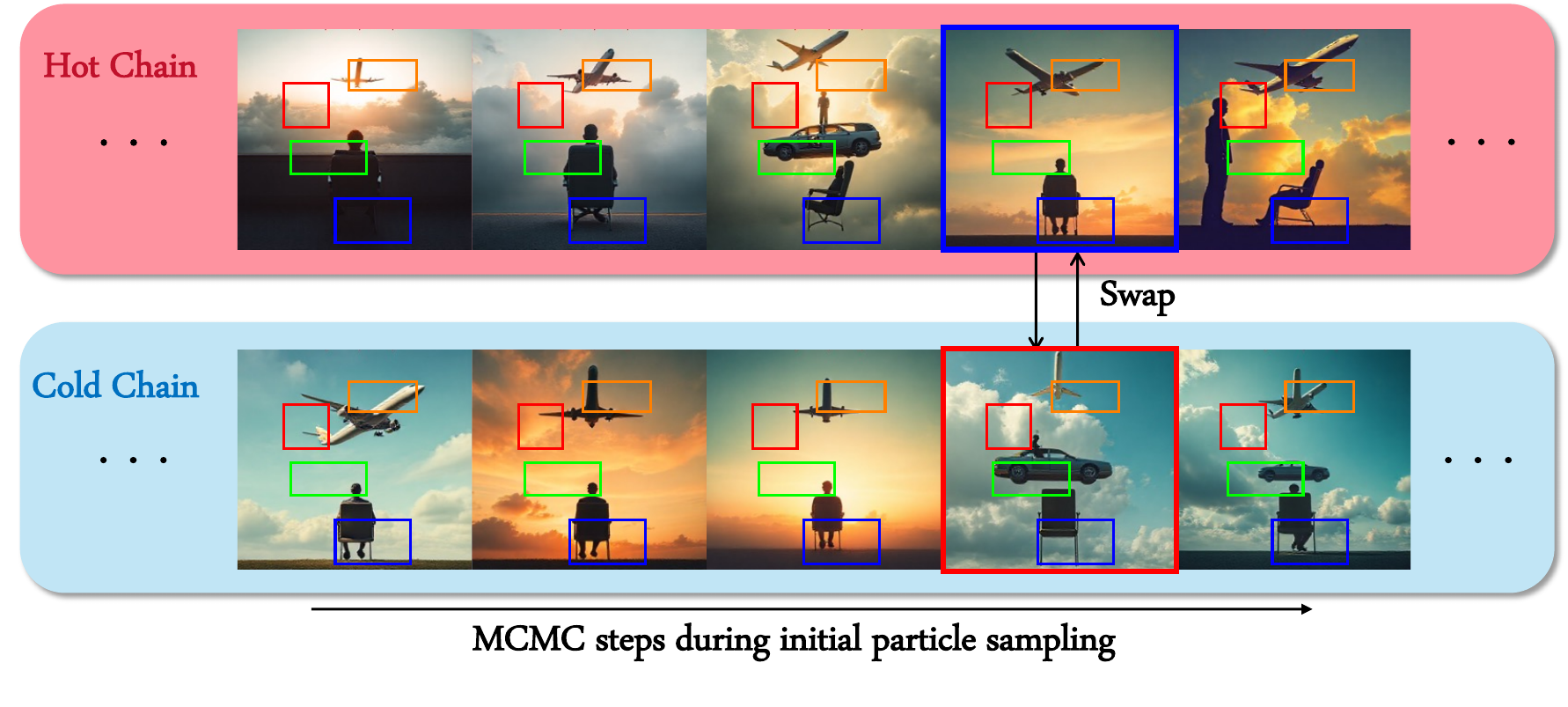}
    \vspace{0.5em}
    
    % --- Trajectory 4 ---
    \footnotesize \textit{Prompt: "A \textcolor{green}{person} and \textcolor{blue}{chair} beneath an \textcolor{orange}{airplane} and on a \textcolor{red}{horse}..."} \\
    \includegraphics[width=0.95\linewidth]{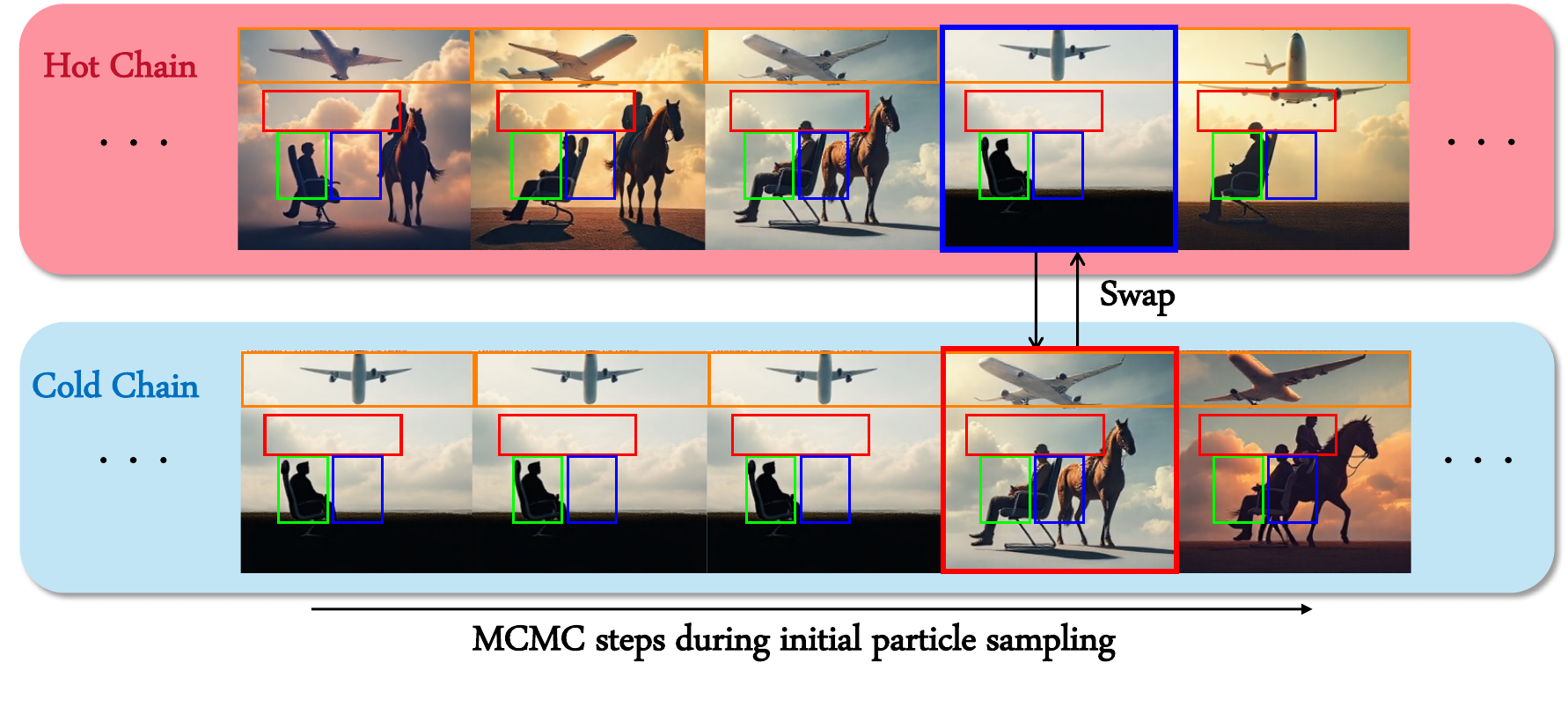}
    
    \caption{\textbf{Replica-exchange swap trajectories from PATHS 
    runs on layout-to-image instances (2/2).} Continuing 
    Figure~\ref{fig:swap_trajectories_1}: the same mechanism 
    operates across more complex spatial-relation prompts. The 
    cold chain alone gets trapped in layouts missing the car 
    (top row) or the horse (bottom row); the hot chain discovers 
    states satisfying these constraints; and the Metropolis swap 
    transfers them to the cold chain. As before, swaps reuse 
    cached reward evaluations and add no further reward-model 
    overhead.}
    \label{fig:swap_trajectories_2}
\end{figure}

\clearpage % 앞의 그림을 모두 출력한 뒤 다음으로 넘어감

\begin{figure}
  \begin{center}
    % 페이지 하단 잘림을 방지하기 위해 간격을 타이트하게 유지합니다.
    \setlength{\tabcolsep}{1.5pt} 
    \renewcommand{\arraystretch}{0.4} 

    \begin{tabular}{c c ccccc}
      %\toprule
      % 최상단 헤더 (Prior / Posterior 그룹화)
      %& & \multicolumn{2}{c}{\small \textbf{Sampling from Prior}} 
      %& \multicolumn{3}{c}{\small \textbf{Sampling from Posterior}} \\
      %\cmidrule(lr){3-4} \cmidrule(lr){5-7}
      
      % 모델별 헤더
      & & \footnotesize TDS & \footnotesize DAS & \footnotesize $\psi$-Sampler & \footnotesize Best-of-4 & \footnotesize \textbf{PATHS (Ours)} \\
      \midrule
      
      % --- Complex Group (3 Rows) ---
      % 3개의 이미지 행 + 3개의 캡션 행 = 총 6개 행을 병합
      \multirow{6}{*}{\raisebox{-3.0cm}{\rotatebox{90}{\textbf{Layout-to-Image}}}} & 
      \multirow{6}{*}{\raisebox{-3.0cm}{\rotatebox{90}{\small \textbf{Complex (= 4 obj.)}}}} & 
        \includegraphics[width=0.17\textwidth]{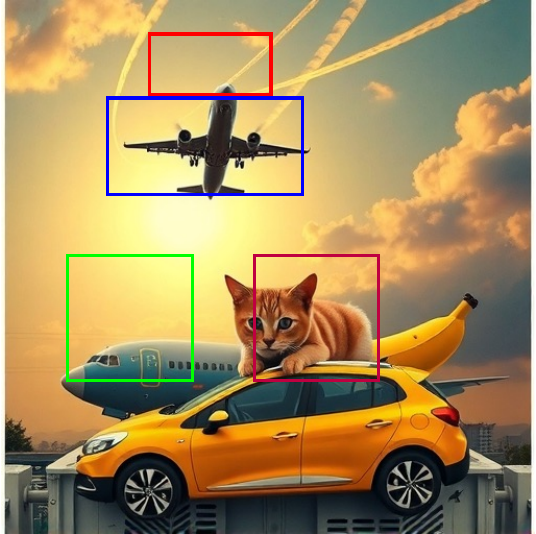} &
        \includegraphics[width=0.17\textwidth]{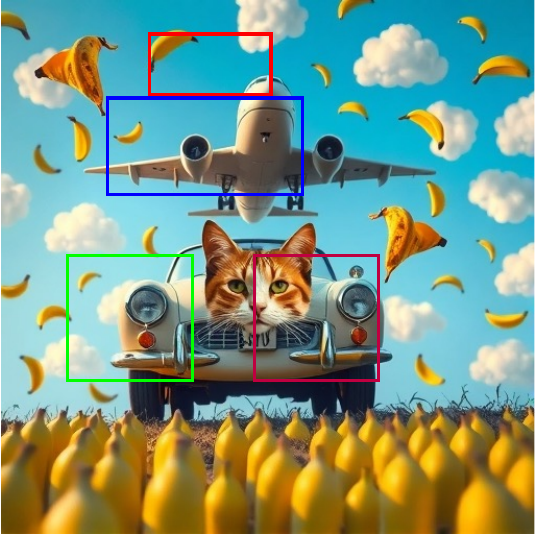} &
        \includegraphics[width=0.17\textwidth]{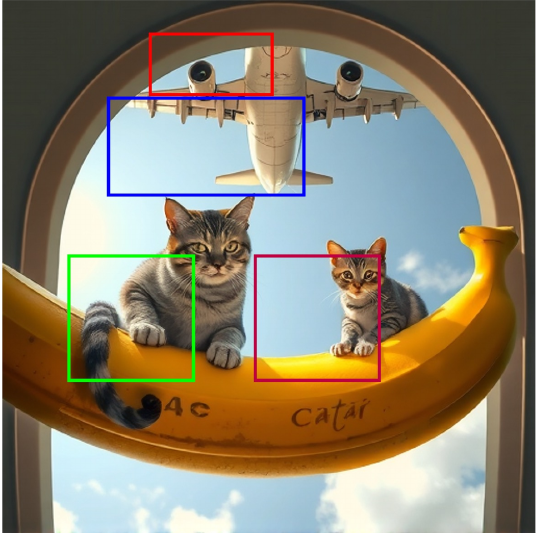} &
        \includegraphics[width=0.17\textwidth]{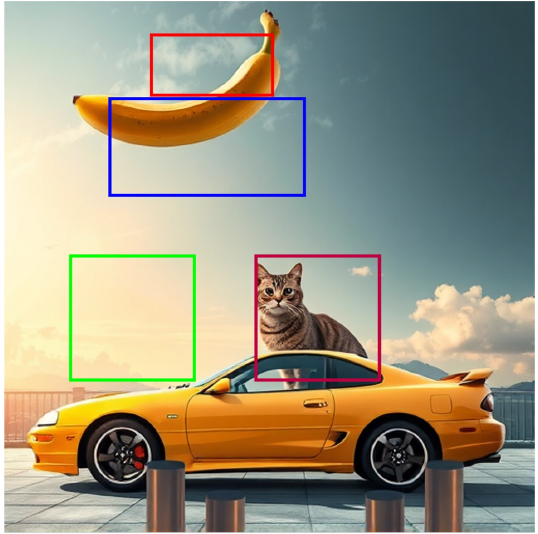} &
        \includegraphics[width=0.17\textwidth]{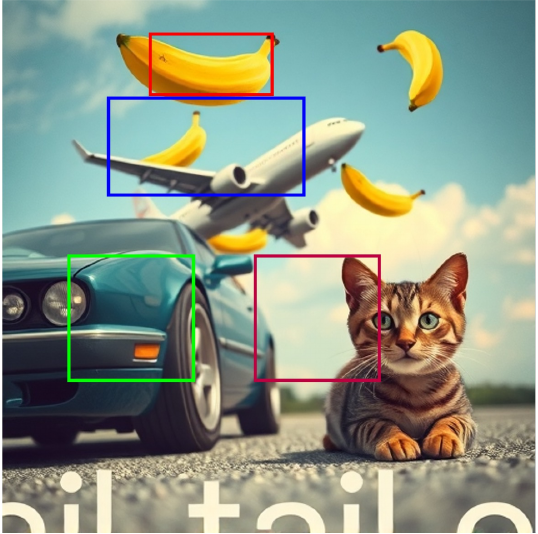} \\    
      & & \multicolumn{5}{c}{\textit{``A \textcolor{green}{car} and \textcolor{purple}{cat} beneath an \textcolor{blue}{airplane} and below a \textcolor{red}{banana}...''}} \\
      \addlinespace[3pt]
      
      & & \includegraphics[width=0.17\textwidth]{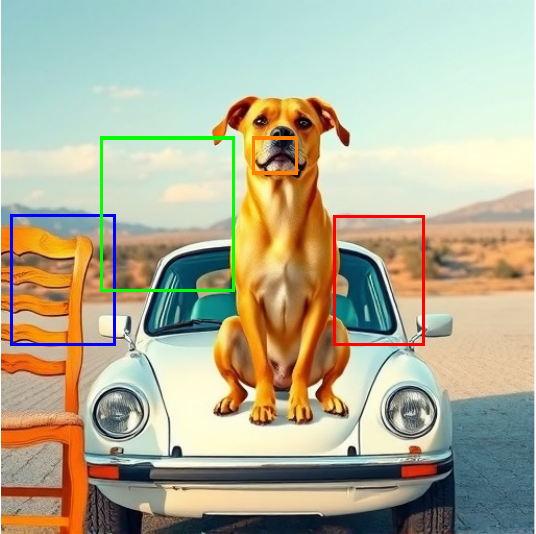} &
      \includegraphics[width=0.17\textwidth]{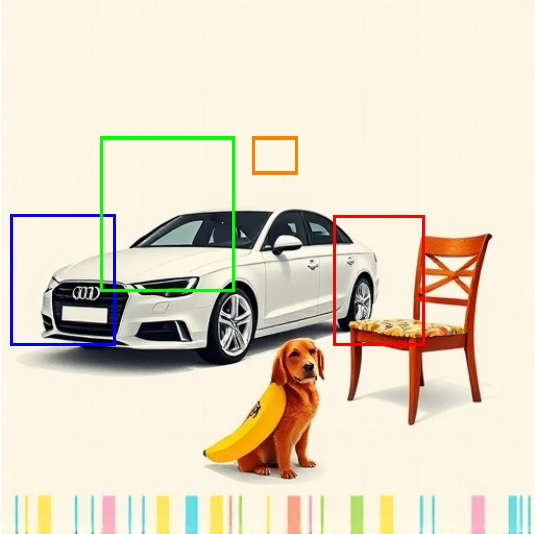} &
      \includegraphics[width=0.17\textwidth]{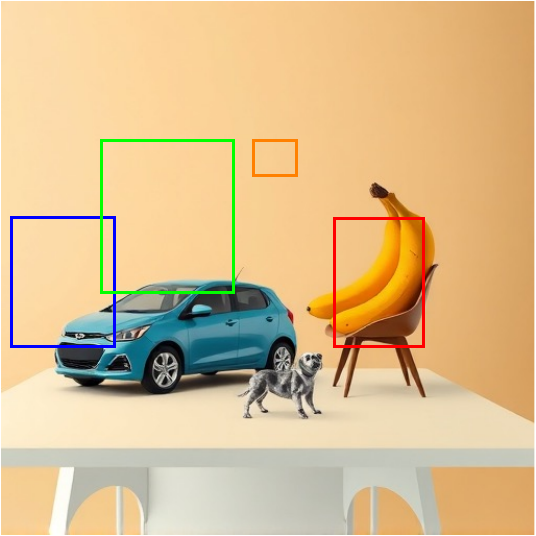} &
      \includegraphics[width=0.17\textwidth]{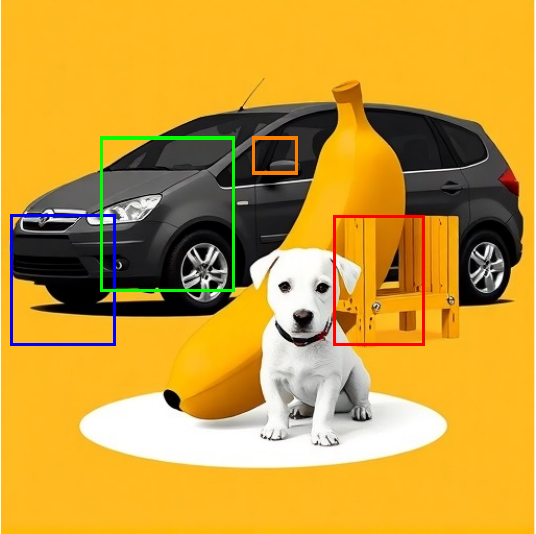} &
      \includegraphics[width=0.17\textwidth]{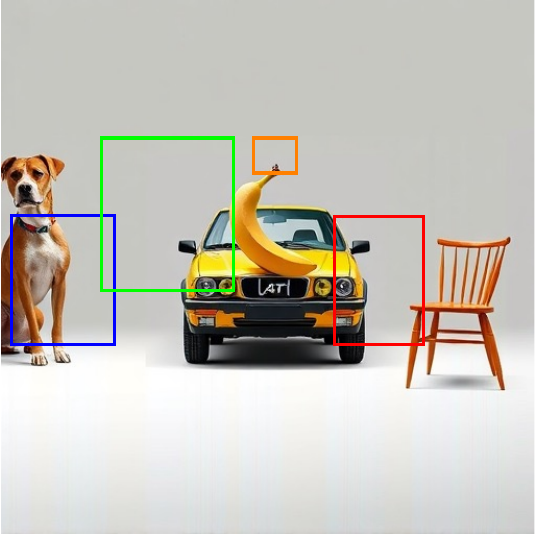} \\
      & & \multicolumn{5}{c}{\textit{``A \textcolor{green}{car} and a \textcolor{orange}{banana} in the middle of \textcolor{blue}{dog} and \textcolor{red}{chair}...''}} \\
      \addlinespace[3pt]

      & & \includegraphics[width=0.17\textwidth]{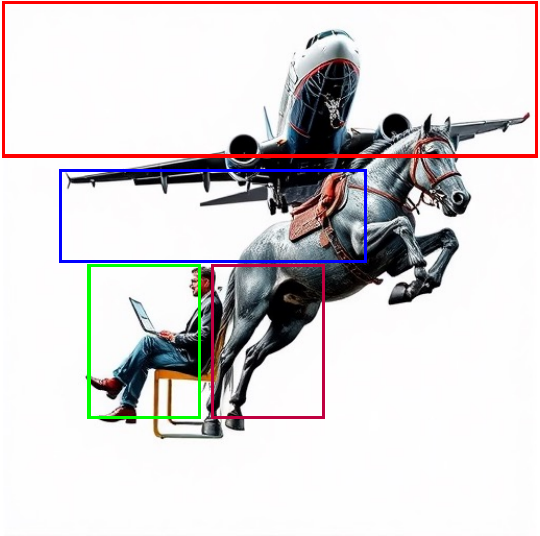} &
      \includegraphics[width=0.17\textwidth]{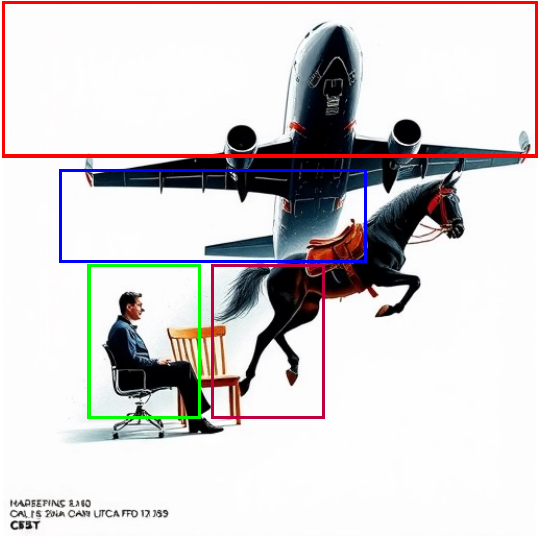} &
      \includegraphics[width=0.17\textwidth]{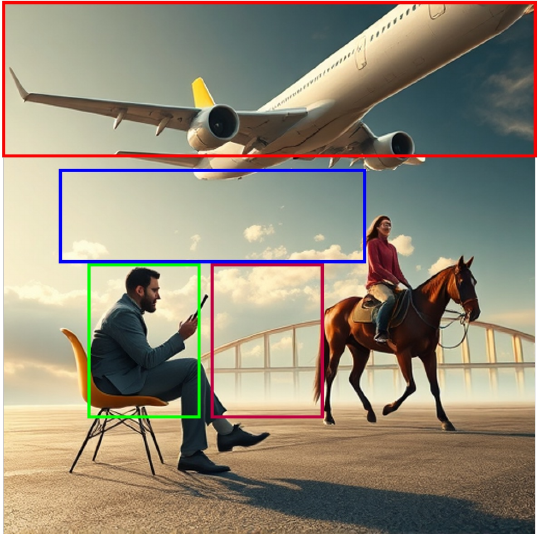} &
      \includegraphics[width=0.17\textwidth]{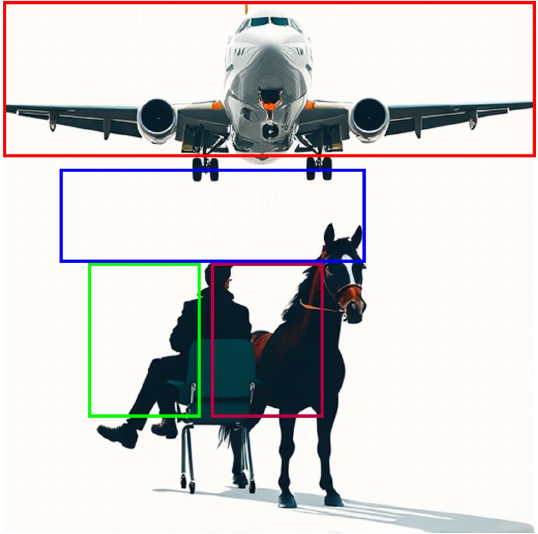} &
      \includegraphics[width=0.17\textwidth]{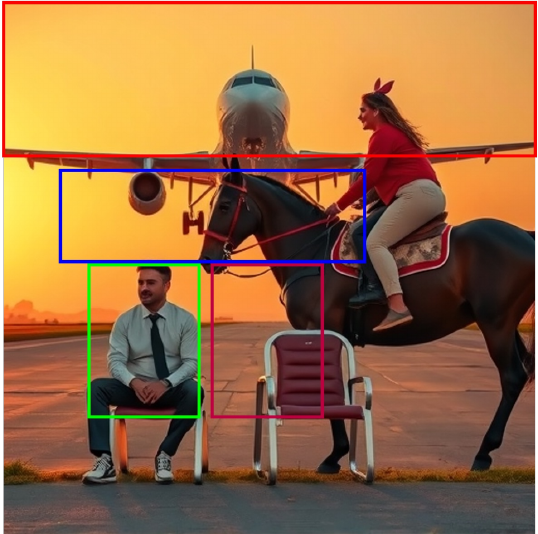} \\
      & & \multicolumn{5}{c}{\textit{``A \textcolor{green}{person} and \textcolor{purple}{chair} beneath an \textcolor{red}{airplane} and on a \textcolor{blue}{horse}...''}} \\
      
      \midrule % 구분선
      
      % --- Simple Group (3 Rows) ---
      \multirow{6}{*}{\raisebox{-3.0cm}{\rotatebox{90}{\textbf{Layout-to-Image}}}} & 
      \multirow{6}{*}{\raisebox{-3.0cm}{\rotatebox{90}{\small \textbf{Simple ($<$ 4 obj.)}}}} & 
      \includegraphics[width=0.17\textwidth]{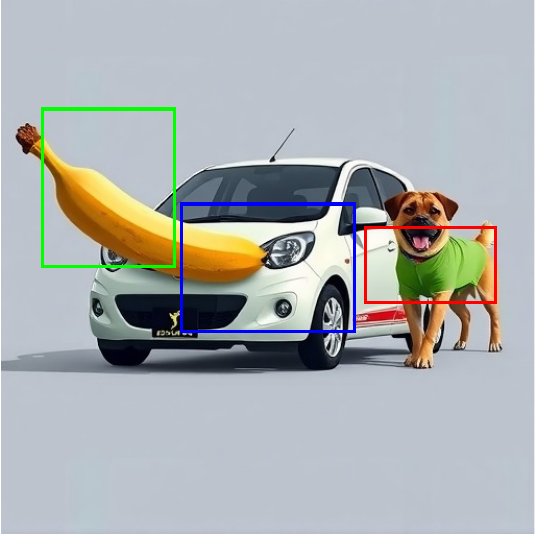} &
      \includegraphics[width=0.17\textwidth]{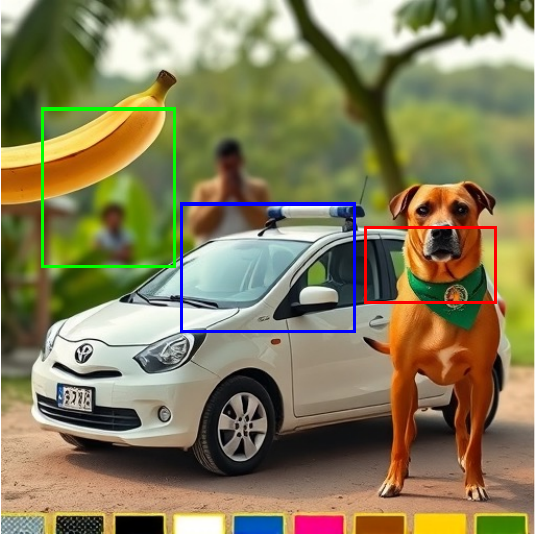} &
      \includegraphics[width=0.17\textwidth]{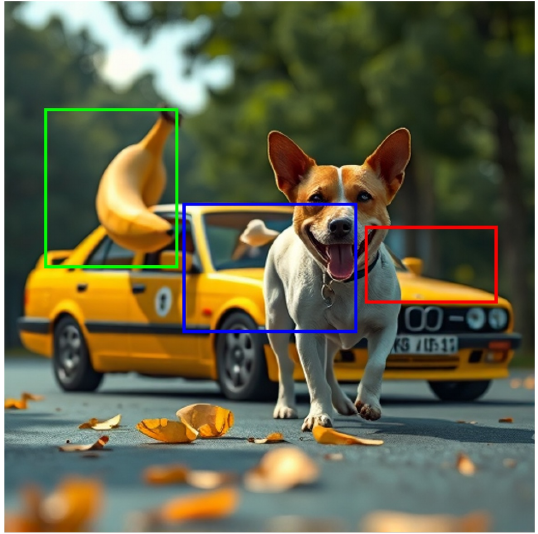} &
      \includegraphics[width=0.17\textwidth]{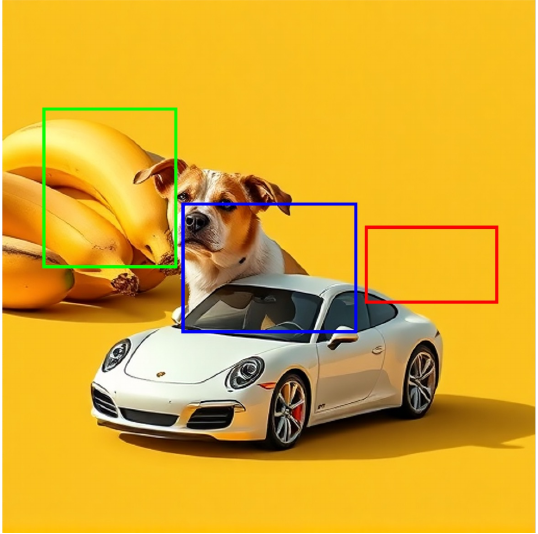} &
      \includegraphics[width=0.17\textwidth]{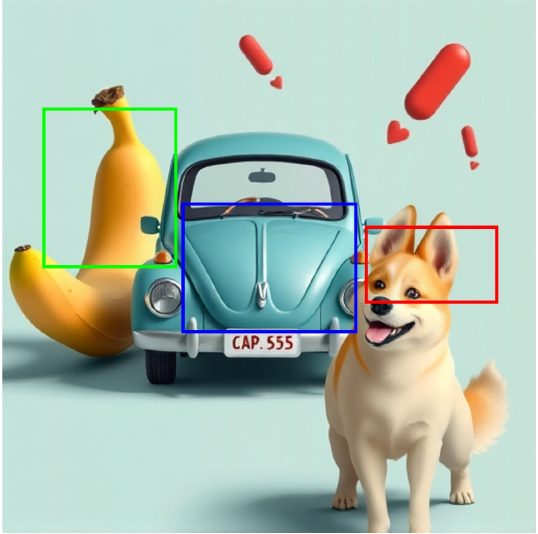} \\
      & & \multicolumn{5}{c}{\textit{``A \textcolor{blue}{car} among \textcolor{green}{banana} and \textcolor{red}{dog}....''}} \\
      \addlinespace[3pt]

      & & \includegraphics[width=0.17\textwidth]{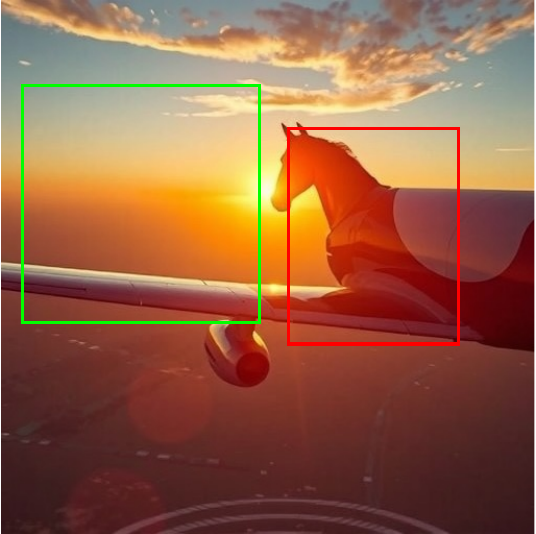} &
      \includegraphics[width=0.17\textwidth]{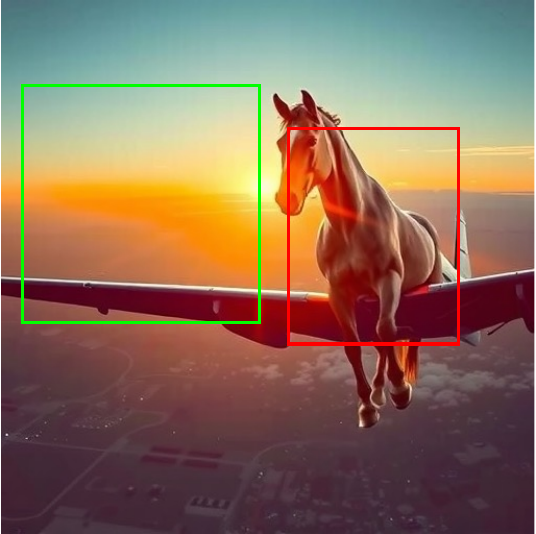} &
      \includegraphics[width=0.17\textwidth]{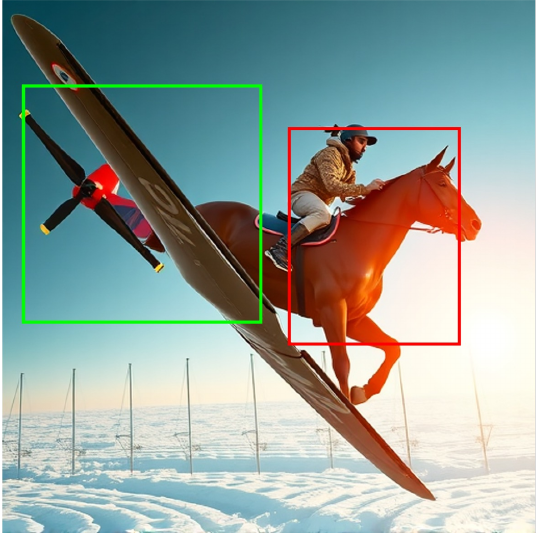} &
      \includegraphics[width=0.17\textwidth]{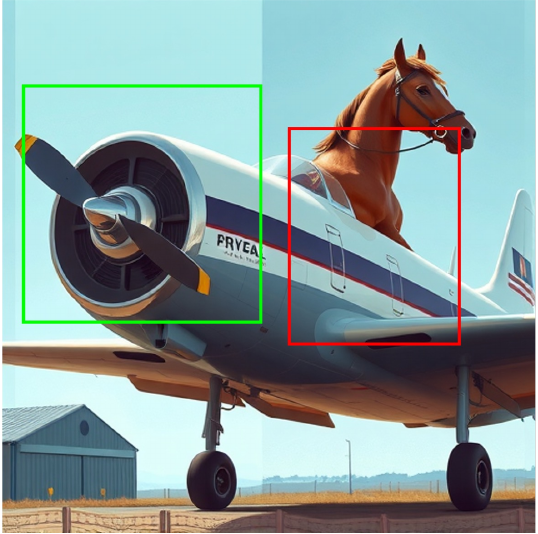} &
      \includegraphics[width=0.17\textwidth]{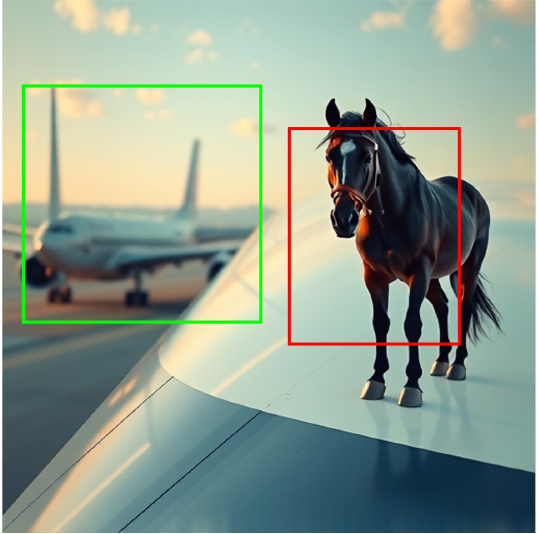} \\
      & & \multicolumn{5}{c}{\textit{``A \textcolor{red}{horse} on the right of an \textcolor{green}{airplane}....''}} \\
      \addlinespace[3pt]

      & & \includegraphics[width=0.17\textwidth]{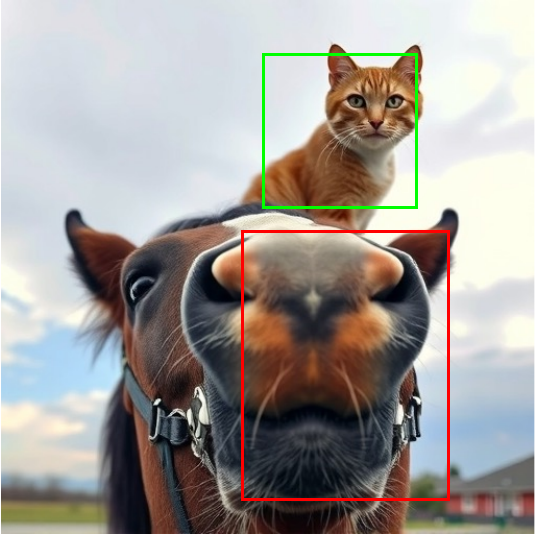} &
      \includegraphics[width=0.17\textwidth]{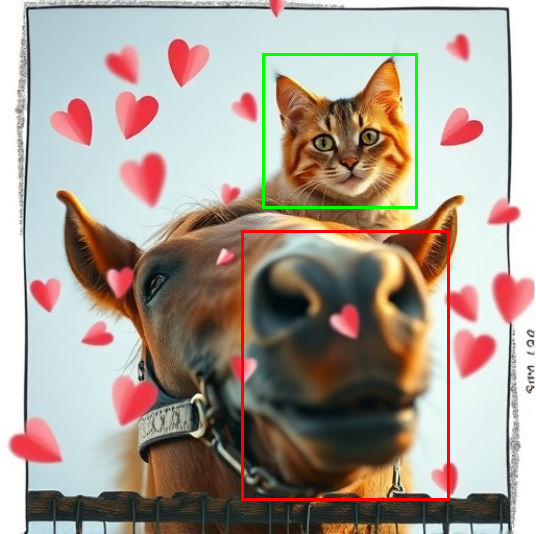} &
      \includegraphics[width=0.17\textwidth]{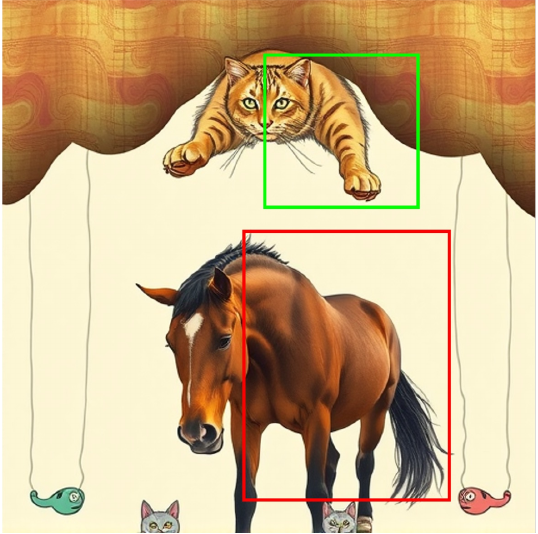} &
      \includegraphics[width=0.17\textwidth]{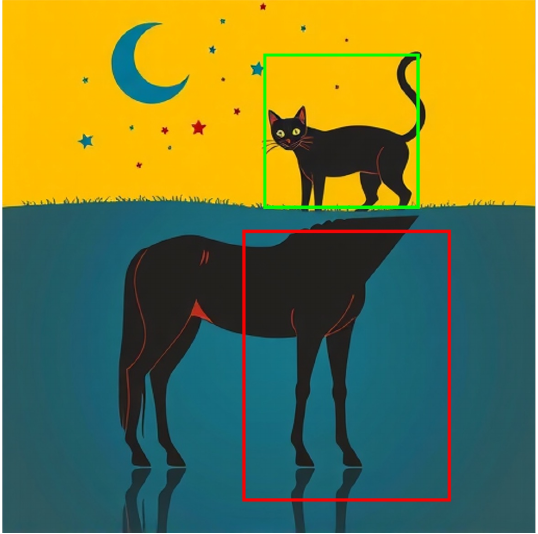} &
      \includegraphics[width=0.17\textwidth]{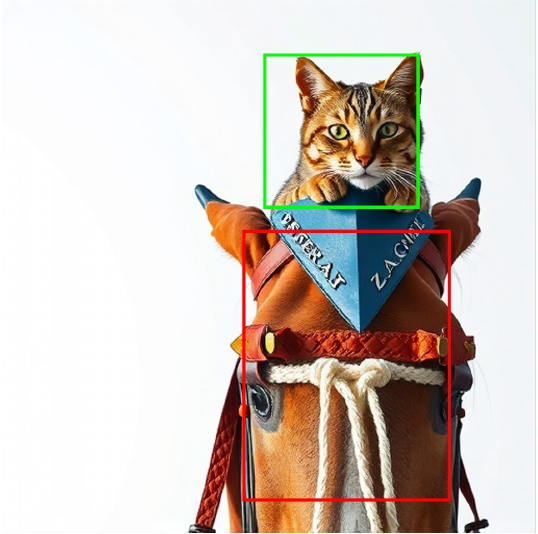} \\
      & & \multicolumn{5}{c}{\textit{``A \textcolor{red}{horse} beneath a \textcolor{green}{cat}...''}} \\
      \bottomrule

    \end{tabular}

    \caption{Qualitative comparison on complex and simple layouts. PATHS (Ours) demonstrates superior performance in aligning with spatial constraints across varying levels of difficulty.}
    \label{fig:qualitative_appendix}
  \end{center}
\end{figure}

\begin{figure}[t]
  \begin{center}
    % 페이지 하단 잘림을 방지하기 위해 간격을 타이트하게 유지합니다.
    \setlength{\tabcolsep}{1.5pt} 
    \renewcommand{\arraystretch}{0.4} 
    \begin{tabular}{c c ccccc}
      %\toprule
      % 최상단 헤더 (Prior / Posterior 그룹화)
      %& & \multicolumn{2}{c}{\small \textbf{Sampling from Prior}} 
      %& \multicolumn{3}{c}{\small \textbf{Sampling from Posterior}} \\
      %\cmidrule(lr){3-4} \cmidrule(lr){5-7}
      
      % 모델별 헤더
      & & \footnotesize TDS & \footnotesize DAS & \footnotesize $\psi$-Sampler & \footnotesize Best-of-4 & \footnotesize \textbf{PATHS (Ours)} \\
      \midrule
      
      % --- Complex Group (3 Rows) ---
       
      \multirow{6}{*}{\raisebox{-4.0cm}{\rotatebox{90}{\textbf{Quantity-Aware}}}} & 
      \multirow{6}{*}{\raisebox{-5.0cm}{\rotatebox{90}{\small \textbf{Complex ($>=$ 25 quantities.)}}}} & 
      \includegraphics[width=0.17\textwidth]{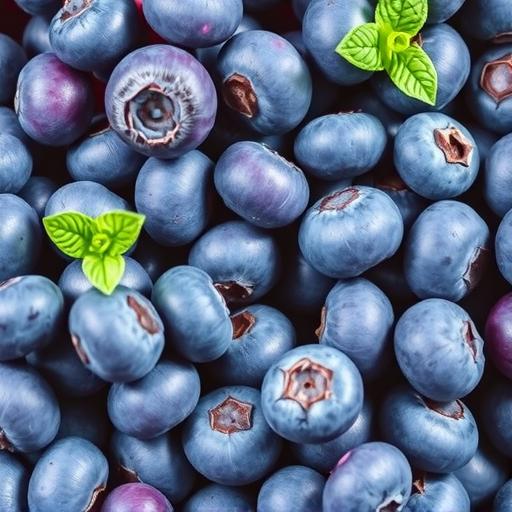} &
      \includegraphics[width=0.17\textwidth]{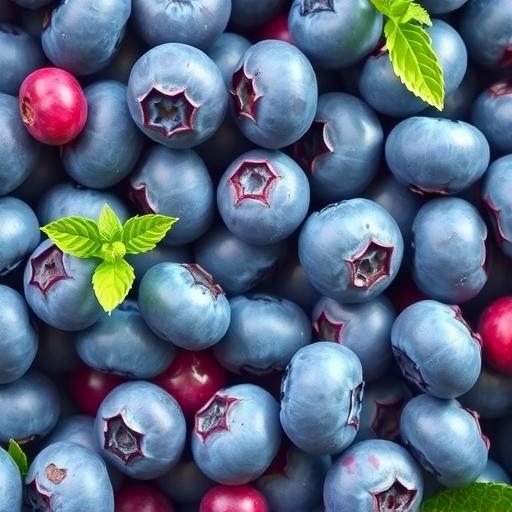} &
      \includegraphics[width=0.17\textwidth]{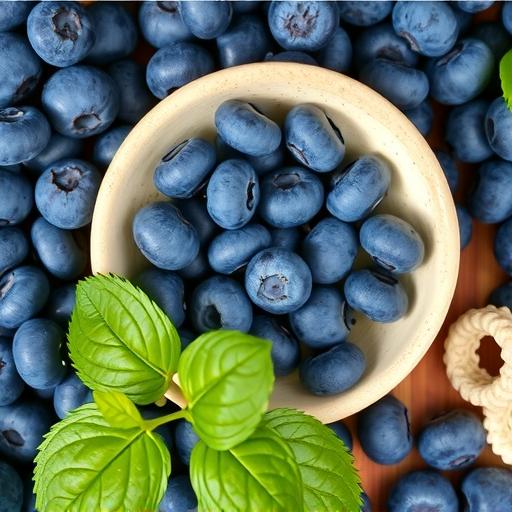} &
      \includegraphics[width=0.17\textwidth]{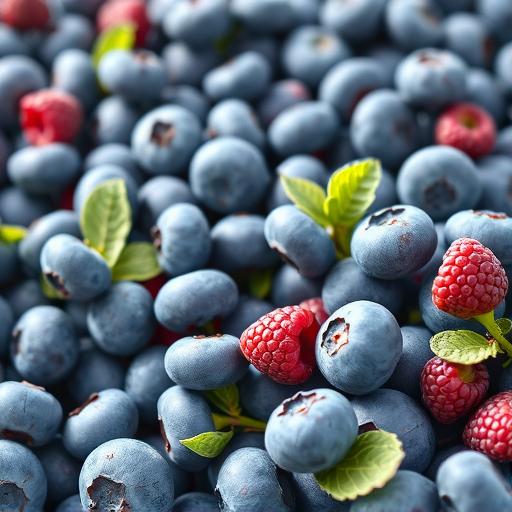} &
      \includegraphics[width=0.17\textwidth]{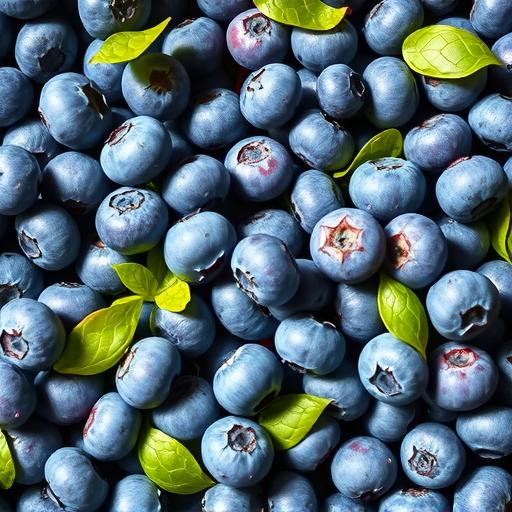} \\
      & & \multicolumn{5}{c}{\scriptsize \textit{``\textcolor{red}{82} blueberries''}} \\
      \addlinespace[3pt]

      & & \includegraphics[width=0.17\textwidth]{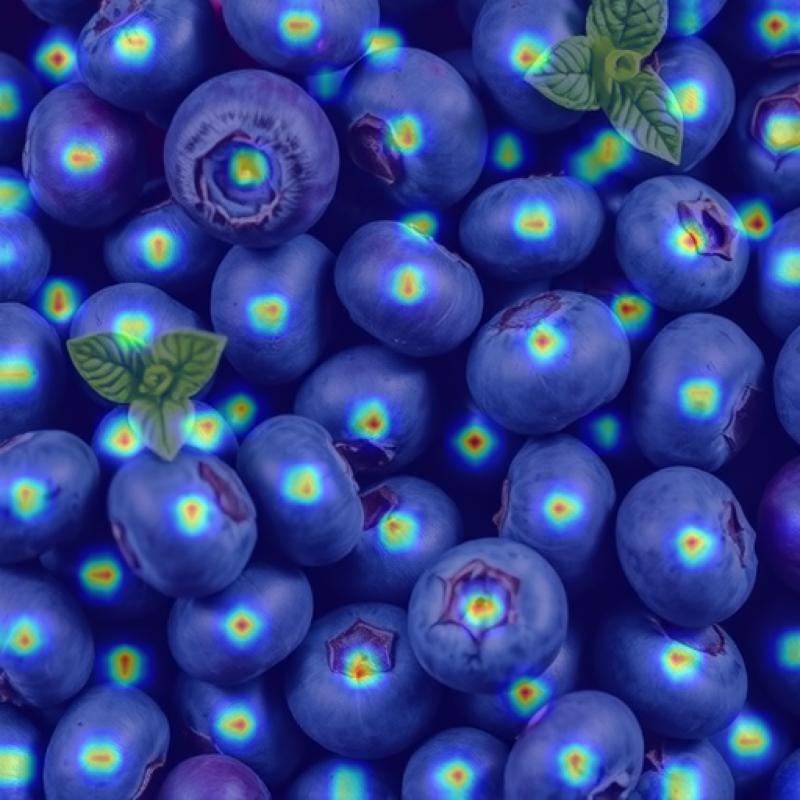} &
      \includegraphics[width=0.17\textwidth]{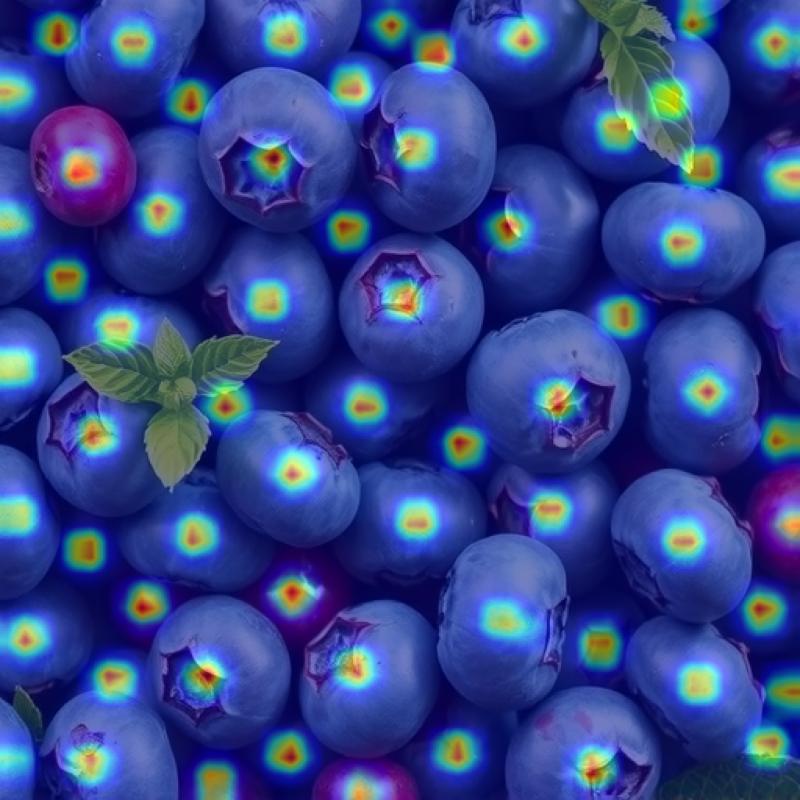} &
      \includegraphics[width=0.17\textwidth]{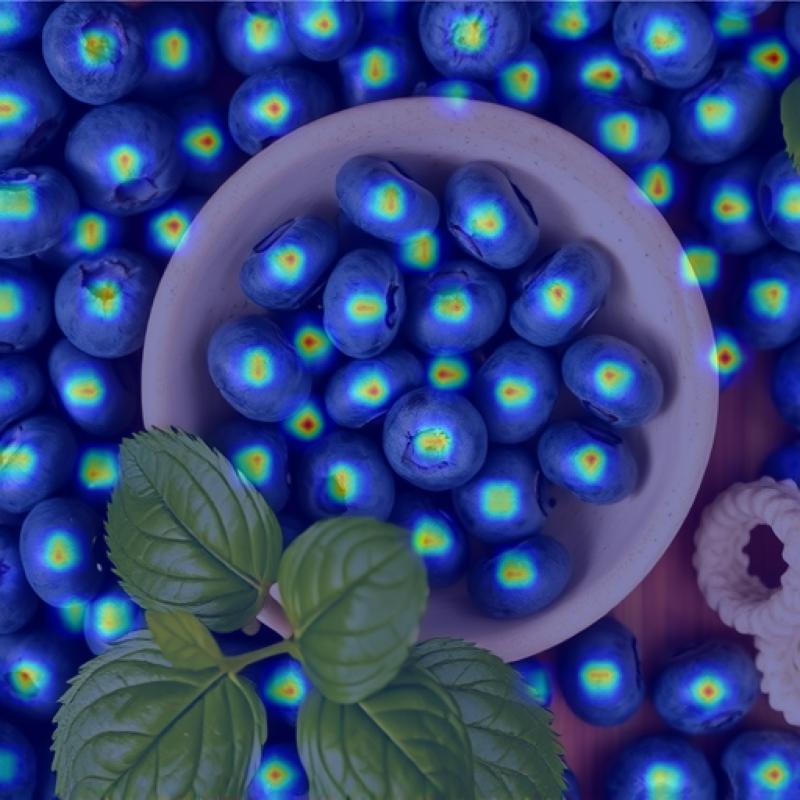} &
      \includegraphics[width=0.17\textwidth]{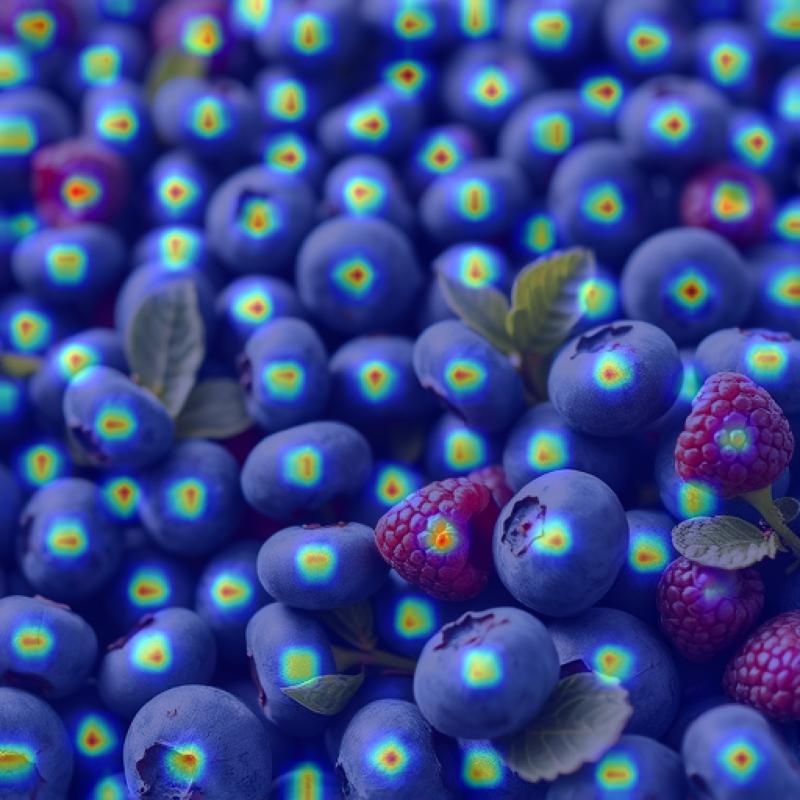} &
      \includegraphics[width=0.17\textwidth]{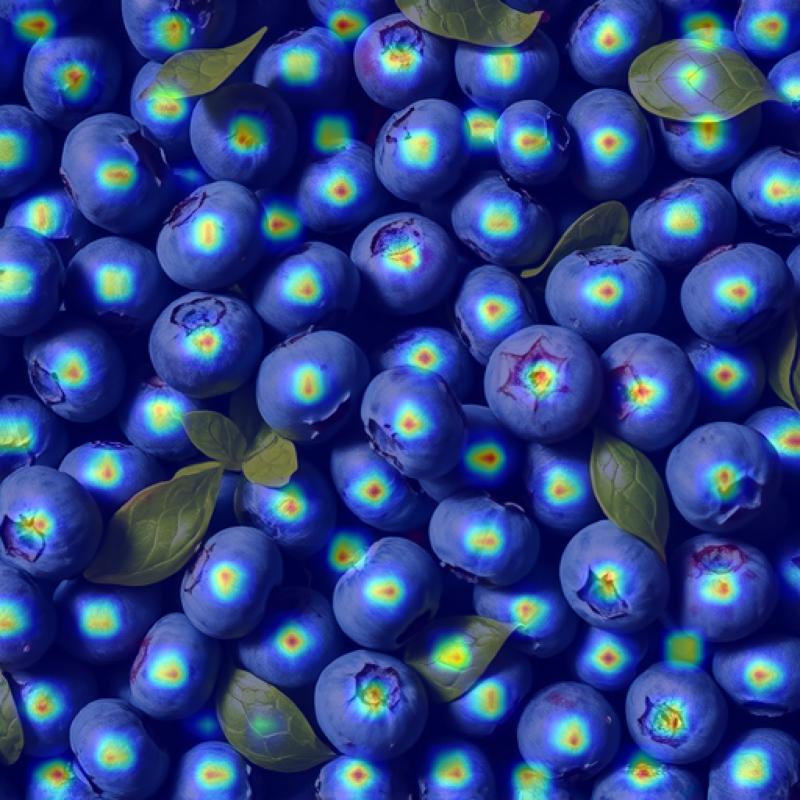} \\
      & & 57 \small($\Delta$25) & 62 \small($\Delta$20) & 63 \small($\Delta$19) & 83 \small($\Delta$-1) & 82 \small(\textcolor{blue}{$\Delta$0)} \\
      %\midrule 

      %\multirow{6}{*}{\raisebox{-3.0cm}{\rotatebox{90}{\textbf{Quantity-Aware}}}} & 
      %\multirow{6}{*}{\raisebox{-3.0cm}{\rotatebox{90}{\small \textbf{Complex ($>=$ 25 quantities.)}}}} 
      & & 
      \includegraphics[width=0.17\textwidth]{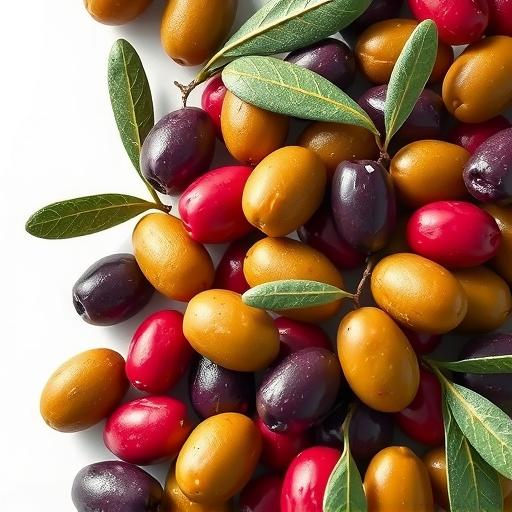} &
      \includegraphics[width=0.17\textwidth]{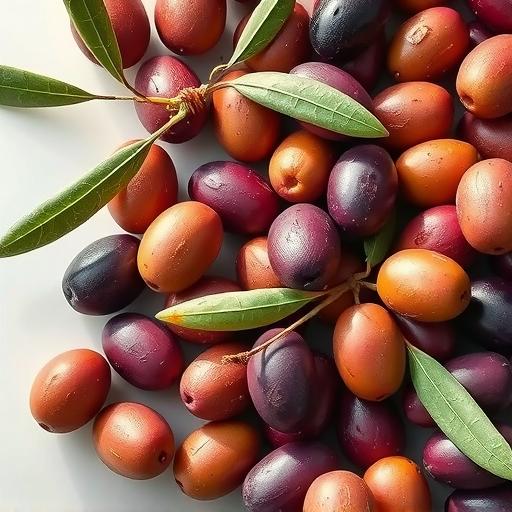} &
      \includegraphics[width=0.17\textwidth]{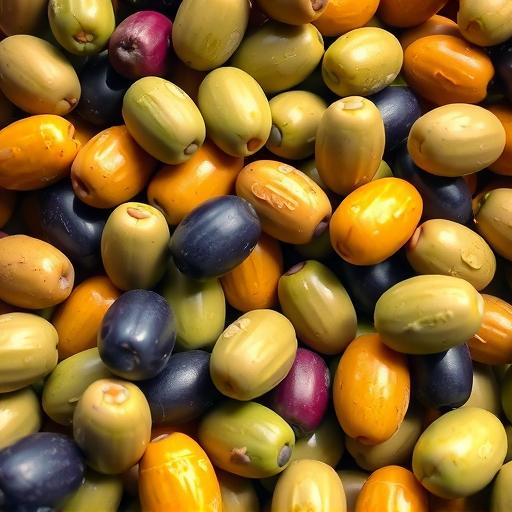} &
      \includegraphics[width=0.17\textwidth]{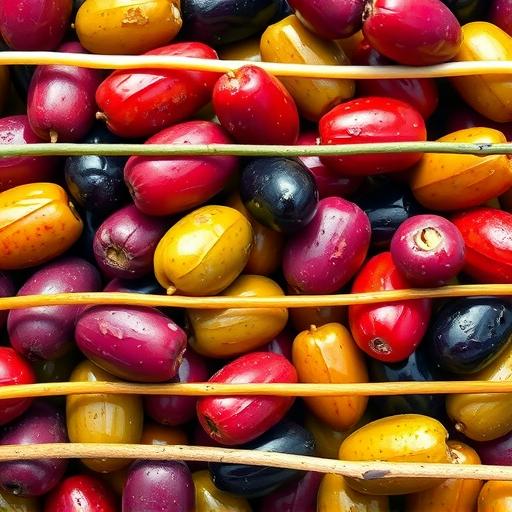} &
      \includegraphics[width=0.17\textwidth]{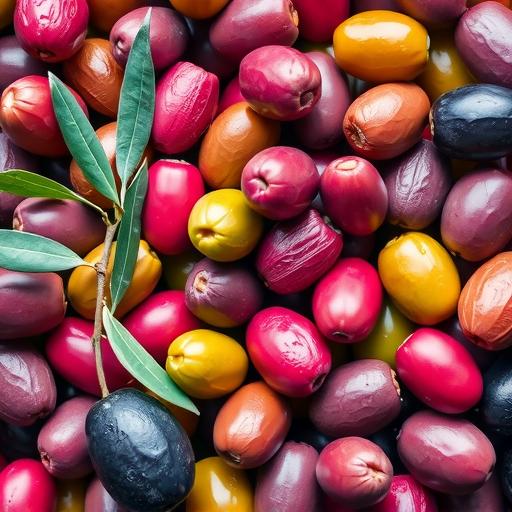} \\
      & & \multicolumn{5}{c}{\scriptsize \textit{``\textcolor{red}{63} olives''}} \\
      \addlinespace[3pt]

      & & \includegraphics[width=0.17\textwidth]{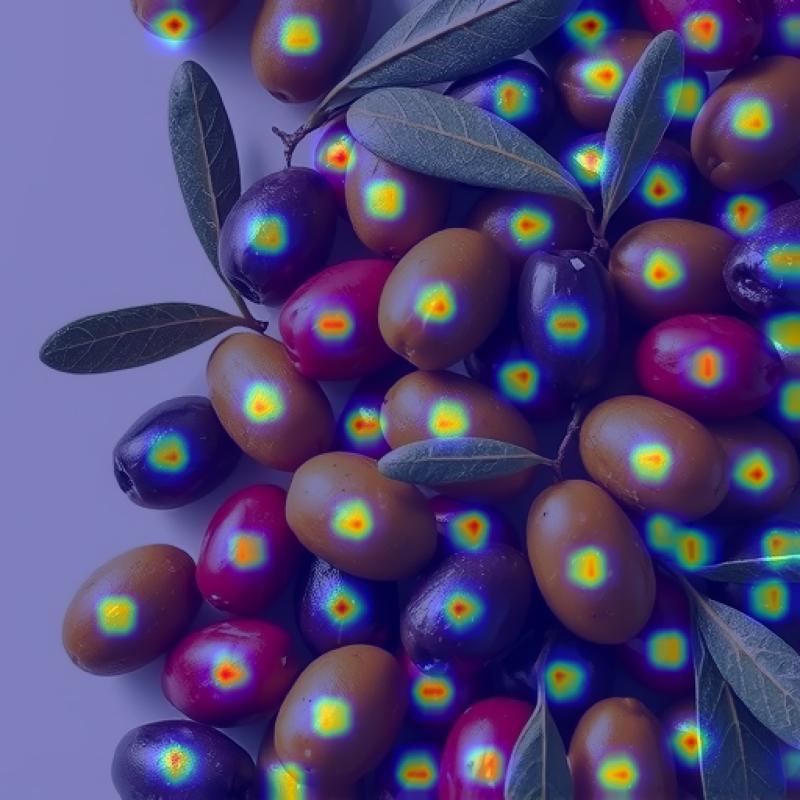} &
      \includegraphics[width=0.17\textwidth]{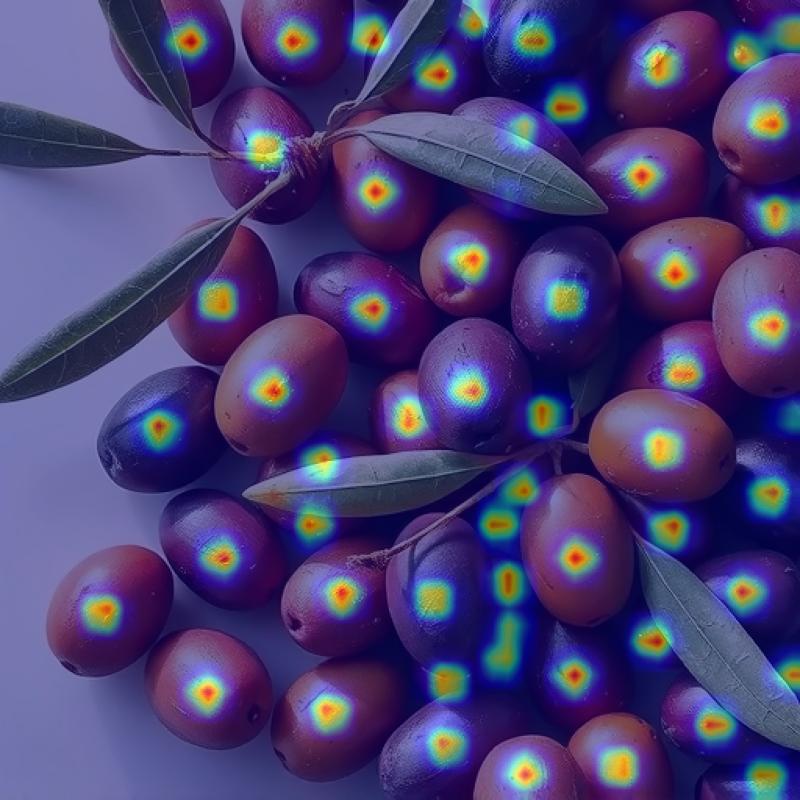} &
      \includegraphics[width=0.17\textwidth]{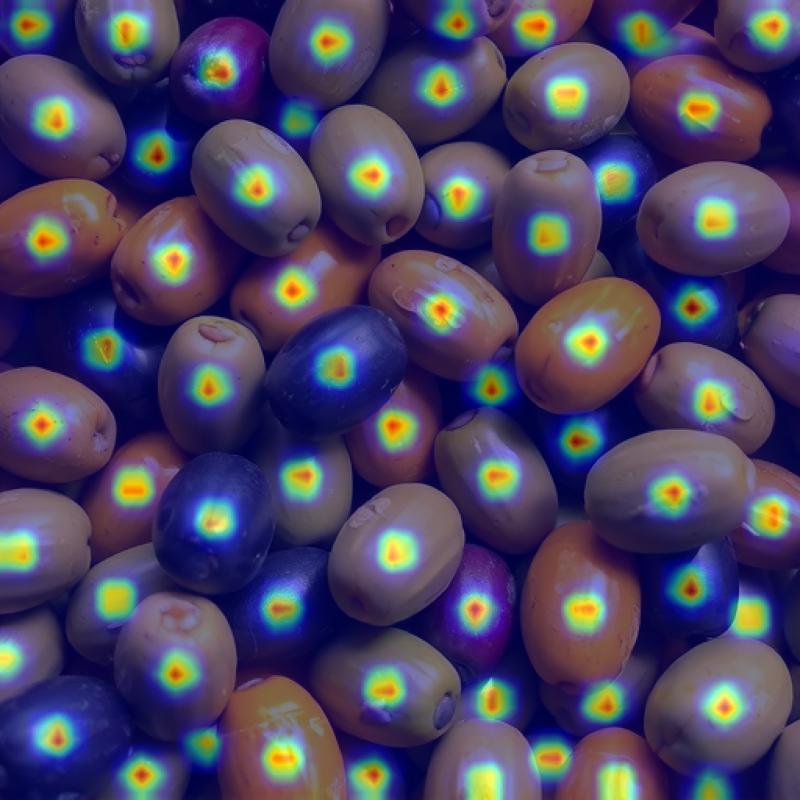} &
      \includegraphics[width=0.17\textwidth]{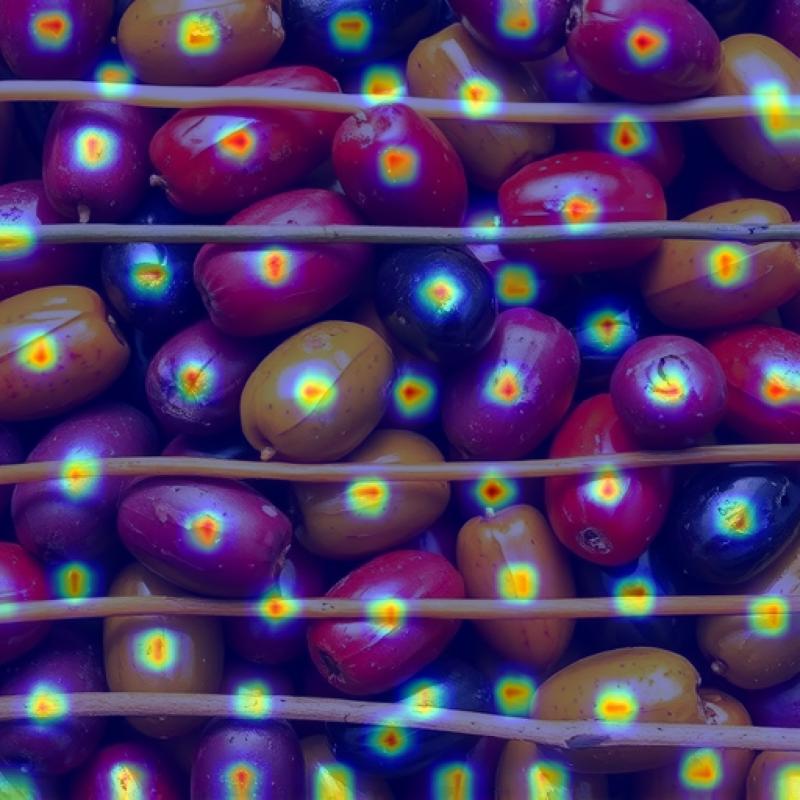} &
      \includegraphics[width=0.17\textwidth]{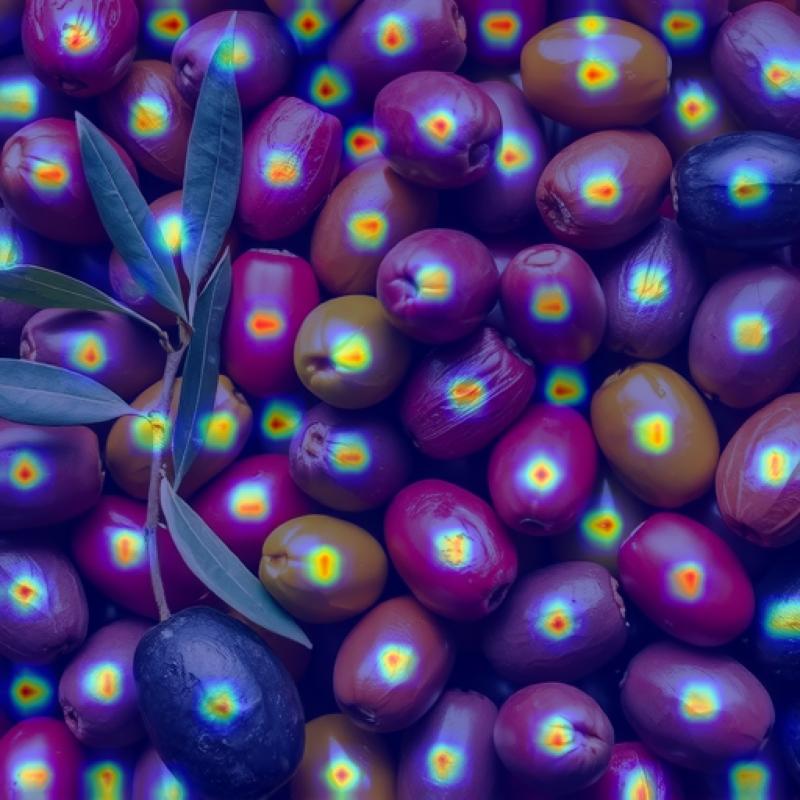} \\
      & & 51 \small($\Delta$12) & 53 \small($\Delta$10) & 62 \small($\Delta$1) & 56 \small($\Delta$7) & 63 \small(\textcolor{blue}{$\Delta$0)} \\
      \midrule % 구분선
      
      % --- Simple Group (3 Rows) ---
      \multirow{6}{*}{\raisebox{-4.0cm}{\rotatebox{90}{\textbf{Quantity-Aware}}}} & 
      \multirow{6}{*}{\raisebox{-5.0cm}{\rotatebox{90}{\small \textbf{Simple ($<$ 25 quantities.)}}}} & 
      \includegraphics[width=0.17\textwidth]{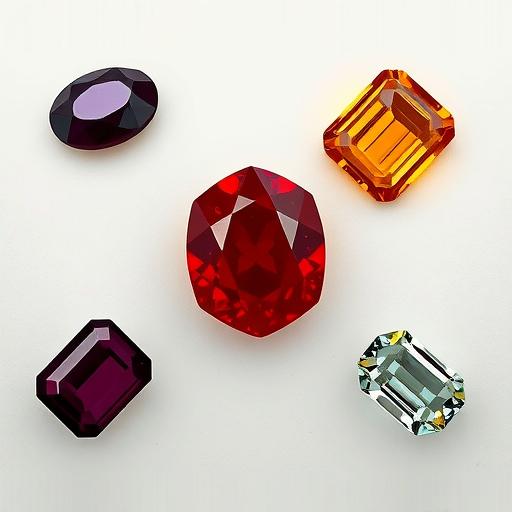} &
      \includegraphics[width=0.17\textwidth]{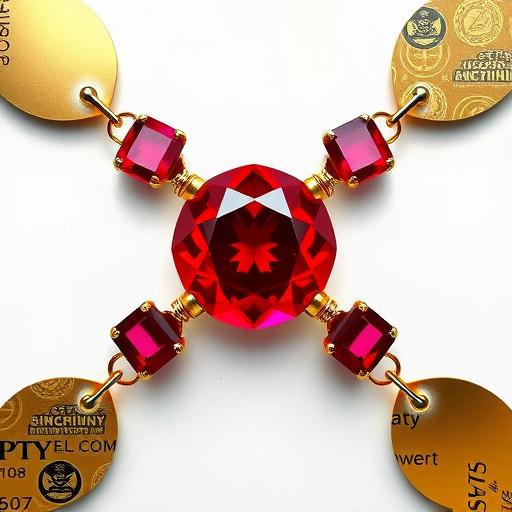} &
      \includegraphics[width=0.17\textwidth]{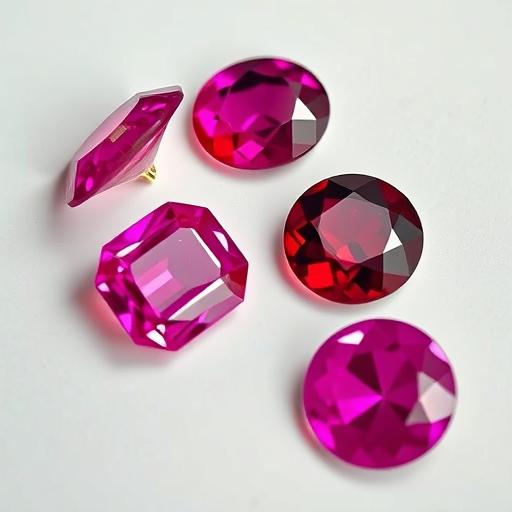} &
      \includegraphics[width=0.17\textwidth]{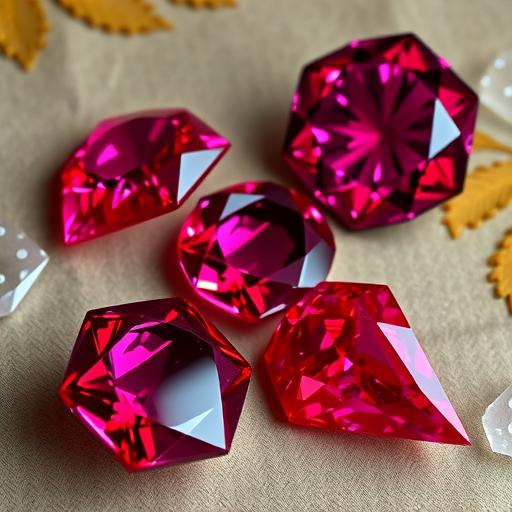} &
      \includegraphics[width=0.17\textwidth]{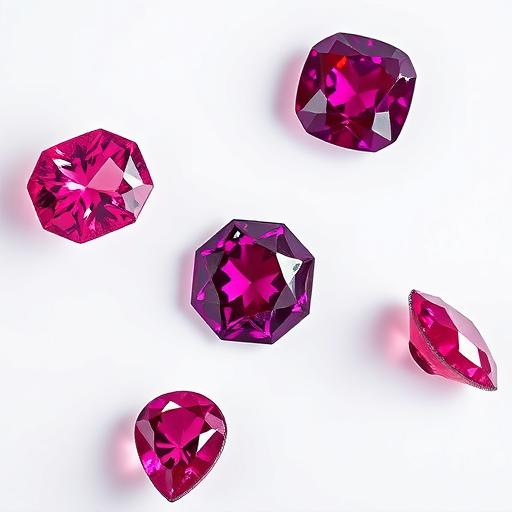} \\
      & & \multicolumn{5}{c}{\scriptsize \textit{``\textcolor{red}{5} rubies''}} \\
      \addlinespace[3pt]

      & & \includegraphics[ width=0.17\textwidth]{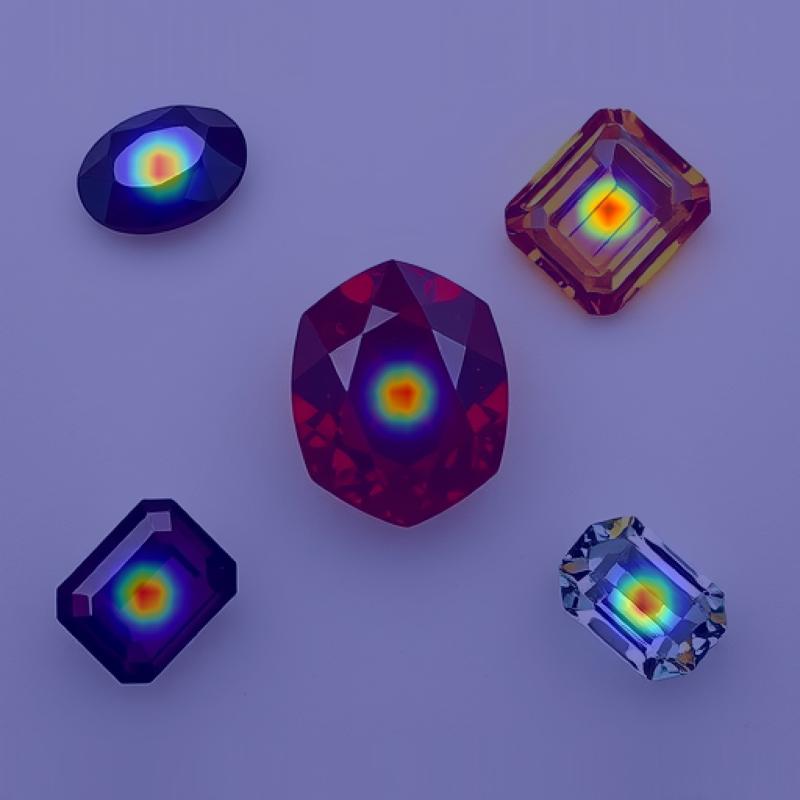} &
      \includegraphics[ width=0.17\textwidth]{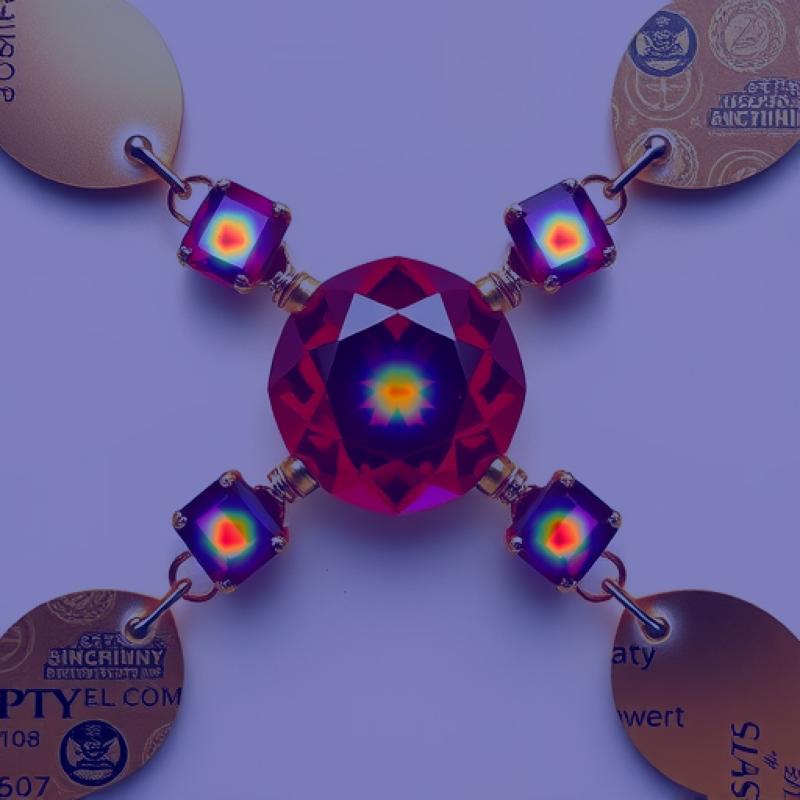} &
      \includegraphics[width=0.17\textwidth]{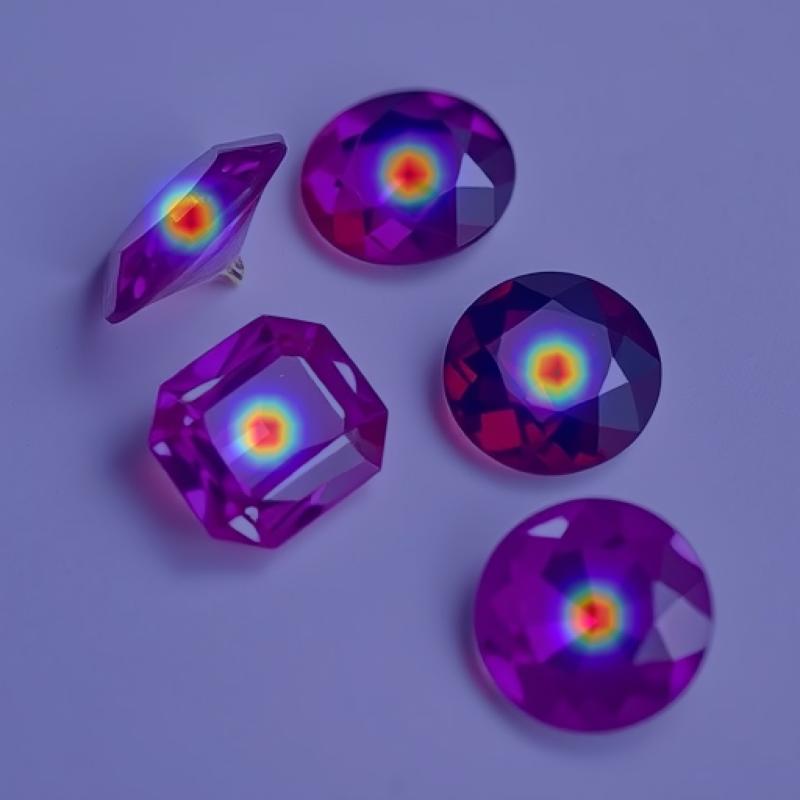} &
      \includegraphics[width=0.17\textwidth]{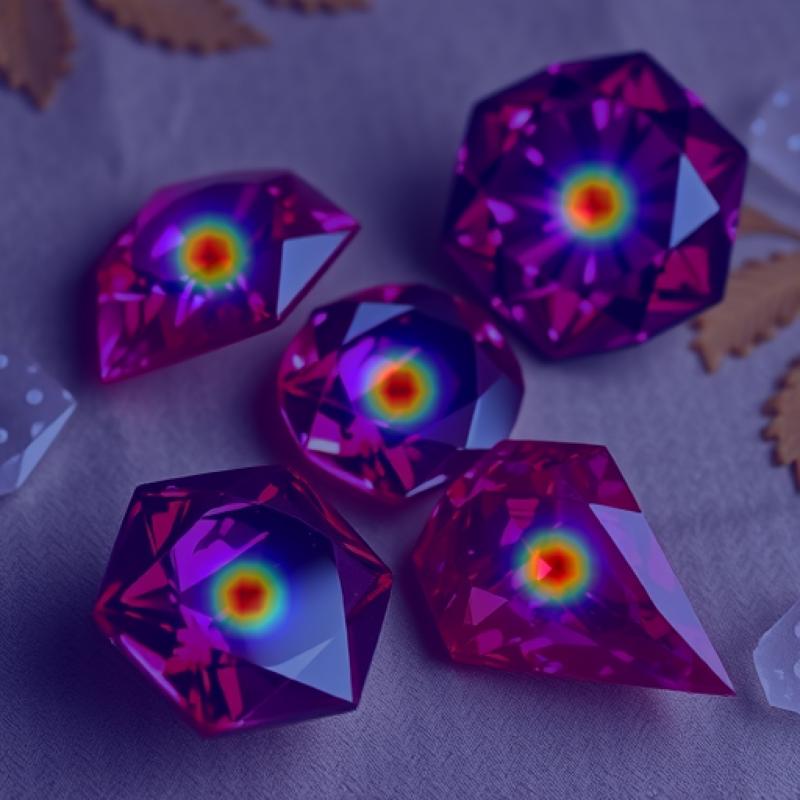} &
      \includegraphics[width=0.17\textwidth]{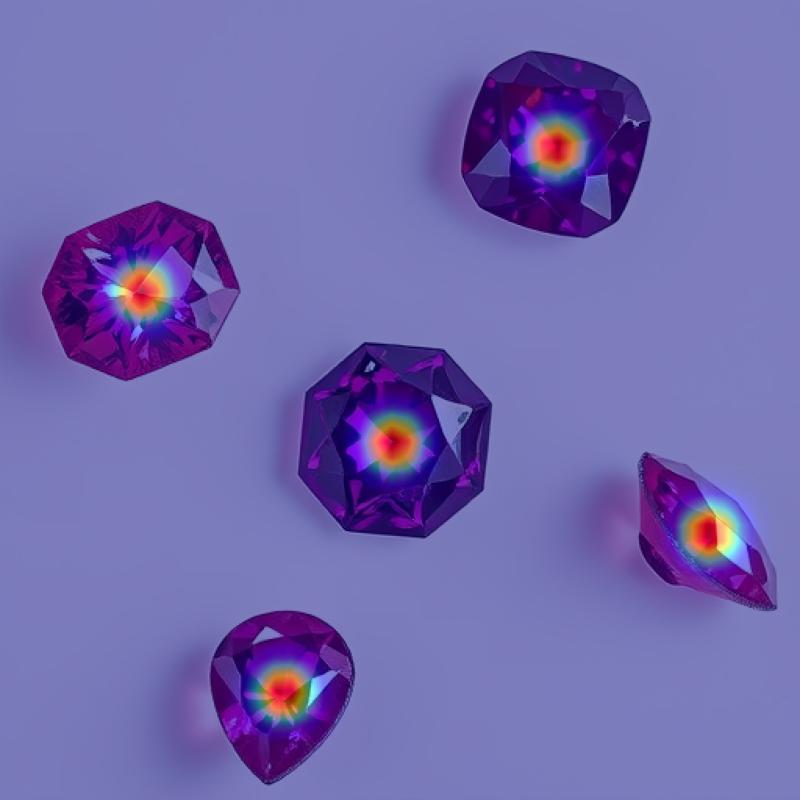} \\
      & & 5 \small($\Delta$0) & 5 \small($\Delta$0) & 5 \small($\Delta$0) & 5 \small($\Delta$0) & 5 \small(\textcolor{blue}{$\Delta$0}) \\
      & & 
      \includegraphics[width=0.17\textwidth]{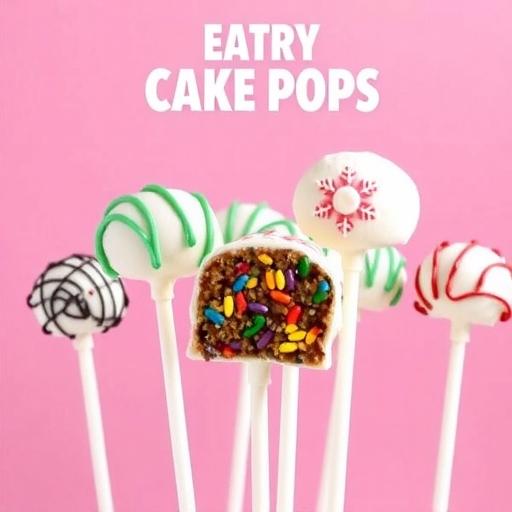} &
      \includegraphics[width=0.17\textwidth]{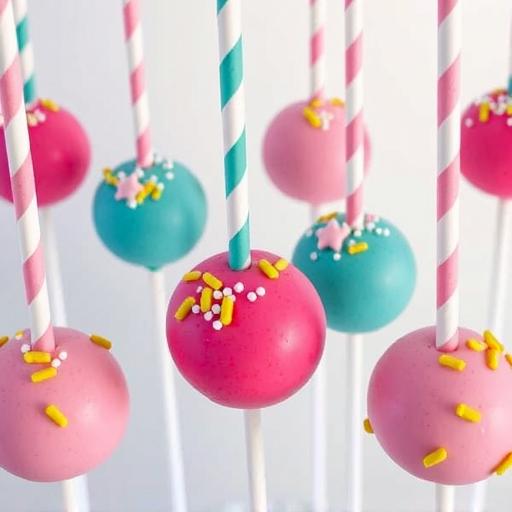} &
      \includegraphics[width=0.17\textwidth]{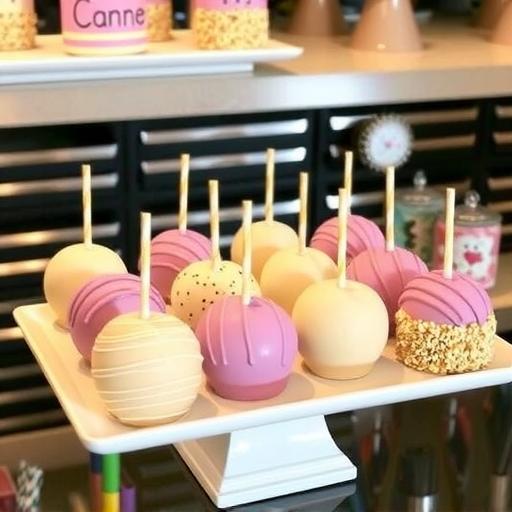} &
      \includegraphics[width=0.17\textwidth]{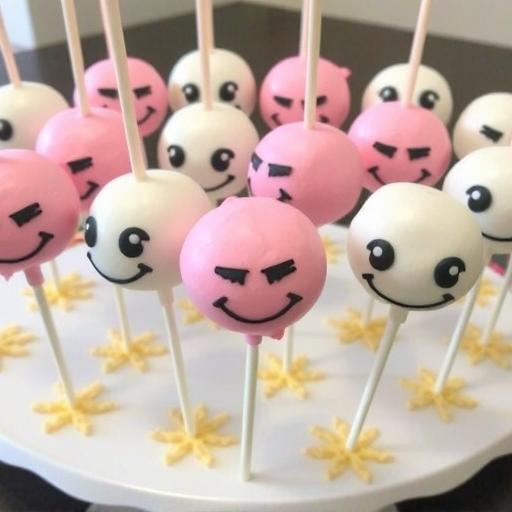} &
      \includegraphics[width=0.17\textwidth]{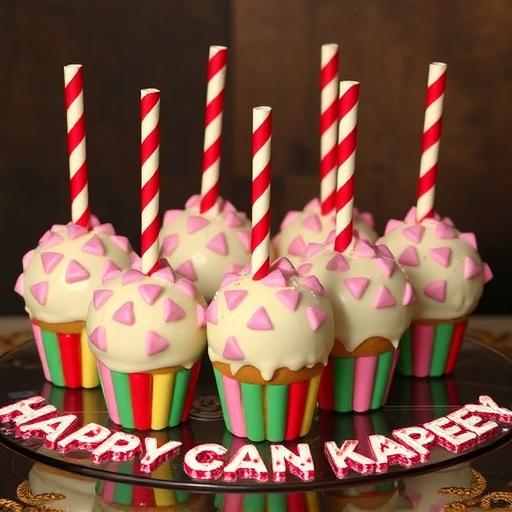} \\
      & & \multicolumn{5}{c}{\scriptsize \textit{``\textcolor{red}{14} cake pops''}} \\
      \addlinespace[3pt]

      & & \includegraphics[width=0.17\textwidth]{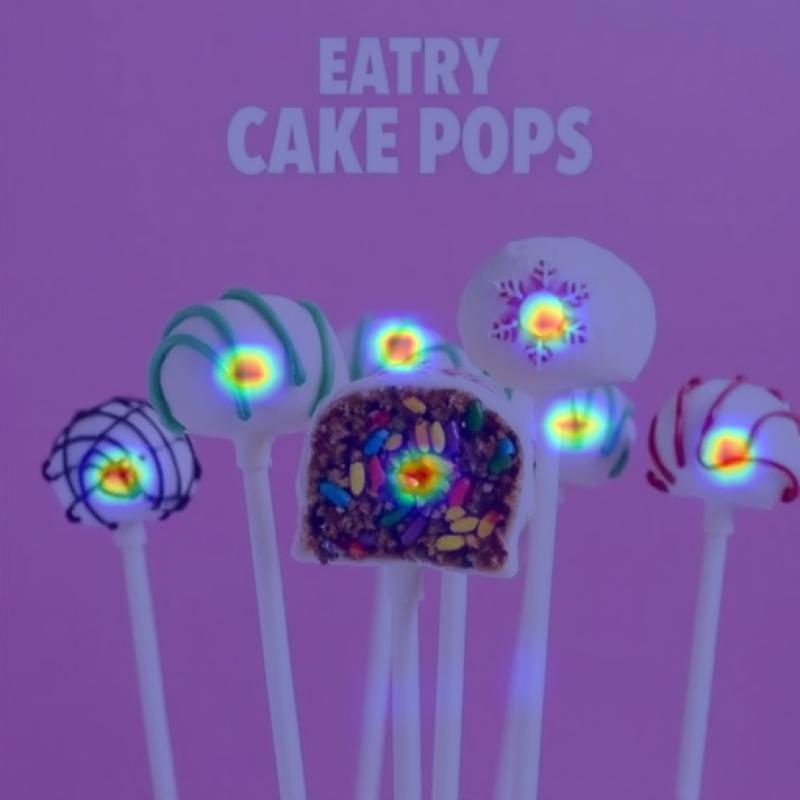} &
      \includegraphics[width=0.17\textwidth]{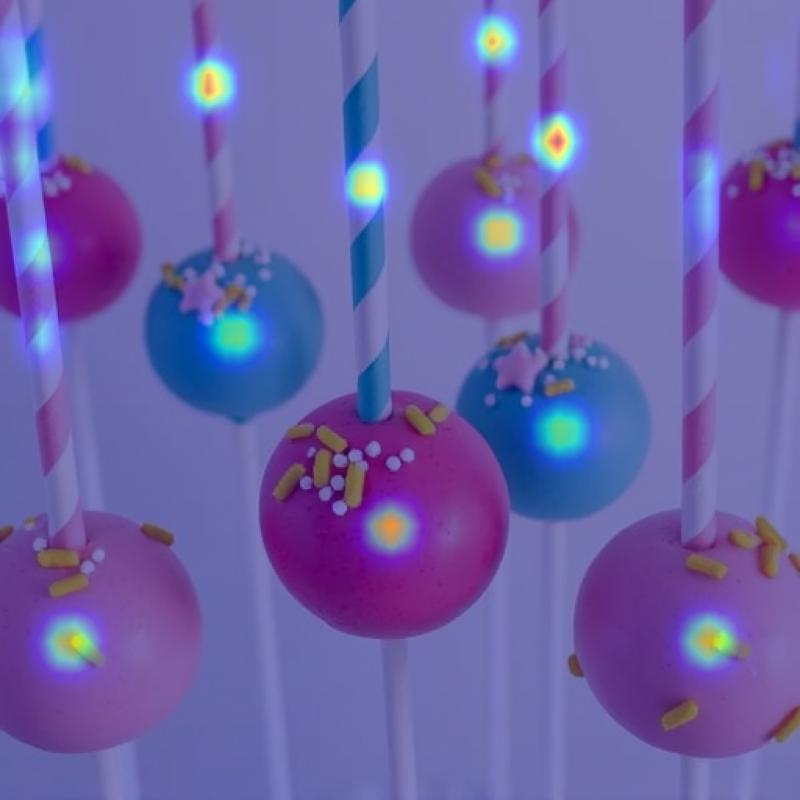} &
      \includegraphics[width=0.17\textwidth]{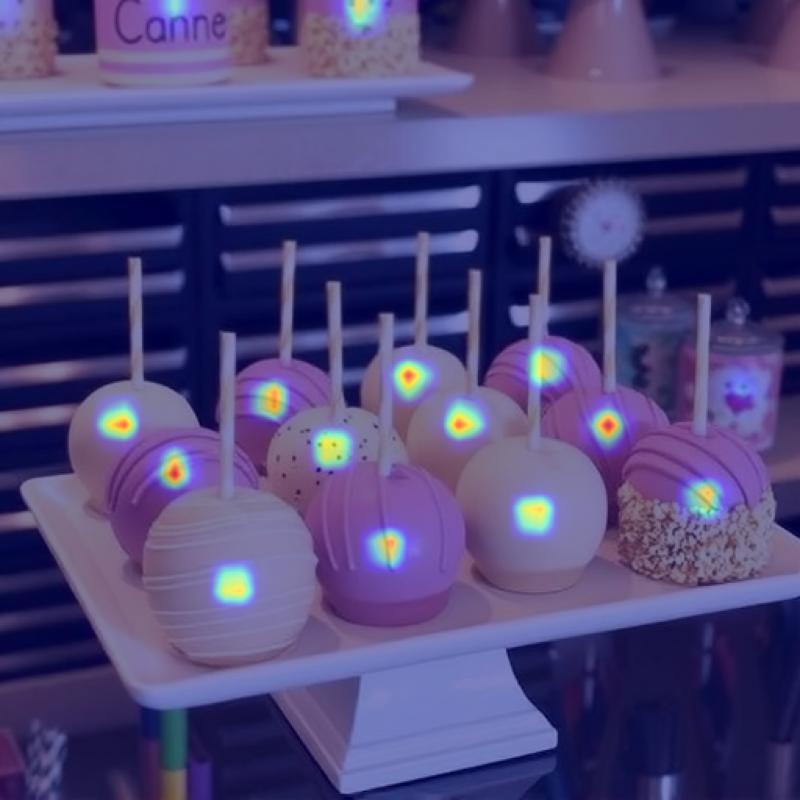} &
      \includegraphics[width=0.17\textwidth]{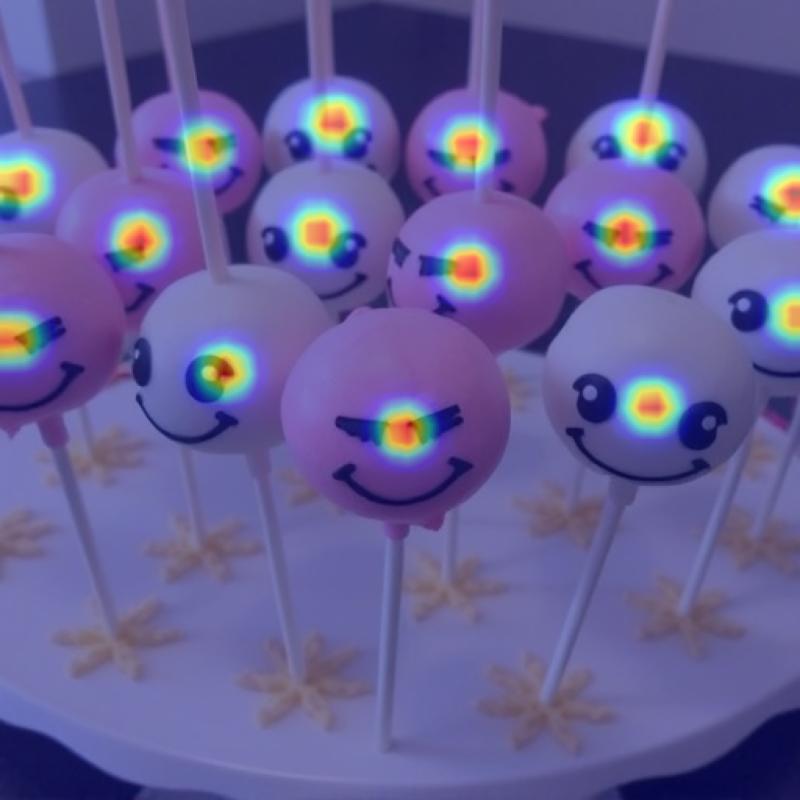} &
      \includegraphics[width=0.17\textwidth]{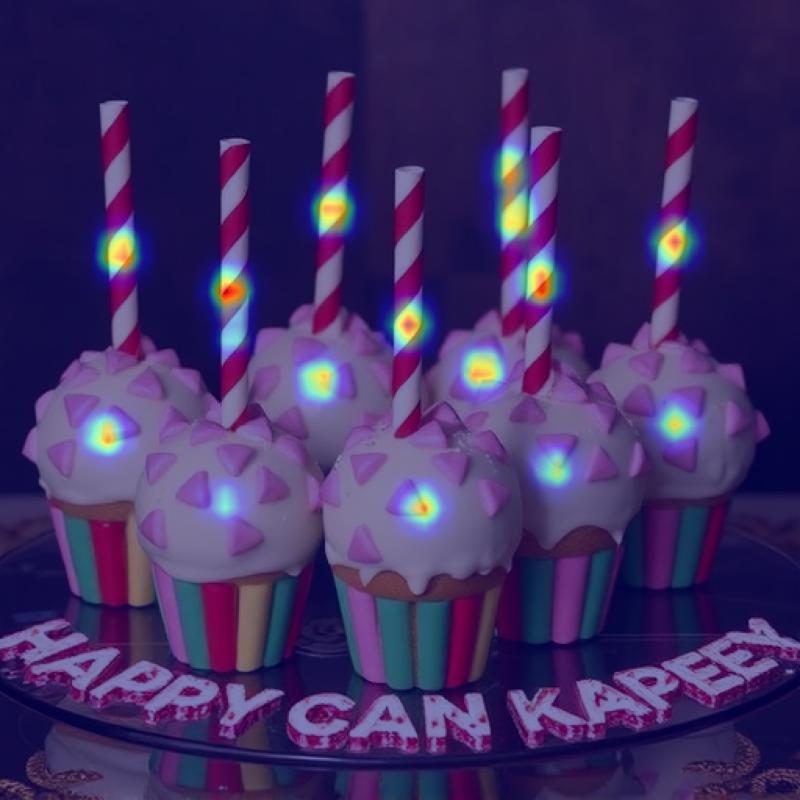} \\
      & & 8 \small($\Delta$6) & 14 \small($\Delta$0) & 14 \small($\Delta$0) & 14 \small($\Delta$0) & 14 \small(\textcolor{blue}{$\Delta$0}) \\
      
      \bottomrule
    
    \end{tabular}

    \caption{Qualitative comparison on complex and simple quantity-aware tasks. PATHS (Ours) demonstrates superior performance in satisfying quantity constraints across varying levels of difficulty.}
    \label{fig:qualitative_appendix_quantity}
  \end{center}
\end{figure}

%\scriptsize 

%%%%%%%%%%%%%%%%%%%%%%%%%%%%%%%%%%%%%%%%%%%%%%%%%%%%%%%%%%%%

\end{document}